\documentclass[nohyperref]{article}
\usepackage[final]{neurips_2024}

\usepackage{microtype}
\usepackage{graphicx}
\usepackage{subfigure}
\usepackage{multirow}
\usepackage{booktabs} 
\usepackage{natbib}
\setcitestyle{numbers}
\setcitestyle{square}
\setcitestyle{comma}

\usepackage{setspace}
\usepackage{longtable}
\usepackage{array}
\usepackage{tabularx}
\usepackage{lipsum} 
\usepackage{wrapfig} 
\usepackage{booktabs} 
\usepackage{bbm}
\usepackage{amsmath}
\usepackage{amssymb}
\usepackage{mathtools}
\usepackage{amsthm}
\usepackage{makecell}
\usepackage{xcolor}         
\usepackage{colortbl}
\usepackage{amsmath}
\usepackage{amssymb}
\usepackage{mathtools}
\usepackage{amsthm}
\usepackage{algorithm}
\usepackage{algorithmic}
\usepackage{amsfonts}
\usepackage[utf8]{inputenc} 
\usepackage[T1]{fontenc}    
\usepackage[hidelinks]{hyperref}
\definecolor{citeColor}{RGB}{0,20,115}
\hypersetup{colorlinks,linkcolor={citeColor},citecolor={citeColor},urlcolor={citeColor}}  
\usepackage{url}            
\usepackage{booktabs}       
\usepackage{amsfonts}       
\usepackage{nicefrac}       
\usepackage{microtype}      

\usepackage{amsmath}
\usepackage{array}
\usepackage{amsthm}
\usepackage{amsfonts}
\usepackage{enumerate}
\usepackage{subfigure}
\usepackage{bm}
\usepackage{tabularx}
\usepackage{enumitem}
\usepackage{multirow}
\usepackage{xcolor}
\usepackage{nicefrac}
\usepackage{booktabs}
\usepackage{bbding}
\usepackage{caption}
\usepackage{graphicx}
\usepackage{bbding}
\usepackage{wrapfig}
\usepackage{lipsum}
\usepackage{amssymb}
\usepackage{multibib}
\usepackage{tcolorbox}
\usepackage[normalem]{ulem}
\usepackage{varwidth}
\usepackage{tikz}
\usetikzlibrary{calc}
\usetikzlibrary{positioning, shapes, fit, backgrounds, decorations.pathreplacing}

\usepackage[capitalize,noabbrev]{cleveref}
\usepackage[textsize=tiny]{todonotes}

\usepackage{etoc}
\etocdepthtag.toc{mtchapter}
\etocsettagdepth{mtchapter}{subsection}
\etocsettagdepth{mtappendix}{none}

\theoremstyle{plain}
\newtheorem{theorem}{Theorem}[section]

\newtheorem{lemma}[theorem]{Lemma}

\theoremstyle{definition}

\newtheorem{condition}[theorem]{Condition}
\theoremstyle{remark}
\newtheorem{remark}[theorem]{Remark}
\newtheorem{observation}[theorem]{Observation}

\newcolumntype{L}[1]{>{\raggedright\let\newline\\\arraybackslash\hspace{0pt}}m{#1}}
\newcolumntype{C}[1]{>{\centering\let\newline  \\\arraybackslash\hspace{0pt}}m{#1}}%
\newcolumntype{R}[1]{>{\raggedleft\let\newline \\\arraybackslash\hspace{0pt}}m{#1}}

\title{Can Language Models Perform Robust Reasoning 
\\ in Chain-of-thought Prompting with Noisy Rationales?}

\makeatletter
\renewcommand*{\@fnsymbol}[1]{\ensuremath{\ifcase#1\or \dagger\or \ddagger\or
		\mathsection\or \mathparagraph\or \|\or **\or \dagger\dagger
		\or \ddagger\ddagger \else\@ctrerr\fi}}
\makeatother

\author{
	\textbf{Zhanke Zhou}$^{1}$ \quad
	\textbf{Rong Tao}$^{1}$ \quad
	\textbf{Jianing Zhu}$^{1}$ \quad
        \textbf{Yiwen Luo}$^{2}$ \quad
        \textbf{Zengmao Wang}$^{2}$ \quad
	\textbf{Bo Han}$^{1}\thanks{
		Correspondence to Bo Han (bhanml@comp.hkbu.edu.hk).}$ \\
	$^{1}$TMLR Group, Hong Kong Baptist University \quad
	$^{2}$Wuhan University \\
	\texttt{\{cszkzhou, csrongtao, csjnzhu, bhanml\}@comp.hkbu.edu.hk} \\
	\texttt{\{luoyiwen, wangzengmao\}@whu.edu.cn}
}

\definecolor{mygray}{gray}{.92}
\newcommand{\gray}{\cellcolor{mygray}}
\definecolor{baselinecolor}{rgb}{1, 1, 1}
\newcommand{\baseline}{\cellcolor{baselinecolor}}
\definecolor{ourmethodcolor}{rgb}{0.94, 0.97, 1}
\newcommand{\ours}{\cellcolor{ourmethodcolor}}

\newcommand{\noisy}[1]{\textcolor[rgb]{0.70, 0.29, 0.27}{#1}}

\begin{document}
\maketitle

\vspace{-20pt}
\begin{abstract}
\vspace{-8pt}

This paper investigates an under-explored challenge in large language models (LLMs): chain-of-thought prompting with \textit{noisy rationales}, which include irrelevant or inaccurate reasoning thoughts within examples used for in-context learning. 
We construct NoRa dataset that is tailored to evaluate the robustness of reasoning in the presence of noisy rationales. 
Our findings on NoRa dataset reveal a prevalent vulnerability to such noise among current LLMs, with existing robust methods like self-correction and self-consistency showing limited efficacy. 
Notably, compared to prompting with clean rationales, GPT-3.5 drops by $1.4\%$-$19.8\%$ in accuracy with irrelevant thoughts and more drastically by $2.2\%$-$40.4\%$ with inaccurate thoughts.

\vspace{+2pt}

Addressing this challenge necessitates external supervision that should be accessible in practice.
Here, we propose the method of contrastive denoising with noisy chain-of-thought (CD-CoT). It enhances LLMs' denoising-reasoning capabilities by contrasting noisy rationales with \textit{only one clean rationale}, which can be the minimal requirement for denoising-purpose prompting.
This method follows a principle of exploration and exploitation: (1) rephrasing and selecting rationales in the input space to achieve explicit denoising and (2) exploring diverse reasoning paths and voting on answers in the output space. 
Empirically, CD-CoT demonstrates an average improvement of $17.8\%$ in accuracy over the base model and shows significantly stronger denoising capabilities than baseline methods.
The source code is publicly available at:
\url{https://github.com/tmlr-group/NoisyRationales}.

\end{abstract}

\begin{figure}[h]
    \vspace{-12pt}
    \centering
    \includegraphics[width=0.85\textwidth]{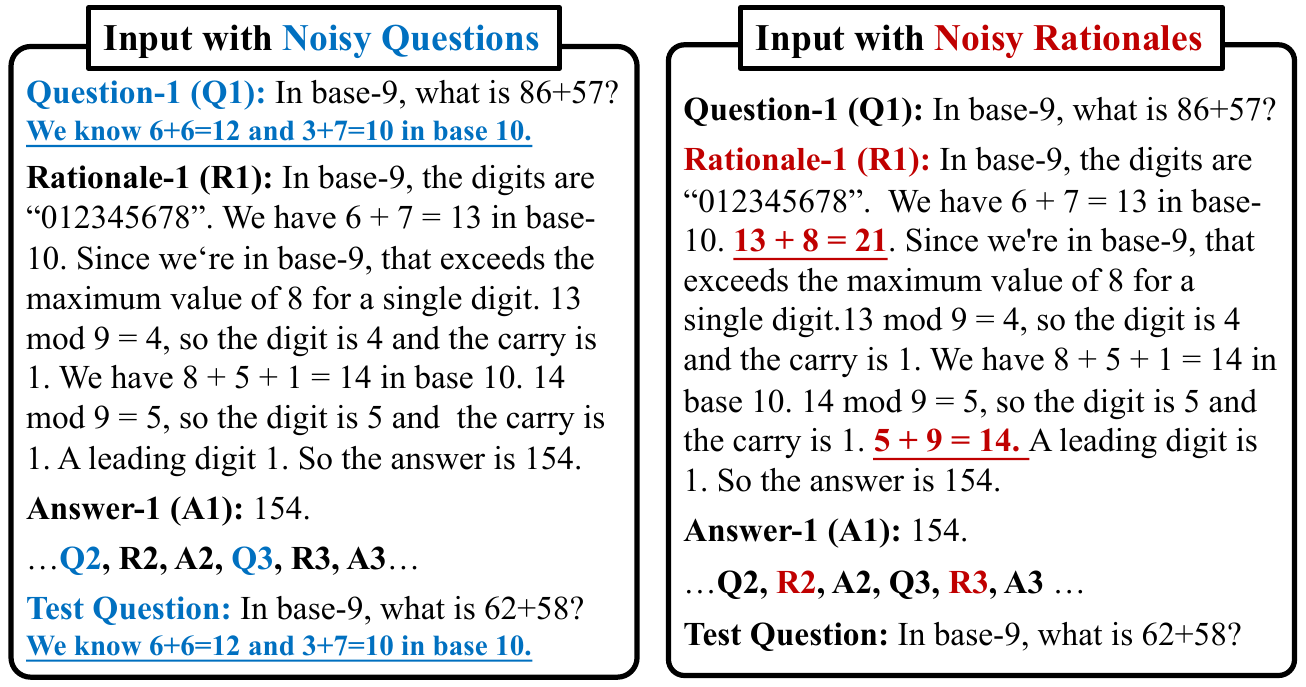}
    \vspace{-2pt}
    \caption{
    Exemplars of noisy questions~\cite{shi2023large} and \textit{noisy rationales} (our new research problem).
    Each input includes three prompting examples and one test question.
    Notably, the test question asks about base-9 calculation, 
    while the misguiding base-10 information is given in noisy questions or rationales.
    }
    \label{fig: problem}
    \vspace{-10px}
\end{figure}

\begin{figure}
\begin{minipage}[t]{0.49\linewidth}
\centering
\includegraphics[width=\textwidth]{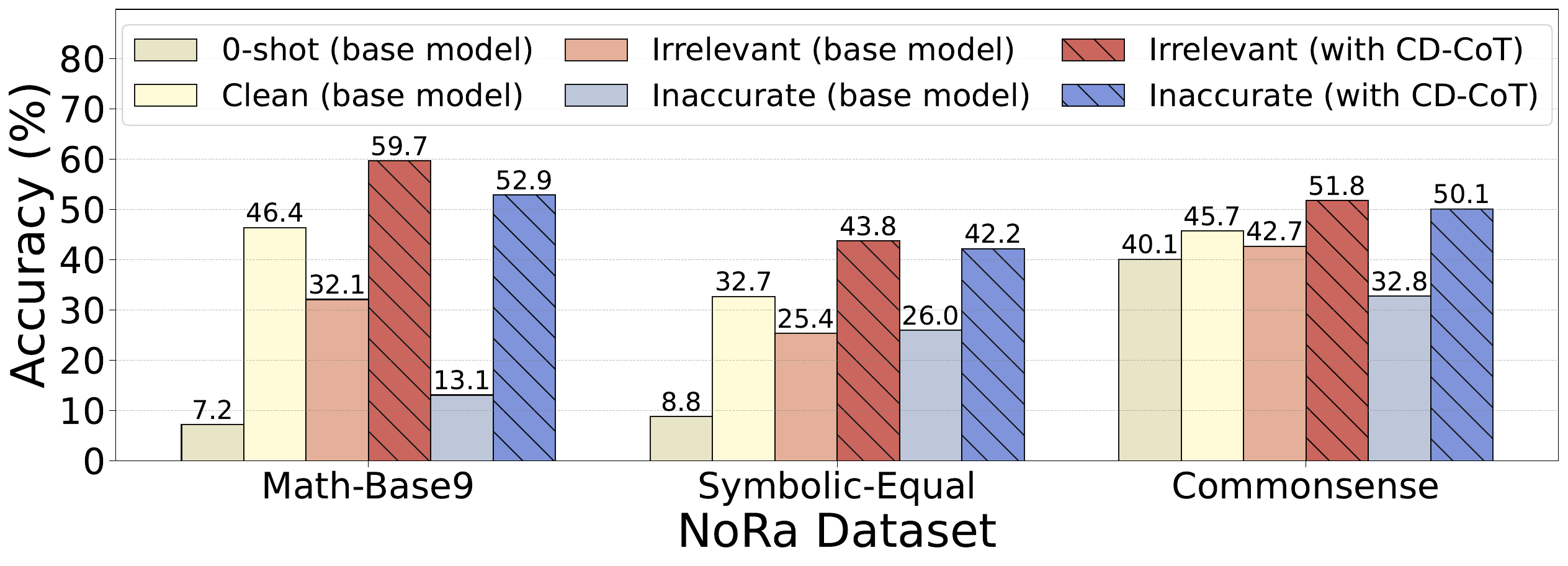}
\caption{Results of GPT-3.5 with 0-shot, 3-shot clean rationales, and 3-shot noisy rationales: Both inaccurate and irrelevant rationales degenerate performance significantly, while the proposed CD-CoT improves robustness against noisy rationales.}
\label{fig:nora_performance}
\end{minipage}
\hfill
\begin{minipage}[t]{0.48\linewidth}
\centering
\includegraphics[width=\textwidth]{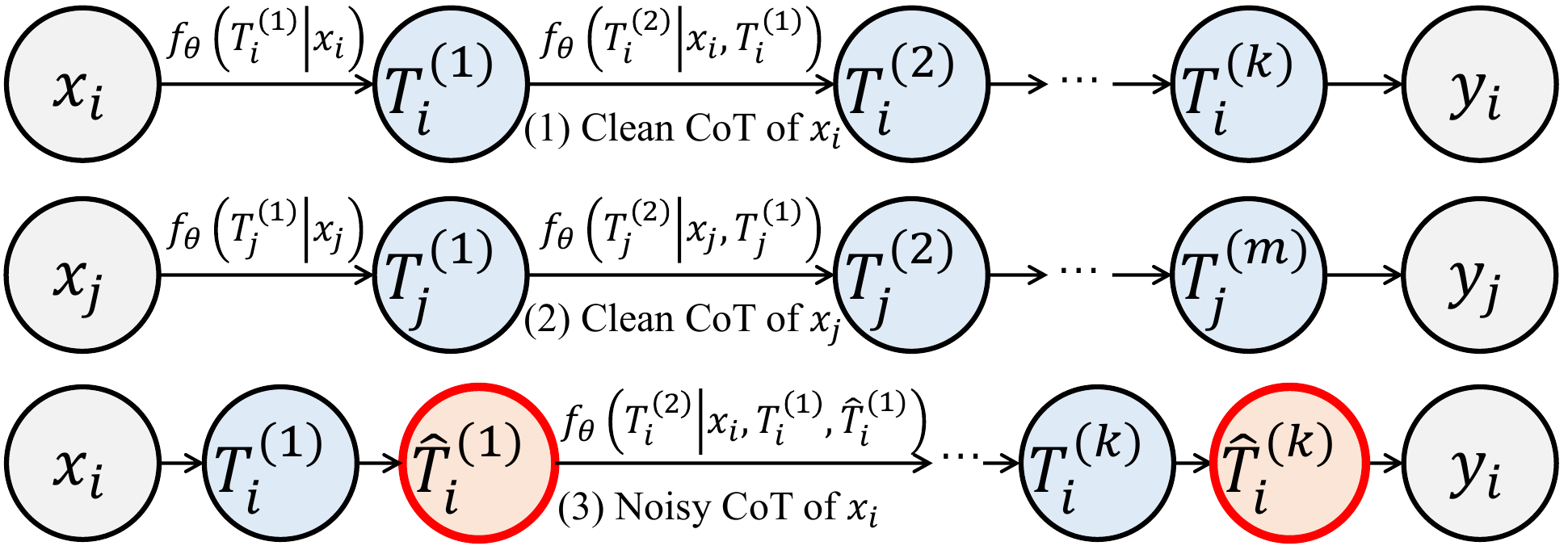}
\caption{
Chain modeling of the noisy rationale problem: Recovering chain (3) from chain (1) with the guidance of chain (2).
From question $x_i$ to answer $y_i$, the rationale of chain (3) includes clean thoughts $T_{i}^{(j)}$ and noisy thoughts $\hat{T}_{i}^{(j)}$.
}
\label{fig: modeling}
\end{minipage}
\vspace{-14px}
\end{figure}

\section{Introduction}
\vspace{-8px}

In-context learning (ICL) is a common approach in large language models (LLMs), enabling models to extrapolate from a few examples and adapt without fine-tuning~\cite{brown2020language, wei2022emergent, dong2022survey}.
However, ICL's efficacy is closely tied to the quality and clarity of the prompting examples, particularly in the prevailing chain-of-thought (CoT) strategy that provides rationales, \textit{i.e.}, intermediate reasoning steps to solve a question~\cite{wei2022chain}.
Recent research has shown that LLMs struggle with noisy questions: they are easily
distracted by irrelevant context and exhibit instability with slight input modifications~\cite{shi2023large, tian2023r3, zheng2023noisy}.

Notably, this work shifts focus from the well-studied noisy questions (Noisy-Q) problem to the under-explored \textit{noisy rationales} (Noisy-R) problem, wherein \textit{factually inaccurate or irrelevant reasoning steps} are paired with valid question-answer examples, as illustrated in Fig.~\ref{fig: problem}.
Here, the emphasis on Noisy-R is due to its practical challenges, with examples drawn from diverse sources such as crowd-sourced platforms, dialogue systems, and machine-generated data 
\footnote{The sources of noisy rationales are discussed in Appendix~\ref{appendix: benchmark} with extensive real-world examples.}
\cite{havrilla2024understanding, li2023metaagents, stacey2024lucid, agrawal2023language, thorstad2023cognitive, liu2023we}.
However, the robustness of LLMs against Noisy-R remains unknown. 
A new benchmarking dataset is needed to conduct a systematic evaluation of current LLMs and verify the corresponding countermeasures.

In this work, we first construct the NoRa (\textbf{No}isy \textbf{Ra}tionales) dataset, a comprehensive testbed to evaluate the robustness of LLM reasoning against noisy rationales across various reasoning domains (in Sec.~\ref{sec: dataset}).
The NoRa contains a total of $26391$ questions, covering three types of reasoning tasks: mathematical, symbolic, and commonsense.
We uniformly formalize the generation of noisy rationales by inserting irrelevant or inaccurate thoughts, controlling reasoning difficulty through noise ratios, and guaranteeing the overall prompting correctness without modifying the question or answer.

With the NoRa dataset, we evaluate several LLMs and reveal that all of them are \textit{intrinsically vulnerable} to noisy rationales (in Sec.~\ref{sec: experiment}).
For example, compared to prompting with clean rationales, GPT-3.5 exhibits an average $3.0\%$ - $33.3\%$ decrease in accuracy with noisy rationales, as in Fig.~\ref{fig:nora_performance}.
Besides, only limited improvements are achieved with existing robust methods based on the model's intrinsic denoising ability, \textit{e.g.}, self-consistency~\cite{wang2023self} and self-denoise~\cite{zhang2023certified}.
We show that Noisy-R is much more challenging than Noisy-Q, requiring context-specific knowledge to guide the denoising.

To solve this, we propose to rectify the rationales with \textit{only one clean CoT demonstration} that can be the most attainable supervision in practice (in Sec.~\ref{sec: method}).
We assume that LLMs can rectify rationales by \textit{contrasting} a noisy rationale with a clean one, as in Fig.~\ref{fig: modeling}.
Guided by this principle, we design the framework of \textbf{C}ontrastive \textbf{D}enoising with noisy \textbf{CoT} (CD-CoT) with four steps: rationale rephrasing, rationale selecting, rationale exploring, and answer voting.
Technically, the first two steps aim to achieve explicit denoising, while the last two steps are for diverse reasoning paths.
Empirically, CD-CoT achieves an average improvement of $17.8\%$ in accuracy \textit{w.r.t.} the base model (refer to Tab.~\ref{tab:proposed-method performance}). Notably, it presents much stronger denoising power than baselines in rectifying the rationales.

\textbf{Contributions.}
To our best knowledge, we are the \textit{first} to investigate the problem of noisy rationales.
\begin{itemize}[leftmargin=.0in]
\vspace{-4pt}
\item 
We formalize the under-explored noisy rationale problem in the prevailing chain-of-thought prompting and construct the NoRa dataset to benchmark the robustness of LLMs against noisy rationales (Sec.~\ref{sec: dataset}). 

\item
We systematically evaluate LLMs with NoRa dataset and extract several insightful observations, \textit{e.g.}, the unsatisfactory robustness and limited denoising power of LLMs under noisy rationales (Sec.~\ref{sec: experiment}).

\item 
We propose to rectify the noisy rationales with only one clean CoT demonstration, design a simple yet effective method, CD-CoT, and verify its effectiveness through comprehensive experiments (Sec.~\ref{sec: method}).
\end{itemize}

\section{Related Work}
\label{sec: related-work}
\vspace{-4px}

\textbf{Limitations of in-context learning (ICL).}
Though effective, ICL suffers from the susceptibility to manual generation and selection of demonstrations (examples), where the ultimate performance is closely tied to the demonstrations' quality and clarity.
Recent investigations on \textit{noisy questions} have shown that
(i) LLMs can be distracted by irrelevant or adversarial context, as they are designed to pay close attention to the context provided in the prompt~\cite{jia2017adversarial, pandia2021sorting, shi2023large, tian2023r3} and (ii) LLM reasoning is unstable, namely, small modifications to the prompt could potentially cause large variations in the model's output~\cite{zhang2023certified, zheng2023noisy}.
Besides, another line of research regarding \textit{noisy answers}~\cite{lei2023s, freeman2023frontier} justifies the feasibility of misleading an LLM to agree factual errors such as "1+1=3" in base-10 calculation.

\textbf{Countermeasures.}
Two intrinsic traits of LLMs are desirable for addressing the above limitations:
\begin{itemize}[leftmargin=.0in]
\vspace{-4pt}

\item
\textit{Self-correction}, wherein LLMs attempt to correct their initial responses based solely on their inherent capabilities without external feedback, \textit{e.g.}, by refining prompts through iterative corrections of responses or question trajectories~\cite{yang2023large,xi2023self}. 
Although LLMs can learn to ignore irrelevant information by examples or instructions~\cite{shi2023large}, they are proved to be still struggling to correct their responses without external feedback, and at times, their performance might even degrade after self-correction~\cite{huang2023large, tyen2023llms}.

\item
\textit{Self-consistency} aims to obtain a consistent answer against input perturbations.
This is achieved by generating multiple responses via randomized smoothing on input questions~\cite{zeng2023certified} or diverse paths for answering one question~\cite{wang2023self} followed by the answer aggregation. 
This strategy brings improvements with extra costs for repeated reasoning. Moreover, it cannot explicitly rectify questions or rationales.

\vspace{-4px}
\end{itemize}

\textbf{Noisy rationales}, as the research focus of this work, mainly originates from (1) the inherent imperfections, inconsistencies, and inaccuracies of humans' cognitive processes~\cite{mayer1977thinking, choi2018quac} and (2) the diversity, unpredictability, and hallucination of the LLMs' generative mechanisms~\cite{zhang2023automatic, huang2023survey, zhang2023hallucination}.
A detailed literature review and discussion of noise rationales are in Appendix~\ref{appendix: related work} and \ref{appendix: benchmark}, respectively. 

\vspace{-4pt}
\section{The NoRa Dataset}
\label{sec: dataset}
\vspace{-4pt}

In this section, we introduce the NoRa (\textbf{No}isy \textbf{Ra}tionales) dataset for benchmarking the robustness against noisy rationales.
NoRa consists of $26391$ questions and $5$ subsets, covering mathematical, symbolic, and commonsense reasoning tasks, where ICL and CoT demonstrations play a crucial role.


\vspace{-4pt}
\subsection{Definition of Noisy Rationales}
\label{ssec: generation}

We start by formalizing the ICL and CoT demonstrations.
Given a test question $x_{\text{test}}$ and an LLM $f_{\theta}$, one expects to get the correct answer $y_{\text{test}}$ as $f_{\theta}(x_{\text{test}}) \! \mapsto \! y_{\text{test}}$.
This zero-shot manner cannot guarantee effectiveness, especially when encountering unfamiliar contexts or scenarios.
To boost effectiveness, the ICL techniques prompt the LLM with a few examples $S_{n} \! = \! \{(x_{i}, y_{i})\}_{i=1}^{n}$ collected in the current context, each composed of a question $x_{i}$ and answer $y_{i}$, and then construct the new input $x_{\text{ICL}}$ as
\begin{align}
\vspace{-2pt}
x_{\text{ICL}} = \left[S_{n}, x_{\text{test}}\right] &= \left[x_{1}, y_{1}, \ldots, x_{n}, y_{n}, \; x_{\text{test}}\right]. 
\vspace{-2pt}
\end{align}
The guidance by $S_{n}$ makes
$f_{\theta}(x_{\text{ICL}}) \! \mapsto \! y_{\text{test}}$
much easier than
$f_{\theta}(x_{\text{test}}) \! \mapsto \! y_{\text{test}}$.
Then, the CoT further refines $x_{\text{ICL}}$ by constructing the step-by-step rationale $\mathcal{T}_i$, consisting of several thoughts $T_{i}^{(j)}$, namely,
\begin{align}
\vspace{-2pt}
x_{\text{CoT}} = \left[x_{1}, \mathcal{T}_1, y_{1}, \ldots, x_{n}, \mathcal{T}_n, y_{n}, \; x_{\text{test}}\right], 
\text{\; where \;}
\mathcal{T}_i = \left[T_{i}^{(1)}, T_{i}^{(2)}, T_{i}^{(3)}, \ldots, T_{i}^{(k)}\right].
\vspace{-2pt}
\label{eqn: clean CoT}
\end{align}
However, as aforementioned, the thoughts in CoT (Eqn.~\ref{eqn: clean CoT}) can be noisy in practice.
This noise can be attributed to (1) \textit{irrelevant thoughts}, which are irrelevant but correct, or (2) \textit{inaccurate thoughts}, which are relevant but factually wrong.
Here, we \textit{uniformly} formalize these two kinds of noise as 
\begin{align} 
\vspace{-2pt}
\noisy{\hat{\mathcal{T}}_i}= \left[ T_{i}^{(1)}, \noisy{\hat{T}_{i}^{(1)}}, T_{i}^{(2)}, \noisy{\hat{T}_{i}^{(2)}}, 
\ldots, T_{i}^{(k)}, \noisy{\hat{T}_{i}^{(k)}} \right],
\vspace{-2pt}
\label{eqn: irrelevant CoT}
\end{align}
where $\noisy{\hat{T}_{i}^{(j)}}$ represents a noisy thought (irrelevant or inaccurate)
that is coherent with the previous clean thought $T_{i}^{(j)}$ (relevant and correct).
The following introduces the definition of noisy thoughts.


\textbf{Irrelevant thoughts} refer to incorporating irrelevant information unhelpful for solving the question, \textit{e.g.}, discussing the genetic overlap of siblings when the task is to deduce family roles in relationship reasoning.
Redundant information may be introduced by the LLM's diverse response generation or by humans when clarifying concepts in problem-solving examples~\cite{chandler1991cognitive, zhao2023explainability}.
Alternatively, we explore various semantic levels of "irrelevance" for constructing diverse irrelevant thoughts in Appendix~\ref{appendix: noise_semantic_difficulty}.

\textbf{Inaccurate thoughts} refer to factual errors in rationales that are common in mathematical calculation or transcription, \textit{e.g.}, "5+5=10" is wrong in base-9 calculation.
This noise comes from algorithmic limitations, errors in training data, misinterpretations of instructions, and logical fallacies~\cite{koo2023uncovering, sambasivan2021everyone}.

\begin{remark}
\vspace{+2pt}
Both types of noise only impact the finer details of the reasoning chain without affecting the correctness of question $x_{i}$ and answer $y_{i}$.
This distinction ensures that the reasoning based on the noisy demonstration is not fundamentally flawed, namely, only the reasoning rationale $\noisy{\hat{\mathcal{T}}_i}$ is noisy.
\end{remark}

\begin{table*}[t!]
    \centering
    \fontsize{7}{7}\selectfont
    \setlength\tabcolsep{2pt}
    \begin{tabular}{p{0.1\textwidth}|p{0.42\textwidth}|p{0.42\textwidth}}
        \toprule
        \multicolumn{1}{c|}{Task} &  \multicolumn{1}{c|}{Irrelevant Thoughts} & \multicolumn{1}{c}{Inaccurate Thoughts} \\
        \midrule
        \multicolumn{1}{c|}{\multirow{7}*{NoRa-Math}}  & In base-9, digits run from 0 to 8.
We have 3 + 2 = 5 in base-10. 
Since we're in base-9, that doesn't exceed the maximum value of 8 for a single digit.
5 mod 9 = 5, so the digit is 5 and the carry is 0. \noisy{\uline{There are five oceans on Earth: the Atlantic, Pacific, Indian, Arctic, and Southern.}} We have 8 + 6 + 0 = 14 in base 10. 
14 mod 9 = 5, so the digit is 5 and the carry is 1. 
A leading digit 1. 
So the answer is 155. Answer:  155
 &In base-9, digits run from 0 to 8.  We have 3 + 2 = 5 in base-10. \noisy{\uline{5 + 4 = 9.}} Since we're in base-9, that doesn't exceed the maximum value of 8 for a single digit. 5 mod 9 = 5, so the digit is 5 and the carry is 0.  \noisy{\uline{5 + 9 = 14.}} We have 8 + 6 + 0 = 14 in base 10. 14 mod 9 = 5, so the digit is 5 and the carry is 1. A leading digit 1. So the answer is 155. Answer: 155

 \\ \midrule
        \multicolumn{1}{c|}{\multirow{6}*{NoRa-Symbolic}}   & ... "turn around right" means the agent needs to turn right, and repeat this action sequence four times to complete a 360-degree loop. \noisy{\uline{Many GPS navigation systems will issue a 'turn around' command if the driver deviates from the planned route.}} So, in action sequence is I\_TURN\_RIGHT I\_TURN\_RIGHT I\_TURN\_RIGHT I\_TURN\_RIGHT.  ...
 & ... "turn around right" means the agent needs to turn right, and repeat this action sequence four times to complete a 360-degree loop. \noisy{\uline{Turn opposite is I\_TURN\_RIGHT I\_TURN\_LEFT.}} So, in action sequence is I\_TURN\_RIGHT I\_TURN\_RIGHT I\_TURN\_RIGHT I\_TURN\_RIGHT. ...\\  \midrule
        \multicolumn{1}{c|}{\multirow{8}*{NoRa-Com.}}  &  The relations path are son, sister, uncle, which means Francisco is David's son's sister's uncle. For son's sister, we have son's sister is daughter. So the relations path are reduced to daughter, uncle. \noisy{\uline{In genetics, mitochondrial DNA is always inherited from the mother, making the mother-daughter genetic link unique.}} For daughter's uncle, we have daughter's uncle is brother. So the relations path are reduced to brother. Therefore, the answer is brother. Answer:brother &The relations path are son, sister, uncle, which means Francisco is David's son's sister's uncle. For son's sister, we have son's sister is daughter. So the relations path are reduced to daughter, uncle. For daughter's uncle, we have daughter's uncle is brother. \noisy{\uline{We have brother' sister is brother.}} So the relations path are reduced to brother. Therefore, the answer is brother. Answer:brother\\
        \bottomrule
    \end{tabular}
    \vspace{-4pt}
    \caption{
    Noisy rationales (consisting \noisy{\uline{noisy thoughts}}) sampled from the NoRa dataset. 
    Full examples of NoRa are in Appendix~\ref{appendix: benchmark-examples},
    and
    real-world examples of noisy rationales are in Appendix~\ref{appendix: real-world examples}.
    }
    \label{tab:dataset-examples}
\vspace{-8pt}
\end{table*}

\vspace{-6pt}
\subsection{Tasks and Statistics}



The NoRa dataset covers the three types of reasoning tasks listed below.
In noise generation, irrelevant thoughts, sourced from extraneous scientific or social facts, and inaccurate thoughts, arising from misguided reasoning, are both based on contextual thoughts of Eqn.~\ref{eqn: clean CoT}
(see examples in Tab.~\ref{tab:dataset-examples}).

\begin{itemize}[label=\textbullet, leftmargin=.0in]

\item
\textbf{NoRa-Math.}
This task is derived from the Base Calculation dataset~\cite{wu2023reasoning} for evaluating non-standard base arithmetic skills and features two sub-tasks of \textit{base-9} and \textit{base-11} addition.
Here, the mastery of mathematical concepts and the calculation rules of specific bases are the keys to solving these tasks.

\item
\textbf{NoRa-Symbolic.}
We utilize the SCAN dataset~\cite{lake2018generalization} here, which aims to transform natural language into symbolic, machine-understandable instructions.
This transformation is learned from the prompting examples, comprising two sub-tasks:
(1) equal-length subtask, where the transformed instructions in both $S_{n}$ and $x_{\text{test}}$ have the equal length; and (2) longer-length subtask, where the transformed instructions in $x_{\text{test}}$ is longer than those in $S_{n}$, presenting an easy-to-hard generalization challenge.

\item
\textbf{NoRa-Commonsense.}
This task is constructed based on the CLUTRR dataset~\cite{sinha2019clutrr}, which is geared towards family relation path reasoning, \textit{e.g.}, "who is aunt's sister's mother?" It requires the mastery and application of commonsense knowledge of relationships and cognitive skills for reasoning.

\vspace{-4pt}
\end{itemize}

\textbf{Noise Ratio.}
Given the noise ratio $\epsilon \! \in \! (0,1)$, the expected number of added noisy thoughts for a $k$-length CoT demonstration is $\lfloor \epsilon \! \cdot \! k + \nicefrac{1}{2} \rfloor$.
Specifically, for an irrelevant thought $\hat{T}_{i}^{(j)}$ in $j$-th position of $i$-th example, a Bernoulli distribution $\text{Bern}(\epsilon) \! \in \! \{0,1\}$ is adopted to indicate its binary existence. 
We also further investigate the impact of the number of noisy thoughts, provided in Appendix~\ref{appendix: the Number of Noisy Thoughts}.

\begin{wraptable}{r}{0.53\textwidth}
    \vspace{-4pt}
    \centering
    \setlength\tabcolsep{3pt}
    \resizebox{0.53\textwidth}{!}{
        \begin{tabular}{c|c|ccccc}
            \toprule
            \multirow{2}{*}{Difficulty} &  Noise & \multicolumn{5}{c}{\#total thoughts (\#noisy thoughts) of prompting rationales (Avg.)} \\
            &  Ratio & Math Base-9 & Math Base-11 & Sym.~Equal & Sym.~Longer & Com. \\
            \midrule
            Easy        & 0.3   &  10 (2)&  10 (2)& 11.5 (2.7) & 11.0 (2.5) & 7 (2) \\
            Medium      & 0.5   &  12 (4)&  12 (4)& 13.3 (4.5) & 12.7 (4.2) & 8 (3)\\
            Hard        & 0.8   &  14 (6)&  14 (6)&  16.0 (7.1) & 15.2 (6.8) & 9 (4)\\
            \midrule
            \multicolumn{2}{c|}{\#questions}  & 4024 & 9269 & 4182 & 3920 & 4996 \\
            \bottomrule
        \end{tabular}}
    \vspace{-2pt}
    \caption{Statistics of NoRa dataset.}
    \label{tab:dataset_statis}
\vspace{-14pt}
\end{wraptable}
\textbf{Statistics.}
Tab.~\ref{tab:dataset_statis} provides a categorization of task difficulties as Easy, Medium, and Hard, with corresponding noise ratios of $0.3$, $0.5$, and $0.8$.
Regarding the number of thoughts in a rationale,
Math entails $8$ thoughts, Symbolic varies from $2$ to $12$, and Commonsense requires $5$ thoughts.
Additionally, a detailed introduction to NoRa is in Appendix~\ref{appendix: benchmark-details}.

\begin{table*}[t]
\definecolor{mygray}{gray}{.9}
    \centering
    \fontsize{7}{7}\selectfont
    \renewcommand\arraystretch{0.88}
    \setlength{\tabcolsep}{6pt}
    \begin{tabular}{c|l|c|cccc|cccc}
        \toprule
        \multirow{2}{*}{Task} & \multicolumn{1}{c|}{\multirow{2}{*}{Method $\mathcal{M}$}} &\multirow{2}{*}{\makecell[c]{$\operatorname{Acc}(\mathcal{M}, \mathcal{Q}, \mathcal{P}_{\text{clean}})$}} &  \multicolumn{4}{c|}{$\operatorname{Acc}(\mathcal{M}, \mathcal{Q}, \mathcal{P}_{\text{irrelevant}})$} &  \multicolumn{4}{c}{$\operatorname{Acc}(\mathcal{M}, \mathcal{Q}, \mathcal{P}_{\text{inaccurate}})$}\\
            &   &   &    Easy & Medium & Hard & Avg. & Easy & Medium & Hard & Avg. \\
        \midrule
        \multirow{6}{*}{\makecell[c]{Math \\ Base-9}}     
            &  \makecell[c]{Base}                           & \gray 46.4            & \gray \underline{39.3}          & \gray \underline{30.3}        & \gray \underline{26.6} & \gray \underline{32.1}      & \gray 23.2     & \gray 10.1      & \gray 6.0         & \gray 13.1   \\
            &  w/ ISC ~\citep{huang2023large}               & 24.3            & 17.7           &  14.7      & 12.7        & 15.0        & 18.4           & 13.7     & 12.3          & 14.8 \\
            &  w/ SP  ~\cite{xi2023self}                    & 26.2              & 25.5           &  25.5      & 21.9         & 24.3        & 20.0           & \underline{18.4}&  \textbf{14.3}  & 17.6  \\ 
            &  w/ SM ~\citep{robey2023smoothllm}            & 37.4              & 30.0            &  22.7     & 16.5          & 23.1          & 24.7        & \textbf{19.2}  & \underline{12.4} & \textbf{18.8}   \\    
            &  w/ SD ~\citep{zhang2023certified}            & \underline{47.9} & 37.2  &  25.4  & 24.7   &29.1 & \underline{29.3}  & 12.5         & 8.7      & 16.8 \\
            &  w/ SC ~\cite{wang2023self}                   & \textbf{61.5}    & \textbf{51.1}    & \textbf{39.0}   & \textbf{36.2} &\textbf{42.1}& \textbf{32.7}&  15.3       & 7.5         & \underline{18.5}   \\ 
        \midrule
        \multirow{6}{*}{\makecell[c]{Math \\ Base-11}}     
            &  \makecell[c]{Base}                           & \gray \underline{23.9} & \gray  \underline{19.1} & \gray  13.6          & \gray  10.7       & \gray  14.5                & \gray  14.0          & \gray  6.7             & \gray  3.6    & \gray  8.1 \\
            &  w/ ISC ~\citep{huang2023large}               & 11.2           & 8.3              & 7.8          & 6.0           & 7.4                 & 6.5          & 5.2          & 4.7       & 5.5  \\
            &  w/ SP ~\cite{xi2023self}                     & 20.7           & 17.5            & \textbf{16.7} &\underline{14.0}&\underline{16.0}  &\underline{14.1} &  \textbf{10.7} & \textbf{10.8}   & \textbf{11.9}  \\    
            &  w/ SM ~\citep{robey2023smoothllm}            & 16.3           & 12.0           & 6.0            & 5.7            & 7.9                & 12.0        & \underline{9.3}& \underline{7.7} & 9.7  \\      
            &  w/ SD ~\citep{zhang2023certified}            & 17.9           & 12.3           & 12.0          & 13.3           & 12.5            & 17.0              & 8.7             & 5.3          & 10.3   \\ 
            &  w/ SC ~\cite{wang2023self}                   & \textbf{33.7}  & \textbf{25.3}  &\underline{16.3} & \textbf{15.0}  & \textbf{18.9}   &  \textbf{19.7} &\underline{9.3}& 3.3       & \underline{10.8}   \\
        \midrule
        \multirow{6}{*}{\makecell[c]{Symbolic \\ Equal}}     
            &  \makecell[c]{Base}                            & \gray \underline{32.7} & \gray  \underline{28.1} & \gray  \underline{25.1} & \gray  \underline{23.0} & \gray  \underline{25.4} & \gray \underline{29.1} & \gray  \underline{26.1}     & \gray  22.7          & \gray  \underline{26.0} \\
            &  w/ ISC~\citep{huang2023large}                 & 23.9           &  20.0         & 16.3               & 15.5           & 17.3               & 19.2       &           18.3               & 18.1         & 18.5\\ 
            &  w/ SP ~\cite{xi2023self}                      & 23.2           &  23.0       & 22.6                 & 22.7           & 22.8       &  23.7     & 22.5          & \underline{23.5}  & 23.2\\    
            &  w/ SM ~\citep{robey2023smoothllm}             & 25.0           &  20.7        & 19.7                & 16.7           & 19.0       &  21.0      & 20.3        &          20.0              & 20.4\\    
            &  w/ SD ~\citep{zhang2023certified}             & 9.9                 &  10.1           & 10.9                  & 10.3              & 10.4          &  10.1         & 10.9           & 10.4           & 10.5 \\
            &  w/ SC  ~\cite{wang2023self}                   &\textbf{35.3}   &  \textbf{31.0} & \textbf{28.3}     & \textbf{27.0} & \textbf{28.8}&\textbf{33.3}      & \textbf{30.7}    & \textbf{26.0}    &\textbf{30.0} \\ 
         \midrule
        \multirow{6}{*}{\makecell[c]{Symbolic \\ Longer}}     
            &  \makecell[c]{Base}                             & \gray \underline{9.2} & \gray \underline{6.3}  & \gray  \underline{7.2} & \gray  \underline{6.0} & \gray  \underline{6.5} & \gray  \underline{7.0} & \gray  \underline{6.8} & \gray  \underline{6.0}  & \gray  \underline{6.6}\\
            &  w/ ISC ~\citep{huang2023large}                 &  4.9          &  4.6          & 2.7              & 3.7              & 3.7        & 3.4          & 4.3              & 3.3             & 3.7\\
            &  w/ SP ~\cite{xi2023self}                       &  5.1       &  4.3             & 4.1              & 3.9              & 4.1         & 4.9          & 4.0             & 4.5              & 4.5\\    
            &  w/ SM ~\citep{robey2023smoothllm}              &  1.7             &  0.7          & 0.7                   & 1.3                    & 1.0           & 1.3       & 0.7         & 0.3          & 0.8\\    
            &  w/ SD ~\citep{zhang2023certified}              &  0.1            &  0.1          & 0.1                   & 0.2                   & 0.1           & 0.1       & 0.3           & 0.0           & 0.1  \\
            &  w/ SC  ~\cite{wang2023self}                    &  \textbf{13.0}  & \textbf{7.7}     & \textbf{9.0}  & \textbf{6.3}          & \textbf{7.7} & \textbf{8.0}     & \textbf{8.0}    & \textbf{8.7}         & \textbf{8.2} \\ 
         \midrule 
        \multirow{6}{*}{Commonsense}
            &   \makecell[c]{Base}                             & \gray   45.7           & \gray  44.3              & \gray  42.3              & \gray  41.4              & \gray  42.7              & \gray  36.7              & \gray  33.4              & \gray  28.3 & \gray  32.8 \\
            &  w/ ISC ~\citep{huang2023large}                 &  21.8           & 24.3              & 22.5              & 21.4              & 22.7              & 23.3              & 26.5              & 24.0 & 24.6 \\
            &  w/ SP ~\cite{xi2023self}                       &  47.9           & 48.2              & 46.7              & \underline{48.1}  & 47.7              & \underline{49.6}  & 46.6              & 46.5 & 47.6\\    
            &  w/ SM ~\citep{robey2023smoothllm}              &\underline{53.3} &\underline{50.3}   & \underline{50.0}  & 46.7              & \underline{49.0}  & 47.7              & \underline{49.0}  & \underline{49.3} & \underline{48.7}  \\   
            &  w/ SD ~\citep{zhang2023certified}              & \textbf{54.0}   & \textbf{58.3}     & \textbf{57.3}     & \textbf{57.7}     &\textbf{57.8}      & \textbf{57.0}     & \textbf{58.3}     & \textbf{53.7} & \textbf{56.3} \\
            &  w/ SC  ~\cite{wang2023self}                     &52.0            & 46.3              & 45.0              & 44.7              & 45.3              & 44.7              & 44.7              & 38.0 & 42.5 \\ 
        \bottomrule
    \end{tabular}
    \caption{
    Reasoning accuracy on NoRa dataset with 3-shot prompting examples with clean, irrelevant, or inaccurate rationales.
    The \textbf{boldface} numbers mean the best results, while the \underline{underlines} numbers indicate the second-best results.
    Note the referenced results of \colorbox{mygray}{\parbox{1.6cm}{Base model}} are highlighted in gray.
    }
    \label{tab:vanilia_LLM}
\vspace{-12pt}
\end{table*}

\section{Evaluating Language Models on NoRa dataset}
\label{sec: experiment}
In this section, we comprehensively evaluate representative LLMs and robust methods on the newly constructed NoRa dataset. We first introduce the basic evaluation setups and then present several observations on the \textit{unsatisfactory robustness} of current LLMs and methods under noisy rationales.

\textbf{Baseline methods.}
We select five representative methods as baselines to ensure a comprehensive assessment that encompasses the two traits of self-correction and self-consistency. 
ISC~\citep{huang2023large} and SP~\cite{xi2023self} exemplify self-correction, focusing on response rectification and prompt rephrasing, respectively.
SM~\citep{robey2023smoothllm}, SD~\citep{zhang2023certified}, and SC~\cite{wang2023self} fall under self-consistency:
SM~\citep{robey2023smoothllm} injects perturbations into prompts for robustness, SD~\citep{zhang2023certified} masks prompts and asks LLMs to reconstruct them, while SC directly samples outputs without denoising. 
These methods are further introduced in Appendix~\ref{app:implementation-baselines}.

\textbf{LLM basis.} 
We employ GPT-3.5-turbo-0613~\citep{floridi2020gpt} as our base LLM~(denoted as \colorbox{mygray}{\parbox{0.65cm}{Base}}) for the analyses presented in this study. In addition, we conduct evaluations on three supplementary models, including Gemini-Pro (Jan. 2024)~\citep{team2023gemini}, Llama2-70B~\citep{touvron2023llama}, and Mixtral-8x7B~\citep{jiang2024mixtral}.
For all baselines, we consistently set the temperature parameter $\tau$ to the value of 1.
In order to obtain consistent results, we evaluate 300 questions for each task and repeat the model reasoning five times for each question.

\textbf{Evaluation metric.}
Given a set of test question $\mathcal{Q} = \{(x_{\text{test}}, y_{\text{test}})\}$
and a set of CoT-prompting examples $\mathcal{P} = \left[x_{1}, \mathcal{T}_1, y_{1}, \ldots, x_{n}, \mathcal{T}_n, y_{n} \right]$,
we define the accuracy of the denoising method $\mathcal{M}$ with a specific LLM $f_{\theta}$, namely,
$\operatorname{Acc}(\mathcal{M}, \mathcal{Q}, \mathcal{P}) = 
\nicefrac{\sum_{(x_{\text{test}}, y_{\text{test}}) \in \mathcal{Q}}  
\mathbf{1}[\mathcal{M}(\mathcal{P}, x_{\text{test}})=y_{\text{test}}]}
{|\mathcal{Q}|}$.
We report the results in percentage (\%) with one decimal point.
Therein,
$\operatorname{Acc}(\mathcal{M}, \mathcal{Q}, \mathcal{P}_{\text{clean}})$, 
$\operatorname{Acc}(\mathcal{M}, \mathcal{Q}, \mathcal{P}_{\text{irrelevant}})$,
and
$\operatorname{Acc}(\mathcal{M}, \mathcal{Q}, \mathcal{P}_{\text{inaccurate}})$
indicate accuracy with clean, irrelevant, and inaccurate rationales, respectively.
When there is no prompting example, \textit{i.e.}, $\mathcal{P} = \emptyset$, then
$\operatorname{Acc}(\mathcal{M}, \mathcal{Q}, \emptyset)$ represents the zero-shot result.

\begin{table*}[t!]
\centering
\begin{minipage}{.32\textwidth}
    \centering
    \fontsize{7}{7}\selectfont
    \setlength\tabcolsep{1pt}
        \begin{tabular}{c|c|cccccc}
            \toprule
            \multirow{2}{*}{Task} & \multirow{2}{*}{Setting} & \multicolumn{5}{c}{Temperature} \\
                                       &          & 0 & 0.3 & 0.5 & 0.7 & 1\\
            \midrule
            \multirow{3}{*}{Base-9} & clean & \textbf{61.0} & \underline{60.9} & 57.5 & 55.3 & \gray 46.4 \\
                                       & ina. easy   & \textbf{29.7} & \underline{28.0} & 27.2 & 26.6 & \gray 21.7 \\
                                       & ina. hard   & 5.0 & \underline{5.1} & \textbf{5.5} & 4.6 & \gray 5.0 \\
            \midrule
            \multirow{3}{*}{Base-11} & clean    & \textbf{34.0} & \underline{33.8} & 31.6 & 29.8 & \gray 23.9\\
                                       & irr. easy     & 21.7 & \underline{23.1} & 21.3 & \textbf{23.3} & \gray 19.1\\
                                       & irr. hard   & \underline{17.0} & \textbf{17.5} & 15.5 & 14.1 & \gray 10.7\\
            \midrule
            \multirow{3}{*}{Sym.(E)} & clean    & 34.2 & \textbf{35.8} & \underline{35.7} & 34.6 & \gray 32.7 \\
                                       & irr. easy     & 28.6 & \textbf{31.5} & \underline{29.8} & 29.1 & \gray 28.1\\
                                       & irr. hard   & \textbf{27.0} & 26.1 & \underline{26.2}  & 24.0 & \gray 23.0\\
            \midrule
            \multirow{3}{*}{Sym.(L)} & clean    & 6.3 & 8.3 & \underline{8.9} & \underline{8.9} & \gray \textbf{9.3} \\
                                       & ina. easy     & 5.0 & 7.3 & \textbf{8.6}  & \underline{8.3} & \gray 7.0\\
                                       & ina. hard   & 4.0 & 6.1 & \textbf{6.3} & \underline{6.2} & \gray 6.0\\
            \bottomrule
        \end{tabular}
    \caption{
    Comparing performances of the base model with different temperatures.
    Sym.(E)/(L) are symbolic tasks.} 
    \label{tab:temperatures}
\end{minipage}%
\hfill
\begin{minipage}{.32\textwidth}
\centering
    \fontsize{7}{7}\selectfont
    \setlength\tabcolsep{1pt}
        \begin{tabular}{c|c|cccccc}
            \toprule
            \multirow{2}{*}{Task} & \multirow{2}{*}{Setting} & \multicolumn{5}{c}{\#Prompting Examples} \\ 
                                       &            & 1 & 2 & 3 & 4 & 5 \\ 
            \midrule
            \multirow{3}{*}{Base-9} & clean           & 24.8       & 38.3  & \gray 46.4  & \textbf{50.8}  & \underline{50.5} \\ 
                                       & ina.-easy   & 17.5       & 22.2  & \gray 23.2  & \underline{25.4}  & \textbf{25.6} \\
                                       & ina.-hard   & \textbf{11.3}       & \underline{6.3}   & \gray 6.0   & 5.7   & 5.7 \\
            \midrule
            \multirow{3}{*}{Base-11} & clean           & 11.8 & 20.4 & \gray 23.9 & \underline{29.9} & \textbf{32.1} \\
                                       & irr. easy    & 8.9  & 15.9 & \gray 19.1 & \underline{21.7} & \textbf{26.3} \\
                                       & irr. hard    & 7.7  & 10.0 & \gray 10.7 & \underline{15.2} & \textbf{16.1} \\
            \midrule
            \multirow{3}{*}{Sym.(E)} & clean    & 18.0 & 26.5 & \gray \underline{32.7} & \textbf{39.8} & --- \\ 
                                       & ina.-easy     & 17.3 & 23.6 & \gray \underline{29.1} & \textbf{34.7} & --- \\ 
                                       & ina.-hard   & 15.0 & \underline{21.0} & \gray \textbf{22.7} & --- & --- \\ 
            \midrule
            \multirow{3}{*}{Sym.(L)} & clean      & 2.7 & 7.7 & \gray 9.3 & \underline{11.3} & \textbf{12.2}\\
                                       & irr. easy  & 2.3 & 5.4 & \gray 7.0 & \underline{8.8} & \textbf{8.9 } \\
                                       & irr. hard  & 1.9 & 4.0 & \gray \underline{6.0} & \textbf{6.3} & --- \\
            \bottomrule
        \end{tabular}
    \caption{
    Comparing performances of the base model with a varying number of examples
    ("---" denotes over token limit).} 
    \label{tab:increasing_exemplars}
\end{minipage}%
\hfill
\begin{minipage}{.32\textwidth}
    \centering
    \fontsize{7}{7}\selectfont
    \setlength\tabcolsep{1pt}
    \begin{tabular}{c|c|c|c|cc}
        \toprule
        \multirow{2}{*}{Model} & \multirow{2}{*}{Task} & \multicolumn{4}{c}{Setting} \\
        && 0-shot & clean & irr. & ina.\\
        \midrule
        \multirow{3}{*}{\makecell[c]{GPT3.5}} & 
        \makecell[c]{Base-9} & \gray 7.2 & \gray \textbf{46.4} & \gray \underline{30.3} & \gray 10.1\\
        & \makecell[c]{Sym.(E)} & \gray 8.8 & \gray \textbf{32.7} & \gray 25.1 & \gray \underline{26.1} \\
        & \makecell[c]{Com.} & \gray 40.0 & \gray \textbf{45.7} & \gray \underline{42.3} & \gray 33.4    \\
        \midrule
         \multirow{3}{*}{\makecell[c]{Gemini}} &
        \makecell[c]{Base-9} & 12.7 & \textbf{88.0} & \underline{72.3} & 21.2 \\
        & \makecell[c]{Sym.(E)} & 9.3 & \textbf{44.5} & \underline{38.9} & 36.7 \\
        & \makecell[c]{Com.} & 42.9 & \textbf{55.6} & \underline{53.2} & 33.5    \\
        \midrule
         \multirow{3}{*}{\makecell[c]{Llama2}} &
        \makecell[c]{Base-9} & 1.7 & \textbf{4.9} & \underline{2.9} & 2.7 \\
        & \makecell[c]{Sym.(E)}  & 4.7 & \textbf{10.1} & 8.7 & \underline{9.1} \\
        & \makecell[c]{Com.} & 35.0 & \textbf{42.3} & \underline{41.9} & 40.2    \\
        \midrule
         \multirow{3}{*}{\makecell[c]{Mixtral}} &
        \makecell[c]{Base-9} & 3.9 & \textbf{27.5} & \underline{16.3} & 3.7 \\
        & \makecell[c]{Sym.(E)} & 8.3 & \textbf{19.3} & \underline{17.9} & 15.1 \\
        & \makecell[c]{Com.} & 24.2 & \textbf{37.5} & \underline{34.9} &  31.1   \\
        \bottomrule
    \end{tabular}
    \caption{
    Comparing LLMs with 0-shot, 3-shot clean, and 3-shot 
    medium irrelevant (irr.) / inaccurate (ina.) rationales.}
    \label{tab:different_llms_on_NoRa}
\end{minipage}%
\vspace{-10pt}
\end{table*}

\begin{table*}[t!]
    \centering
    \fontsize{7}{7}\selectfont
    \setlength\tabcolsep{6pt}
    \begin{tabular}{c|c|c|c|c|c}
        \toprule
        \multirow{1}{*}{Task} & \multirow{1}{*}{Zero-shot} & \multirow{1}{*}{Few-shot (No Shuffle)} 
        & \multirow{1}{*}{Shuffle Questions $x_i$} &  \multirow{1}{*}{Shuffle Rationales $\mathcal{T}_i$} & \multirow{1}{*}{Shuffle Answers $y_i$}   \\
        \midrule
        \multirow{1}{*}{\makecell[c]{Math Base-9}} & 7.2 &  \gray \textbf{46.4} 
        & \underline{45.5} (0.9\%$\downarrow$) & 34.5 (11.9\%$\downarrow$) & 35.7 (10.7\%$\downarrow$)   \\
        \midrule
        \multirow{1}{*}{\makecell[c]{Math Base-11}} & 5.5 & \gray \underline{23.9}	
        & \textbf{24.8} (0.9\%$\uparrow$) & 21.6 (2.3\%$\downarrow$) & 21.1 (11.7\%$\downarrow$)  \\
        \midrule
        \multirow{1}{*}{\makecell[c]{Symbolic Equal}} & 8.8 & \gray \underline{32.7} 
        & \underline{32.7} (0.0\%$\downarrow$) & \textbf{32.8} (0.1\%$\uparrow$) & 32.3 (0.4\%$\downarrow$) \\
         \midrule
        \multirow{1}{*}{\makecell[c]{Symbolic Longer}} & 0.0 & \gray \textbf{9.2} 
        & \underline{7.0} (2.2\%$\downarrow$) & 6.2 (3.0\%$\downarrow$) & 6.3 (2.9\%$\downarrow$) \\
         \midrule
        \multirow{1}{*}{Commonsense}     
        & 40.0 & \gray \textbf{45.7} 
        & 38.7 (7.0\%$\downarrow$) & 39.7 (6.0\%$\downarrow$) & \underline{39.8} (5.9\%$\downarrow$) \\
        \bottomrule
    \end{tabular}
    \caption{
    Performance (in accuracy\%) on NoRa dataset under different few-shot shuffle configurations. 
    }
    \label{tab:shuffle study}
\vspace{-20pt}
\end{table*}

\textbf{Unreliability revealing with noisy rationales.}
We conduct the reasoning tasks on LLM with Noisy-R and summarize the results in Tab.~\ref{tab:vanilia_LLM}.
\textit{Overall, the base LLM with all the existing reasoning methods is severely affected by irrelevant or inaccurate noise, with overall showing a $0.2\%$-$25.3\%$ decrease with irrelevant noise and a more drastic $0.1\%$-$54.0\%$ decrease with inaccurate noise compared with clean rationales.}
While robust methods like SP and SD exhibit resilience to noise on partial tasks, their performance remains inconsistent and often declines.
To further reveal the unreliability, we start by analyzing the two categories of robust methods mentioned above in the following observations.

\textbf{\textit{Observation 4.1.} Self-correction methods perform poorly on most tasks with noisy rationales.} 
Therein, ISC~\citep{huang2023large} and SP~\cite{xi2023self} rely on the inherent capabilities of LLMs to enhance the quality of generated responses. 
However, in the absence of external feedback, the model's self-correction ability in reasoning tasks is limited, often resulting in the miscorrection of the given content (see Tab.~\ref{tab:comparison-rephrased-rationales}).
SP can only slightly improve the accuracy of commonsense tasks, while ISC performs unsatisfactorily across all tasks.
As can be seen from Tab.~\ref{tab:vanilia_LLM}, these methods perform even worse than the base model.

\textbf{\textit{Observation 4.2.} Self-consistency methods can improve robustness without true denoising.} 
Two self-consistency approaches, SM~\citep{robey2023smoothllm} and SD~\citep{zhang2023certified}, are originally proposed to address Noisy-Q issues.
When applied to our Noisy-R scenarios, they tend to easily disrupt the intrinsic logical coherence within the thought chain.
Although these methods utilizing smooth strategies (\textit{e.g.}, random smoothing or masking) perform well on the commonsense dataset, they can hardly handle the more difficult reasoning tasks and even degenerate close to 0\%, \textit{e.g.}, in the Symbolic Longer task.
Another method, SC~\citep{wang2023self}, performs better than the base model in all tasks, improving both clean and noisy reasoning performance.
However, SC does not conduct explicit denoising on rationales during its reasoning procedure. 
In addition, SC also requires a high computation cost (refer to Appendix~\ref{app:exp-quantitative}).

Besides these methods, next, we analyze LLMs' \textit{intrinsic} properties under noisy rationales as follows.

\textbf{\textit{Observation 4.3.} Adjusting model temperature can help reasoning under noisy rationales.} 
In Tab.~\ref{tab:temperatures}, we evaluate the base LLM using different temperatures on 3-shot demonstrations. Overall, reducing temperature can enhance the model's accuracy under both noisy and clean rationale reasoning, compared to the default temperature of 1.
However, the relationship between temperature and accuracy is not linear for noisy reasoning; instead, there are multiple peaks in accuracy within the temperature range of $0$ to $1$.
Additionally, it is found that excessively low temperatures (\textit{e.g.}, 0) tend to result in verbose and repeated responses, which cause the model to exceed token limits up to $30\%$ in symbolic tasks where the length of expected answers is quite variable among different questions.

\textbf{\textit{Observation 4.4.} Prompting with more noisy examples boosts reasoning accuracy on most tasks.}
In Tab.~\ref{tab:increasing_exemplars}, we evaluate the model using different numbers of exemplars while keeping the temperature at 1.
In general, the LLM's accuracy will still improve as the number of noisy examples increases in the clean and noisy settings.
However, it should be noted that in tasks with high-level noise from NoRa-Math, increasing prompting examples can degenerate accuracy. 
For example, in the base-9 inaccurate-hard task, prompting with noisy rationales is even worse than the 0-shot accuracy of $7.2\%$. 
Further, we provide a deeper analysis of increasing the number of noisy examples in Appendix \ref{app: number of noisy examples}.

\textbf{\textit{Observation 4.5.} Different LLMs are generally vulnerable to noisy rationales.}
In Tab.~\ref{tab:different_llms_on_NoRa}, we evaluate different LLMs across three settings: 0-shot CoT, 3-shot clean rationales, and 3-shot medium-level noisy rationales.
Notably, Gemini-Pro outperforms GPT-3.5 in overall performance. 
However, it demonstrates a similar degree of sensitivity to noise, with a $2.4\%$-$15.7\%$ performance decline with irrelevant rationales and a $7.8\%$-$66.8\%$ decline with inaccurate rationales compared to clean rationales.
While Mixtral 8x7B shows a slight underperformance compared to GPT-3.5, it also manifests a vulnerability to noise, incurring a $1.4\%$-$11.2\%$ loss with irrelevant rationales and a greater $4.2\%$-$23.8\%$ loss with inaccurate rationales.
By contrast, Llama2-70B performs suboptimally, with a $0.4\%$-$2.0\%$ drop for irrelevant thoughts and a larger $1.0\%$-$2.2\%$ drop for inaccurate thoughts.

\textbf{Further investigation.}
Inspired by \citet{min2022rethinking}, we further explore the mapping among questions, rationales, and answers through shuffling experiments. 
Specifically, given the 3-shot prompting examples $\{(x_1, \mathcal{T}_1, y_1), (x_2, \mathcal{T}_2, y_2), (x_3, \mathcal{T}_3, y_3)\}$, 
we test three configurations, 
\textit{i.e.}, 
shuffle questions $\{(x_1, \mathcal{T}_3, y_3), (x_2, \mathcal{T}_1, y_1), (x_3, \mathcal{T}_2, y_2)\}$,
shuffle answers $\{(x_1, \mathcal{T}_1, y_3), (x_2, \mathcal{T}_2, y_1), (x_3, \mathcal{T}_3, y_2)\}$, 
and shuffle rationales $\{(x_1, \mathcal{T}_3, y_1), (x_2, \mathcal{T}_1, y_2), (x_3, \mathcal{T}_2, y_3)\}$.
These break the original mappings.
The results under these configurations are shown in Tab.~\ref{tab:shuffle study}, which induces the following observation.

\textbf{\textit{Observation 4.6.} Shuffling the mappings of prompting examples degenerates the reasoning but still performs better than without prompting.}
This means that while LLMs may not heavily rely on the exact mapping (of question, rationale, and answer), they still benefit from demonstrating information even with shuffling.
Notably, this finding is consistent with the conclusions of~\cite{min2022rethinking} that LLMs learn more abstract task information from the demonstrations rather than simply memorizing question-answer pairs.
\textit{More importantly, LLMs are less vulnerable to shuffled mappings than noisy rationales}.
Unlike shuffling, the irrelevant or inaccurate information in noisy rationales introduces misleading elements that significantly interfere with the model's ability to learn \textit{correct} task patterns, thereby resulting in more severe performance degradation. 
This extends \cite{min2022rethinking}'s finding and shows that the quality of reasoning steps can be more crucial than the exact mapping of prompting examples.

\begin{figure}[t!]
\begin{minipage}[t]{0.49\linewidth}
\centering
\includegraphics[width=\textwidth]{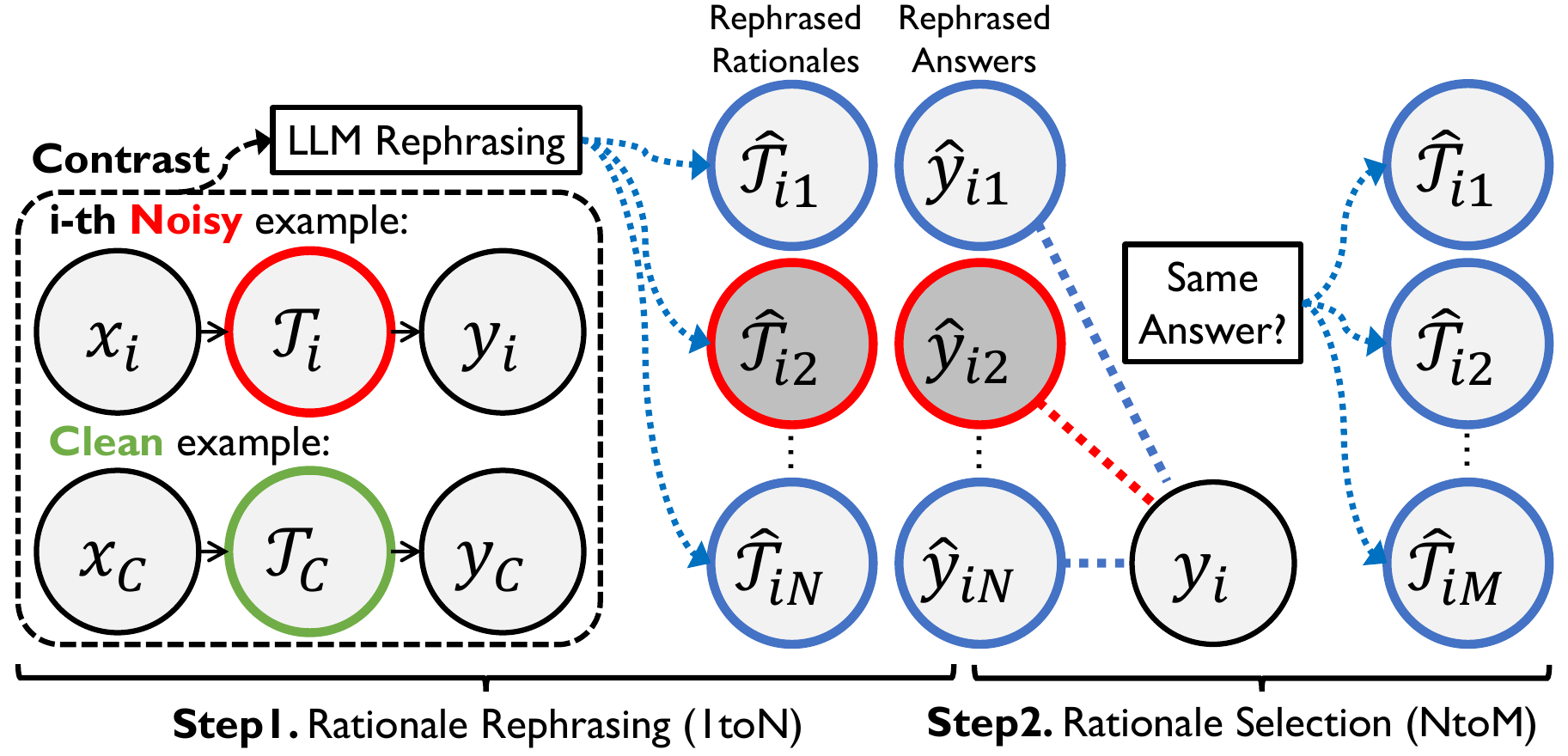}
\caption{CD-CoT's first two steps for data denoising.
First, it rephrases the $i$-th noisy example by contrasting it with the clean example. Then, with the obtained $N$ rephrased examples, it selects the $M$ qualified candidates by checking the validity of the rephrased answers $\hat{y}_{i1}, \ldots, \hat{y}_{iN}$ \textit{w.r.t.} $y_i$.}
\label{fig: method-step12}
\end{minipage}
\hfill
\begin{minipage}[t]{0.49\linewidth}
\centering
\includegraphics[width=\textwidth]{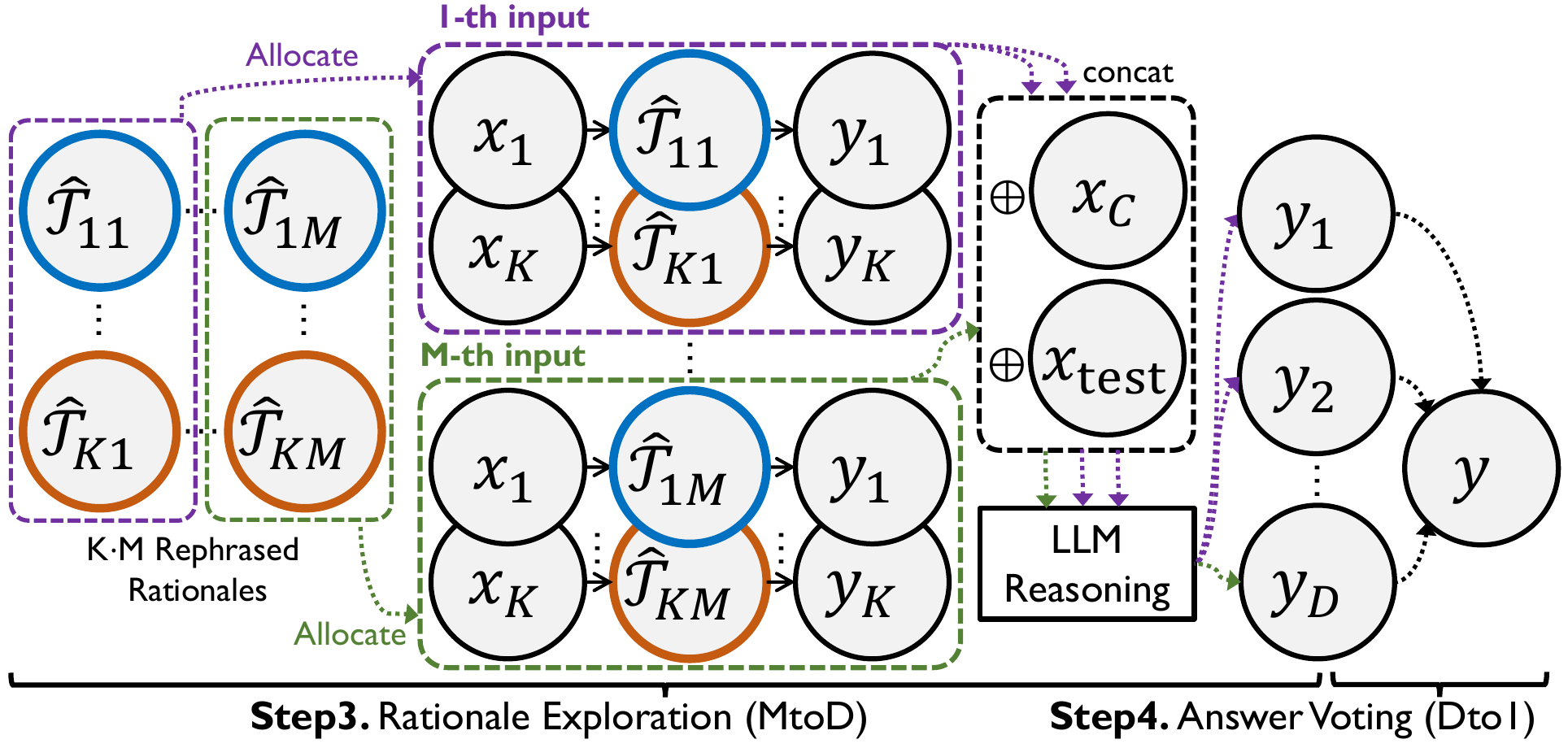}
\caption{CD-CoT constructs $M$ inputs ($K$-shot) by allocating the $K \cdot M$ rephrased rationales.
These inputs are concatenated with the clean example and test question and then fed to an LLM for reasoning separately.
The obtained $D$ answers are equally voted to obtain the final answer $y$.}
\label{fig: method-step34}
\end{minipage}
\end{figure}

\vspace{-20pt}
\section{Method}
\label{sec: method}
\vspace{-14pt}

This section aims to enable LLMs to discern and remove noisy thoughts.
The observations in Sec.~\ref{sec: experiment} and previous works show that current LLMs cannot achieve this with their intrinsic denoising ability, even enhanced with self-denoising methods. 
Therefore, we would claim that the \textit{external supervision} is necessary for enhancement, which should be sufficient for denoising and accessible in practice.
Existing methods with external supervision~\cite{huang2023large, tyen2023llms, chia2023contrastive} require (1) oracle feedback on the test question, (2) human feedback of errors on specific tokens or positions, or (3) expert knowledge to construct detailed descriptions of specific tasks.
By contrast, we believe that \textit{a clean CoT demonstration} is more attainable and practical, which can be the \textit{minimal requirement} for denoising-purpose prompting.

Therein, we assume that LLMs \textit{can} identify noisy thoughts 
by \textit{contrasting} a pair of noisy and clean rationales and discerning their differences, similar to contrastive learning~\cite{he2020momentum, chen2020simple, khosla2020supervised}.
Here, the denoising power could come from the abilities of the instruction following and step-by-step reasoning~\cite {wei2022emergent, schaeffer2023emergent}. 
Hence, we propose the framework of CD-CoT, \textbf{C}ontrastive \textbf{D}enoising with noisy \textbf{CoT}.
The design principle is to explore and then exploit, \textit{i.e.}, (1) rephrasing and selecting rationales in input space to achieve explicit denoising, and then (2) exploring diverse rationales and voting answers in output space for deriving the final answer, as in Figs.~\ref{fig: method-step12}~\&~\ref{fig: method-step34}.
The details are as follows.


\begin{algorithm}[t!]
\setstretch{0.95}
\caption{CD-CoT: Contrastive Denoising with Noisy Chain-of-Thought.}
\begin{algorithmic}[1]
\REQUIRE
an LLM $f_{\theta}$,
the prompt of contrastive denoising $\mathcal{P}_{\text{denoise}}$,
one test question $x_\text{test}$,
one clean example $(x_\text{C}, \mathcal{T}_\text{C}, y_\text{C})$,
K prompting examples $S_{n} \! = \! \{(x_{i}, \mathcal{T}_i, y_{i})\}_{i=1}^{K}$,
hyper-parameters $N, M$, 
and reasoning budget $\{B_i\}_{i=1}^{M}$ (satisfies that $\Sigma_{i=1}^{M} B_i = D$, where $D$ is the total budget).

\FOR{$i=1\dots K$}
    \STATE initialize the set of rephrased results of $i$-th example $\mathcal{R}_i \leftarrow \emptyset$.
    \FOR{$j=1\dots N$}
        \STATE {\noisy{\texttt{\# Step-1: Rationale Rephrasing via Supervised Contrasting}}}
        \STATE obtain a rephrased example as $(x_{i}, \hat{\mathcal{T}}_i, \hat{y}_{i}) \leftarrow f_{\theta}
        \Big(\mathcal{P}_{\text{denoise}}
        (x_\text{C}, \mathcal{T}_\text{C}, y_\text{C}, x_{i}, \mathcal{T}_i, y_{i})
        \Big)$.
        \STATE 
        if match answer
        $\hat{y}_{i} = y_{i}$, then store the rephrased example as $\mathcal{R}_i \leftarrow \mathcal{R}_i \cup \{(x_{i}, \hat{\mathcal{T}}_i, \hat{y}_{i})\}$.
    \ENDFOR    
    \STATE {\noisy{\texttt{\# Step-2: Rationale Selection}}}
    \STATE randomly select $M$ rephrased examples from $\mathcal{R}_i$ and obtain $\tilde{\mathcal{R}}_i = \{(x_{is}, \hat{\mathcal{T}}_{is}, \hat{y}_{is})\}_{s=1}^{M}$.
\ENDFOR

\STATE {\noisy{\texttt{\# Step-3: Rationale Exploration}}}
\STATE initialize the set of answers $\mathcal{Y} \leftarrow \emptyset$.
\FOR{$i=1\dots M$}
    \STATE construct an input $\mathcal{P}_i \leftarrow \{ (x_{ji}, \hat{\mathcal{T}}_{ji}, \hat{y}_{ji}) \}_{j=1}^{K}$, where
    $(x_{ji}, \hat{\mathcal{T}}_{ji}, \hat{y}_{ji})$ is the $i$-th element of $\hat{\mathcal{R}}_j$.
    \STATE concatenate $\mathcal{P}_i$ with the clean example and test question as $\mathcal{P}_i \leftarrow \mathcal{P}_i \cup \{(x_\text{C}, \mathcal{T}_\text{C}, y_\text{C}), x_\text{test}\}$.
    \FOR{$j=1\dots B_M$}
        \STATE get one answer by LLM reasoning as $y_j \leftarrow f_{\theta}(\mathcal{P}_i)$.
        \STATE store the answer as $\mathcal{Y} \leftarrow \mathcal{Y} \cup \{y_j\}$.
    \ENDFOR
\ENDFOR

\STATE {\noisy{\texttt{\# Step-4: Answer Voting}}}
\STATE initialize the dictionary of answer count $\mathcal{C}$ 
that $\forall y_j \in \mathcal{Y}, \mathcal{C}[y_j] = 0$.
\FOR{$j=1\dots D$}
    \STATE update $\mathcal{C}[y_j] \leftarrow (\mathcal{C}[y_j] + 1)$.
\ENDFOR
\STATE get the final answer $y$ with maximum counts as $y \leftarrow \arg \max_{y} \mathcal{C}[y]$.
\STATE \textbf{return } the answer $y$.

\end{algorithmic}
\label{alg: CD-CoT}
\end{algorithm}
\setlength{\textfloatsep}{0.2cm}


\subsection{Implementation } 
\label{implementation}
\vspace{-6pt}

\textbf{Step-1: Rephrasing via Supervised Contrasting ($1$ to $N$).}
First, we establish a general prompt of contrastive rephrasing to construct a pair of contrastive examples, as shown in the template below. This steers the model towards learning from the clean example and then rephrasing and rectifying the noisy examples.
To be specific, given one clean example and $K$ noisy examples, we generate $N$ rephrased rationales for each noisy example independently and obtain $K \! \cdot \! N$ rephrased rationales.

\vspace{-6pt}
\begin{center}
\begin{tcolorbox}[
colback=white,
colframe=black,
width=1.0\linewidth,
arc=1mm, auto outer arc,
boxrule=0.75pt]

\footnotesize 

\textbf{Prompt of Contrastive Rationale Rephrasing:} \\
Here are two examples for the same type of task: the first example has correct explanation and correct answer, and the second example has distracted explanation and correct answer. Please follow the first example and give me the correct explanation and answer for the second example, which should be logically consistent with the first one.

\textbf{First Example:} {\scriptsize Q: [Question], E: [Explanation], A: [Answer].}

\textbf{Second Example:} {\scriptsize Q: [Question], E: [Explanation], A: [Answer].}

\end{tcolorbox}
\end{center}
\vspace{-6pt}

\textbf{Step-2: Rationale Selection ($N$ to $M$, $N \! \geq \! M$).}
Next, we employ answer matching to select those rephrased examples with unchanged answers, leaving behind a refined candidate pool.
Subsequently, we randomly select $M$ rephrased rationales from the pool and concatenate them to form the contexts.

\textbf{Step-3: Rationale Exploration ($M$ to $D$, $M \! \leq \! D$).}
For the $M$ different contexts, we explore rationales by repeated reasoning with the budget of $D$ reasoning repetitions.
Notably, a higher temperature parameter, \textit{e.g.}, 1, is set to introduce more randomness in generating diverse rationales.

\textbf{Step-4: Answer Voting ($D$ to 1).}
Ultimately, all the $D$ answers are equally voted into a final answer.


\textbf{Instantiation.}
By tuning the hyper-parameters $ N$, $ M$, and $ D$, we balance exploration and exploitation in the input and output space. 
The overall procedure of our proposed CD-CoT is presented in Algorithm~\ref{alg: CD-CoT}. 
Besides, we further explain the details of each step of this algorithm in Appendix~\ref{app:implementation-CD-CoT}.

\textbf{Theoretical analysis.}
To understand the underlying mechanism of CD-CoT, we also conduct the theoretical analysis based on the distinguishability~\citep{xie2021explanation} of in-context learning. 
The full analysis is in Appendix~\ref{app:theo}, where we find that the noisy demonstration in ICL can decrease the distinguishability of in-context matching with the clean-prompt distribution, while our method can mitigate this issue. 
Besides, we build a self-supervised variant of CD-CoT and empirically evaluate it in Appendix~\ref{app:variant of CD-CoT}.

\vspace{-6pt}
\subsection{Empirical Study}
\label{sec: Empirical Study}

In this part, we empirically verify the effectiveness of CD-CoT and start by introducing the baselines.

\textbf{Baseline methods.}
We employ three methods that require  \textit{additional information}: (1) Self-Correction with Oracle Feedback~(SCO)~\cite{huang2023large} utilizes the \textit{ground truth answers} of test questions to determine when to terminate the self-correction loop; (2) Backtracking~(BT)~\cite{tyen2023llms} guides self-correction by providing the model with the \textit{position} of the first noisy thought; (3) Contrastive Chain-of-Thought~(CC)~\cite{chia2023contrastive} conducts direct reasoning with all the noisy or \textit{clean examples} without implicit or explicit denoising.

\textbf{Main results.}
As in Tab.~\ref{tab:proposed-method performance}, CD-CoT demonstrates a significant performance improvement across all datasets, with an average improvement of $17.8\%$ compared with the base model under noisy settings.
Notably, on Math-Base-9, Math-Base-11, and Symbolic-Equal, CD-CoT surpasses all baseline methods by a significant margin.
On Symbolic-Longer and Commonsense, CD-CoT only slightly lags behind SCO.
However, SCO requires the ground truth answer to the test question, which should be unknown in practice, as pointed out in~\citep{huang2023large}.
In comparison, CD-CoT only necessitates an additional clean demonstration, making it much more practical to apply across realistic scenarios.
\textit{Notably, CD-CoT outperforms SCO in $20$ out of $30$ settings and surpasses BT, CC in all $30$ settings}.

Besides, CD-CoT displays \textit{remarkable resistance} to the magnitude of noise. Therein, CD-CoT demonstrates enhanced resilience against inaccurate noise on mathematical tasks, which are quite challenging.
For instance, on Math Base-9 with inaccurate rationales, the average accuracies of SCO and BT decline significantly by $28.8\%$ and $26.3\%$ compared to the accuracies with clean rationales. 
In contrast, CD-CoT exhibits a more modest decline of $7.0\%$.
An ablation study of components in Appendix~\ref{appendix: ablation study on CD-CoT} demonstrates the denoising power and performance gain of CD-CoT, attributed to its contrastive denoising with rationale rephrasing as well as repeated reasoning with voting components.

\begin{table*}[t!]
    \centering
    \fontsize{7}{7}\selectfont
    \setlength\tabcolsep{5.2pt}
    \resizebox{\linewidth}{!}{
    \begin{tabular}{c|l|c|c|cccc|cccc}
        \toprule
        \multirow{2}{*}{Task} & \multicolumn{1}{c|}{\multirow{2}{*}{Method $\mathcal{M}$}} 
        &\multirow{2}{*}{\makecell[c]{Additional \\ Information}}  
        &\multirow{2}{*}{\makecell[c]{$\operatorname{Acc}(\mathcal{M}, \mathcal{Q}, \mathcal{P}_{\text{clean}})$}} &  \multicolumn{4}{c|}{$\operatorname{Acc}(\mathcal{M}, \mathcal{Q}, \mathcal{P}_{\text{irrelevant}})$} &  \multicolumn{4}{c}{$\operatorname{Acc}(\mathcal{M}, \mathcal{Q}, \mathcal{P}_{\text{inaccurate}})$}\\
            &   &  &  &    Easy & Medium & Hard & Avg. & Easy & Medium & Hard & Avg. \\
        \midrule
        \multirow{5}{*}{\makecell[c]{Math \\ Base-9}} 
            &  \makecell[c]{Base}                 & - & \gray 46.4             & \gray  39.3           & \gray   30.3               & \gray  26.6 & \gray  32.1               & \gray  23.2             & \gray  10.1             & \gray  6.0           & \gray  13.1    \\
            &  w/ SCO~\cite{huang2023large}  & Ground Truth
            & \baseline \underline{53.6}    & \baseline \underline{46.3} & \baseline 39.6   & \baseline 36.4    & \baseline 40.8     & \baseline 34.7     & \baseline 22.0     & \baseline 17.7     & \baseline 24.8 \\
            &  w/ BT~\cite{tyen2023llms}  
            & Noise Position 
            & 47.2    & 39.2  & 34.2  & 29.9   & 34.4     & 30.1   & 18.4     & 14.1    & 20.9 \\
            &  w/ CC~\cite{chia2023contrastive}  
            & Clean Demo    
            & 44.9     & 43.3     &  \underline{44.6}    & 
            \underline{45.5}    & \underline{44.5}    & \underline{37.2}   & \underline{31.7}  & \underline{30.7}   & \underline{33.2} \\
            &  w/ CD-CoT (ours)         
            & Clean Demo 
            & \ours \textbf{60.7}      & \ours \textbf{59.7}     & \ours \textbf{60.7}    & \ours \textbf{57.2}    & \ours \textbf{59.2}     & \ours \textbf{54.0}   & \ours \textbf{58.7}     & \ours  \textbf{48.4}    & \ours \textbf{53.7}    \\
        \midrule
        \multirow{5}{*}{\makecell[c]{Math \\ Base-11}}
             &  \makecell[c]{Base}   & -                  & \gray  23.9
             & \gray  19.1
             & \gray  13.6 
             & \gray  10.7          
             & \gray  14.5                  & \gray  14.0         
             & \gray  6.7
             & \gray  3.6    & \gray  8.1 \\
            &  w/ SCO~\cite{huang2023large} 
            & Ground Truth     
            & \baseline \textbf{33.0} & \baseline \underline{29.2}  & \baseline \underline{24.0}  & \baseline \underline{20.0}    & \baseline \underline{24.4}     & \baseline \textbf{29.2}  & \baseline \underline{20.0}    & \baseline
            \underline{17.2}    & \baseline \underline{22.1} \\
            &  w/ BT~\cite{tyen2023llms} 
            & Noise Position 
            & 24.3    & 17.9    & 17.2   & 13.7    & 16.3     &12.8   & 9.2   & 6.8   & 9.6 \\
            & w/ CC~\cite{chia2023contrastive} 
            & Clean Demo
            & 22.3     & 19.1        & 18.4        & 18.2        & 18.6     & 19.0       & 15.3      & 14.6       & 16.3 \\
            & w/ CD-CoT (ours)   
            & Clean Demo       
            &  \ours \underline{31.0}      & \ours \textbf{33.7}    & \ours \textbf{32.7}   & \ours \textbf{34.7}    & \ours \textbf{33.7}
            & \ours  \underline{29.0}  & \ours \textbf{30.7}    & \ours \textbf{25.3}    & \ours \textbf{28.3}    \\ 
        \midrule
        \multirow{5}{*}{\makecell[c]{Symbolic \\ Equal}}
            &  \makecell[c]{Base}    
            & - 
            &  \gray 32.7  & \gray   28.1 & \gray  25.1  & \gray  23.0  & \gray  25.4  & \gray   29.1 & \gray  26.1  & \gray  22.7              & \gray  26.0 \\
            & w/ SCO~\cite{huang2023large}  
            & Ground Truth 
            & \baseline \underline{38.5} & \baseline \underline{34.9}    & \baseline \underline{33.4}   & \baseline 
            \underline{32.7}  & \baseline \underline{33.7}     & \baseline \underline{34.0}    & \baseline \underline{34.1}    & \baseline \underline{34.5}    & \baseline \underline{34.2} \\
            & w/ BT~\cite{tyen2023llms}  
            & Noise Position     
            & 31.8 & 26.0    & 22.7    & 22.6 & 23.8     & 26.3    & 22.7    &  22.9    & 24.0 \\
            & w/ CC~\cite{chia2023contrastive}  
            & Clean Demo      
            & 37.8 &  33.8   & 32.7    & 32.0 & 32.8     &  31.3   & 33.0    & 29.9    & 31.4 \\
            & w/ CD-CoT (ours)     
            & Clean Demo     
            & \ours \textbf{42.7}     & \ours \textbf{44.7}   & \ours \textbf{42.7}   & \ours \textbf{44.0}   & \ours \textbf{43.8}     & \ours  \textbf{42.6}  & \ours 
            \textbf{41.3}    & \ours 
            \textbf{42.7}   & \ours \textbf{42.2}    \\ 
         \midrule
        \multirow{5}{*}{\makecell[c]{Symbolic \\ Longer}}
        &  \makecell[c]{Base} 
        & - & \gray   9.2     & \gray   6.3  & \gray  7.2  & \gray  6.0   & \gray  6.5   & \gray  7.0 & \gray  6.8  & \gray  6.0   & \gray  6.6 \\
            & w/ SCO~\cite{huang2023large}  
            & Ground Truth     
            & \baseline \textbf{18.7}  & \baseline \textbf{12.1}  & \baseline \underline{10.5}  & \baseline \underline{11.3}  & \baseline \underline{11.3}     & \baseline \textbf{15.2}   & \baseline \textbf{15.9}   & \baseline \underline{9.8}   & \baseline \textbf{13.6} \\
            &  w/ BT~\cite{tyen2023llms}  
            & Noise Position     
            & 7.2   & 3.4   & 3.5   & 2.5   & 3.1     & 3.8   & 3.6   & 3.6    & 3.7 \\
            & w/ CC~\cite{chia2023contrastive}  
            & Clean Demo    
            & 9.4   & 9.8   & 7.9   & 7.9   & 8.5    & 8.5  & 7.4 & 6.5 & 7.5  \\
            & w/ CD-CoT (ours)    
            & Clean Demo      
            & \ours \underline{12.3}      & \ours \underline{12.0}   & \ours \textbf{12.0}   & \ours \textbf{13.0}  & \ours \textbf{12.3}     & \ours \underline{12.3} & \ours \underline{10.0}   & \ours \textbf{11.0}  & \ours \underline{11.1}    \\ 
         \midrule
        \multirow{5}{*}{Commonsense}     &  \makecell[c]{Base}         & -                    & \gray   45.7           & \gray  44.3              & \gray  42.3              & \gray  41.4              & \gray  42.7              & \gray  36.7              & \gray  33.4              & \gray  28.3 & \gray  32.8 \\
            & w/ SCO~\cite{huang2023large}  
            & Ground Truth      
            & \baseline \textbf{63.5}   & \baseline \textbf{60.1}
            & \baseline \textbf{56.1}    & \baseline \textbf{60.3}   & \baseline \textbf{58.8}   & \baseline  \textbf{56.2}  & \baseline \textbf{58.5}   & \baseline \textbf{57.9}   & \baseline \textbf{57.5} \\
            & w/ BT~\cite{tyen2023llms}     
            & Noise Position  
            & 47.7    & 23.5    & 28.3   & 32.5    & 28.1     & 11.6    & 11.0   & 15.8   & 12.8 \\
            & w/ CC~\cite{chia2023contrastive}  
            & Clean Demo   
            &  48.3 & 45.7 & 43.6  & 44.0 & 44.4  & 42.1 & 40.8 & 40.5 & 41.1 \\
            & w/ CD-CoT (ours) 
            & Clean Demo         
            & \ours \underline{49.0}      & \ours \underline{50.3}    & \ours \underline{54.7}   & \ours \underline{50.3}    & \ours \underline{51.8}     & \ours  \underline{51.0}  & \ours \underline{49.7}    & \ours \underline{49.7}    & \ours \underline{50.1}    \\ 
        \bottomrule
    \end{tabular}
    }
    \vspace{-4pt}
    \caption{Performance of denoising methods that require additional information for supervision.}
    \label{tab:proposed-method performance}
\end{table*}

\begin{table}[t!]
    \centering
    \begin{minipage}{0.56\linewidth}
        \centering
        \renewcommand\arraystretch{1.0}
        \fontsize{7}{7}\selectfont
        \setlength\tabcolsep{3pt}
        \vspace{-4pt}
        \begin{tabular}{cccc|ccc|ccc}
            \toprule
            \multicolumn{4}{c|}{Hyper-parameters} & 
            \multicolumn{3}{c|}{$\operatorname{Acc}(\mathcal{M}, \mathcal{Q}, \mathcal{P}_{\text{irrelevant}})$} &
            \multicolumn{3}{c}{$\operatorname{Acc}(\mathcal{M}, \mathcal{Q}, \mathcal{P}_{\text{inaccurate}})$} \\
            $N$ & $M$ & $D$ & $C$ &
            Base-9 & Sym.(E) & Com. &
            Base-9 & Sym.(E) & Com. \\
            \midrule
            5 & 1 & 5 & Y &
            57.7 & 38.7 & 55.3 &
            53.3 & 39.7 & 51.0 \\
            5 & 1 & 5 & N & 
            54.7 & 32.7 & 53.7 &
            47.0 & 32.3 & \textbf{55.7}\\
            5 & 2 & 2+3 & Y & 
            \textbf{60.7} & \textbf{42.7} & 54.7 &
            \textbf{58.7} & \underline{41.3} & 49.7 \\
            5 & 2 & 2+3 & N & 
            56.7 & 33.0 & 54.7 &
            49.7 & 32.0 & \underline{53.0} \\
            5 & 3 & 1+2+2 & Y & 
            \textbf{60.7} & 38.7 & 53.3 &
            \underline{58.0} & \textbf{43.3} & 49.0\\
            5 & 3 & 1+2+2 & N & 
            56.0 & 33.3 & \underline{55.7} &
            48.7 & 32.0 & 52.3\\
            5 & 5 & 1 & Y &
            \underline{59.3} & \underline{39.7} & \underline{55.7} &
            \underline{58.0} & 39.0 & 48.7\\
            5 & 5 & 1 & N &
            55.3 & 35.7 & \textbf{55.9} &
            48.7 & 33.3 & 50.7\\
            \bottomrule
        \end{tabular}
        \vspace{5pt}
        \caption{Comparison of accuracy on medium-level tasks.}
        \label{tab:different_instantiation_performance}
        
        \vspace{4pt}
        
        \setlength\tabcolsep{3pt}
        \begin{tabular}{cccc|ccc|ccc}
            \toprule
            \multicolumn{4}{c|}{Hyper-parameters} & 
            \multicolumn{3}{c|}{\#Tokens in step-3 (irr.)} &
            \multicolumn{3}{c}{\#Tokens in step-3 (ina.)} \\
            $N$ & $M$ & $D$ & $C$ &
            Base-9 & Sym.(E) & Com. &
            Base-9 & Sym.(E) & Com. \\
            \midrule
            5 & 1 & 5 & Y &
            1440 & 3162 & 788 &
            1428 & 3170 & 798 \\
            5 & 1 & 5 & N & 
            1301 & 2685 & 660 &
            1295 & 2732 & 667 \\
            5 & 2 & 2+3 & Y & 
            2175 & 4934 & 1269 &
            2156 & 4989 & 1311 \\
            5 & 2 & 2+3 & N & 
            1864 & 4044 & 1005 &
            1842 & 4087 & 1039 \\
            5 & 3 & 1+2+2 & Y & 
            2902 & 6704 & 1772 &
            2878 & 6785 & 1821 \\
            5 & 3 & 1+2+2 & N & 
            2416 & 5360 & 1372 &
            2393 & 5443 & 1420 \\
            5 & 5 & 1 & Y &
            4368 & 10340 & 2764 &
            4339 & 10514 & 2845 \\
            5 & 5 & 1 & N &
            3535 & 8099 & 2088 &
            3506 & 8303 & 2163 \\
            \bottomrule
        \end{tabular}
        \vspace{5pt}
        \caption{Comparison of \#tokens on medium-level tasks.}
        \label{tab:different_instantiation_efficiency}

        \vspace{2pt}
        \setlength\tabcolsep{2pt}
        \begin{tabular}{c|c|ccc|ccc}
            \toprule
            \multirow{2}{*}{Model} &
            \multirow{2}{*}{Method} & 
            \multicolumn{3}{c|}{$\operatorname{Acc}(\mathcal{M}, \mathcal{Q}, \mathcal{P}_{\text{irrelevant}})$} &
            \multicolumn{3}{c}{$\operatorname{Acc}(\mathcal{M}, \mathcal{Q}, \mathcal{P}_{\text{inaccurate}})$} \\
            & &
            Base-9 & Sym.(E) & Com. &
            Base-9 & Sym.(E) & Com. \\
            \midrule
            \multirow{5}{*}{\makecell[c]{GPT-3.5-turbo}} &
            \makecell[c]{Base} & \gray 30.3 & \gray 25.1 & \gray 42.3 & \gray 10.1 & \gray 26.1 & \gray 33.4 \\
            & \makecell[c]{SC} & 36.6 & 28.3 & \underline{45.0} & 17.3 & 30.7 & \underline{44.7}
            \\
            & \makecell[c]{BT} & 34.2 & 22.7 & 28.3 & 18.4 & 22.7 & 11.0
            \\
            & \makecell[c]{CC} & \underline{44.3} & \underline{32.7} & 43.6 & \underline{31.7} & \underline{33.0} & 40.8
            \\
            & \makecell[c]{CD-CoT} & \ours \textbf{60.7} & \ours \textbf{42.7} & \ours \textbf{54.7} & \ours \textbf{58.7} & \ours \textbf{41.3} & \ours \textbf{49.7}
            \\
            \midrule
             \multirow{5}{*}{\makecell[c]{Gemini-Pro}} &
            \makecell[c]{Base} & \gray 72.3 & \gray 38.9 & \gray 53.2 & \gray 21.2 & \gray 36.7 & \gray 33.5 \\
            & \makecell[c]{SC} & 80.3 & \underline{43.3} & \underline{60.0} & 32.3 & \underline{45.0} & 42.7
            \\
            & \makecell[c]{BT} & \underline{82.4} & 29.3 & 37.8 & 26.7 & 28.7 & 33.3
            \\
            & \makecell[c]{CC} & 67.5 & 37.3 & 50.2 & \underline{43.6} & 35.0 & \underline{45.6}
            \\
            & \makecell[c]{CD-CoT} & \ours \textbf{92.7} & \ours \textbf{49.3} & \ours \textbf{57.7} & \ours \textbf{76.7} & \ours \textbf{53.3} & \ours \textbf{55.7}
            \\
            \midrule
             \multirow{5}{*}{\makecell[c]{LLaMA2-70B}} &
            \makecell[c]{Base} & \gray 2.8 & \gray 8.7  & \gray 41.9 & \gray 2.7 & \gray 9.1 & \gray \underline{40.2} \\
            & \makecell[c]{SC} & \textbf{5.0} & 10.3 & \textbf{46.7} & \textbf{3.0} & 9.7 & \textbf{46.0}
            \\
            & \makecell[c]{BT} & 1.4 & \underline{11.2} & 36.1 & 0.9 & \underline{12.5} & 36.2
            \\
            & \makecell[c]{CC} & 1.1 & \textbf{16.3} & 29.9 & \underline{2.8} & \textbf{14.0} & 28.3
            \\
            & \makecell[c]{CD-CoT} & \ours \underline{4.0} & \ours 9.7 & \ours \underline{39.3} & \ours 2.7 & \ours 9.7 & \ours 39.7
            \\
            \midrule
             \multirow{5}{*}{\makecell[c]{Mixtral-8x7B}} &
            \makecell[c]{Base} & \gray 16.3 & \gray 17.9 & \gray 34.9 & \gray 3.7 & \gray 15.1 & \gray 31.1 \\
            & \makecell[c]{SC} & \underline{20.0} & \underline{21.7} & \underline{37.0} & 2.7 & 18.0 & \underline{37.7}
            \\
            & \makecell[c]{BT} & 4.1 & 9.7 & 6.2 & 2.4 & 10.1 & 10.5
            \\
            & \makecell[c]{CC} & \textbf{24.4} & 18.5 & 36.0 & \textbf{12.5} & \underline{18.3} & 35.7
            \\
            & \makecell[c]{CD-CoT} & \ours 8.7 & \ours \textbf{22.7} & \ours \textbf{40.3} & \ours \underline{4.7} & \ours \textbf{21.3} & \ours \textbf{40.3}
            \\
            \bottomrule
        \end{tabular}
        \vspace{5pt}
        \caption{Comparing methods with different LLMs.}
        \label{tab:methods_on_different_llms}
        \vspace{-5pt}
    \end{minipage}
    \begin{minipage}{0.40\linewidth}
        \centering
        \renewcommand\arraystretch{0.8}
        \fontsize{7}{7}\selectfont
        \begin{tabular}{p{5.5cm}}
        \toprule
        \multicolumn{1}{c}{Original Noisy Rationale} \\
        \midrule
        \gray 
        { 
        In base-9, the digits are "012345678". \noisy{\uline{9 + 8 = 17.}} We have 6 + 7 = 13 in base-10. Since we're in base-9, that exceeds the maximum value of 8 for a single digit. 13 mod 9 = 4, so the digit is 4 and the carry is 1. We have 8 + 5 + 1 = 14 in base-10. \noisy{\uline{14 + 8 = 22.}} 14 mod 9 = 5, so the digit is 5 and the carry is 1. A leading digit is 1. So the answer is 154. Answer: 154.}
        \end{tabular}
    
        \vspace{-3pt}
    
        \begin{tabular}{p{0.3cm}|p{4.3cm}}
            \midrule
            \multicolumn{1}{c|}{Method} & \multicolumn{1}{c}{Denoised Rationale} \\
            \midrule
            \multirow{5}{*}{\makecell[c]{SM}} & 
            In base-9, the dkgits are "012345678". \noisy{\uline{9 + 8 = 17.}} We5hav\^ 6r+ 7l=J13 in base-10. Since we're in baze-9, ithFt exceeis the maCimum valur of 8 for a siLgle]digit. x0c3 mod 9 = \^, soWghe digix is 4 anj the carry is 1. We hOve 8 + 5i+ z2= 14 Vn baAe 10. \noisy{\uline{14 d 8 = 2p.}} 14 mod 9 = 5, so Xhe diGit*iU 5 and the carry is 1Q A leading x0cigit is 1. So the answer is 154. Asswer: 154. \\
            \midrule
            \multirow{5}{*}{\makecell[c]{SD}} & 
            In base-9, the digits are "012345678". \noisy{\uline{9 + 8 = 17.}} We have 6 + 7 = 13 in base-10. Since we're in base-9, 13 exceeds the maximum value of 8 for a single digit. 13 mod 9 = 4, so the digit is 4 and the carry is 1. We have 8 + 5 + 1 = 14 in base 10. 14 mod 9 = 8. 14 div 9 = 1, so the digit is 8 and the carry is 1. A leading digit is 1. So the answer is 154. Answer: 154.
     \\
            \midrule
            \multirow{10}{*}{\makecell[c]{SP}} & 
            In base-9, the digits range from 0 to 8. Let's add 86 and 57 step by step.
            First, we add the rightmost digits, which are 6 and 7. In base-10, 6 + 7 equals 13. However, since we are in base-9, the maximum value for a single digit is 8. Therefore, we need to find the remainder when 13 is divided by 9. The remainder is 4, so the digit in the units place is 4 and we have a carry of 1.
            Next, we add the next digits, which are 8 and 5, along with the carry of 1. In base-10, 8 + 5 + 1 equals 14. Again, we need to find the remainder when 14 is divided by 9. The remainder is 5, so the digit in the tens place is 5 and we have a carry of 1. Finally, we have a leading digit of 1. So the final answer in base-9 is 154.
            Answer: 154. \\
            \midrule
            \multirow{5}{*}{\makecell[c]{Ours}} & \ours 
            In base-9, the digits are "012345678". We have 6 + 7 = 13 in base-10. Since we're in base-9, that exceeds the maximum value of 8 for a single digit. 13 mod 9 = 4, so the digit is 4 and the carry is 1. We have 8 + 5 + 1 = 14 in base 10. 14 mod 9 = 5, so the digit is 5 and the carry is 1. A leading digit is 1. So the answer is 154. Answer: 154. \\
            \bottomrule
        \end{tabular}
        \vspace{+1pt}
        \caption{Comparison of rephrased rationales by different reasoning methods.}
        \label{tab:comparison-rephrased-rationales}
        \vspace{-1pt}
    \end{minipage}
\vspace{-4pt}
\end{table}

\textbf{Ablation study of varying hyper-parameters.}
By manipulating the values of $N$, $M$, $D$, and $C$, we generate diverse algorithm instances. 
Here, $D$ denotes the reasoning times allocated to the $M$ inputs, while $C$ signifies whether the clean example is used in step 3.
As demonstrated in Tab.~\ref{tab:different_instantiation_performance}, the clean example utilized by CD-CoT during the reasoning process plays a pivotal role. 
The omission of this clean example results in an average decrease of $3.3\%$ and $4.5\%$ in accuracy under irrelevant noise and inaccurate noise, respectively. 
Besides, the accuracy exhibits subtle variations when employing different algorithm instances, with the highest average accuracy observed at $51.3\%$ and the lowest average accuracy at $49.3\%$. 
Further, Tab.~\ref{tab:different_instantiation_efficiency} presents the average number of tokens used in reasoning.
We set $M \! = \! 2$ to strike a balance. Please refer to Appendix~\ref{app: hyper-parameter pptimization} for detailed hyper-parameter selection. 

\textbf{Ablation study with different LLMs.}
We examine the generalization of CD-CoT across different LLMs. 
As shown in Tab.~\ref{tab:methods_on_different_llms}, with comparably more powerful LLMs, \textit{e.g.}, GPT-3.5-turbo and Gemini-Pro, CD-CoT demonstrates notable improvements in average accuracy. 
It respectively achieves increases of $23.4\%$ and $21.6\%$ in accuracy compared to base models and surpasses all the baselines.

\textbf{Case Study.}
We illustrate the denoising effects of various robust methods using Math Base-9 as an example.
The results in Tab.~\ref{tab:comparison-rephrased-rationales} indicate that the introduction of random characters by SM disrupts the logic of the rationale.
SD fails to eliminate all the noise while recovering the input content, and SP alters the original rationale's reasoning process even when noise removal is successful.
In contrast, CD-CoT significantly removes noisy thoughts and ensures format alignment with the original rationale.
More denoising examples and an entire case study are in Appendix~\ref{app:exp-qualitative} and \ref{app:case}, respectively.

\vspace{-10pt}
\section{Conclusion}
\vspace{-6pt}

In this work, we investigate the under-explored problem of noisy rationales in LLMs. We introduce the NoRa dataset, which tests LLMs against irrelevant or inaccurate thoughts in question-answer scenarios. 
Our findings show LLMs' vulnerability to noisy rationales is inadequately mitigated by existing robust methods. 
We thereby design the CD-CoT method to enhance the robustness via contrastive denoising.
The extension advocates for advancing LLMs by strategies, \textit{e.g.}, external knowledge bases with a retrieval-augmented framework, robust inductive reasoning to extract rules from noisy examples, and multi-modal data integration to enhance the robustness of LLM reasoning.


\clearpage

\section*{Acknowledgements}

ZKZ, RT, JNZ, and BH were supported by Guangdong Basic and Applied Basic Research Foundation Nos. 2022A1515011652 and 2024A1515012399, NSFC General Program No. 62376235, HKBU Faculty Niche Research Areas No. RC-FNRA-IG/22-23/SCI/04, and HKBU CSD Departmental Incentive Scheme.
YWL and ZMW were supported by National Natural Science Foundation of China under Grants 62271357, Natural Science Foundation of Hubei Province under Grants 2023BAB072, and Fundamental Research Funds for the Central Universities under Grants 2042023kf0134.
The authors especially thank Jiangchao Yao, Xuan Li, and Xiao Feng for the constructive discussions and insightful suggestions that helped improve this work.

\bibliography{draft}

\begin{thebibliography}{114}
\providecommand{\natexlab}[1]{#1}
\providecommand{\url}[1]{\texttt{#1}}
\expandafter\ifx\csname urlstyle\endcsname\relax
  \providecommand{\doi}[1]{doi: #1}\else
  \providecommand{\doi}{doi: \begingroup \urlstyle{rm}\Url}\fi

\bibitem[Achiam et~al.(2023)Achiam, Adler, Agarwal, Ahmad, Akkaya, Aleman, Almeida, Altenschmidt, Altman, Anadkat, et~al.]{achiam2023gpt}
Josh Achiam, Steven Adler, Sandhini Agarwal, Lama Ahmad, Ilge Akkaya, Florencia~Leoni Aleman, Diogo Almeida, Janko Altenschmidt, Sam Altman, Shyamal Anadkat, et~al.
\newblock Gpt-4 technical report.
\newblock \emph{arXiv preprint arXiv:2303.08774}, 2023.

\bibitem[Agrawal et~al.(2023)Agrawal, Mackey, and Kalai]{agrawal2023language}
Ayush Agrawal, Lester Mackey, and Adam~Tauman Kalai.
\newblock Do language models know when they're hallucinating references?
\newblock \emph{arXiv preprint arXiv:2305.18248}, 2023.

\bibitem[Aky{\"u}rek et~al.(2023)Aky{\"u}rek, Aky{\"u}rek, Madaan, Kalyan, Clark, Wijaya, and Tandon]{akyurek2023rl4f}
Afra~Feyza Aky{\"u}rek, Ekin Aky{\"u}rek, Aman Madaan, Ashwin Kalyan, Peter Clark, Derry Wijaya, and Niket Tandon.
\newblock Rl4f: Generating natural language feedback with reinforcement learning for repairing model outputs.
\newblock \emph{arXiv preprint arXiv:2305.08844}, 2023.

\bibitem[Brown et~al.(2020)Brown, Mann, Ryder, Subbiah, Kaplan, Dhariwal, Neelakantan, Shyam, Sastry, Askell, et~al.]{brown2020language}
Tom Brown, Benjamin Mann, Nick Ryder, Melanie Subbiah, Jared~D Kaplan, Prafulla Dhariwal, Arvind Neelakantan, Pranav Shyam, Girish Sastry, Amanda Askell, et~al.
\newblock Language models are few-shot learners.
\newblock In \emph{NeurIPS}, 2020.

\bibitem[Chandler and Sweller(1991)]{chandler1991cognitive}
Paul Chandler and John Sweller.
\newblock Cognitive load theory and the format of instruction.
\newblock \emph{Cognition and instruction}, 1991.

\bibitem[Chen et~al.(2020)Chen, Kornblith, Norouzi, and Hinton]{chen2020simple}
Ting Chen, Simon Kornblith, Mohammad Norouzi, and Geoffrey Hinton.
\newblock A simple framework for contrastive learning of visual representations.
\newblock In \emph{ICML}, 2020.

\bibitem[Chen et~al.(2022)Chen, Zhang, Bian, Yang, Ma, Xie, Liu, Han, and Cheng]{chen2022ciga}
Yongqiang Chen, Yonggang Zhang, Yatao Bian, Han Yang, Kaili Ma, Binghui Xie, Tongliang Liu, Bo~Han, and James Cheng.
\newblock Learning causally invariant representations for out-of-distribution generalization on graphs.
\newblock In \emph{NeurIPS}, 2022.

\bibitem[Chen et~al.(2023)Chen, Huang, Zhou, Bian, Han, and Cheng]{chen2023understanding}
Yongqiang Chen, Wei Huang, Kaiwen Zhou, Yatao Bian, Bo~Han, and James Cheng.
\newblock Understanding and improving feature learning for out-of-distribution generalization.
\newblock In \emph{NeurIPS}, 2023.

\bibitem[Chia et~al.(2023)Chia, Chen, Tuan, Poria, and Bing]{chia2023contrastive}
Yew~Ken Chia, Guizhen Chen, Luu~Anh Tuan, Soujanya Poria, and Lidong Bing.
\newblock Contrastive chain-of-thought prompting.
\newblock \emph{arXiv preprint arXiv:2311.09277}, 2023.

\bibitem[Choi et~al.(2018)Choi, He, Iyyer, Yatskar, Yih, Choi, Liang, and Zettlemoyer]{choi2018quac}
Eunsol Choi, He~He, Mohit Iyyer, Mark Yatskar, Wen-tau Yih, Yejin Choi, Percy Liang, and Luke Zettlemoyer.
\newblock Quac: Question answering in context.
\newblock \emph{arXiv preprint arXiv:1808.07036}, 2018.

\bibitem[Chowdhery et~al.(2023)Chowdhery, Narang, Devlin, Bosma, Mishra, Roberts, Barham, Chung, Sutton, Gehrmann, et~al.]{chowdhery2023palm}
Aakanksha Chowdhery, Sharan Narang, Jacob Devlin, Maarten Bosma, Gaurav Mishra, Adam Roberts, Paul Barham, Hyung~Won Chung, Charles Sutton, Sebastian Gehrmann, et~al.
\newblock Palm: Scaling language modeling with pathways.
\newblock \emph{Journal of Machine Learning Research}, 2023.

\bibitem[Cobbe et~al.(2021)Cobbe, Kosaraju, Bavarian, Chen, Jun, Kaiser, Plappert, Tworek, Hilton, Nakano, et~al.]{cobbe2021training}
Karl Cobbe, Vineet Kosaraju, Mohammad Bavarian, Mark Chen, Heewoo Jun, Lukasz Kaiser, Matthias Plappert, Jerry Tworek, Jacob Hilton, Reiichiro Nakano, et~al.
\newblock Training verifiers to solve math word problems.
\newblock \emph{arXiv preprint arXiv:2110.14168}, 2021.

\bibitem[Cohen et~al.(2019)Cohen, Rosenfeld, and Kolter]{cohen2019certified}
Jeremy Cohen, Elan Rosenfeld, and Zico Kolter.
\newblock Certified adversarial robustness via randomized smoothing.
\newblock In \emph{ICML}, 2019.

\bibitem[Del~Vicario et~al.(2017)Del~Vicario, Scala, Caldarelli, Stanley, and Quattrociocchi]{del2017modeling}
Michela Del~Vicario, Antonio Scala, Guido Caldarelli, H~Eugene Stanley, and Walter Quattrociocchi.
\newblock Modeling confirmation bias and polarization.
\newblock \emph{Scientific reports}, 2017.

\bibitem[DeVore and Lorentz(1993)]{devore1993constructive}
Ronald~A DeVore and George~G Lorentz.
\newblock \emph{Constructive approximation}.
\newblock 1993.

\bibitem[Dong et~al.(2022)Dong, Li, Dai, Zheng, Wu, Chang, Sun, Xu, and Sui]{dong2022survey}
Qingxiu Dong, Lei Li, Damai Dai, Ce~Zheng, Zhiyong Wu, Baobao Chang, Xu~Sun, Jingjing Xu, and Zhifang Sui.
\newblock A survey for in-context learning.
\newblock \emph{arXiv preprint arXiv:2301.00234}, 2022.

\bibitem[Floridi and Chiriatti(2020)]{floridi2020gpt}
Luciano Floridi and Massimo Chiriatti.
\newblock Gpt-3: Its nature, scope, limits, and consequences.
\newblock \emph{Minds and Machines}, 2020.

\bibitem[Freeman et~al.(2023)Freeman, Culp, Parisi, Bileschi, Elsayed, Rizkowsky, Simpson, Alemi, Nova, Adlam, et~al.]{freeman2023frontier}
C~Daniel Freeman, Laura Culp, Aaron Parisi, Maxwell~L Bileschi, Gamaleldin~F Elsayed, Alex Rizkowsky, Isabelle Simpson, Alex Alemi, Azade Nova, Ben Adlam, et~al.
\newblock Frontier language models are not robust to adversarial arithmetic, or" what do i need to say so you agree 2+ 2= 5?
\newblock \emph{arXiv preprint arXiv:2311.07587}, 2023.

\bibitem[Gan and Mori(2023)]{gan2023sensitivity}
Chengguang Gan and Tatsunori Mori.
\newblock Sensitivity and robustness of large language models to prompt in japanese.
\newblock \emph{arXiv preprint arXiv:2305.08714}, 2023.

\bibitem[Gebru et~al.(2021)Gebru, Morgenstern, Vecchione, Vaughan, Wallach, Iii, and Crawford]{gebru2021datasheets}
Timnit Gebru, Jamie Morgenstern, Briana Vecchione, Jennifer~Wortman Vaughan, Hanna Wallach, Hal~Daum{\'e} Iii, and Kate Crawford.
\newblock Datasheets for datasets.
\newblock \emph{Communications of the ACM}, 2021.

\bibitem[Gero et~al.(2023)Gero, Singh, Cheng, Naumann, Galley, Gao, and Poon]{gero2023self}
Zelalem Gero, Chandan Singh, Hao Cheng, Tristan Naumann, Michel Galley, Jianfeng Gao, and Hoifung Poon.
\newblock Self-verification improves few-shot clinical information extraction.
\newblock \emph{arXiv preprint arXiv:2306.00024}, 2023.

\bibitem[Gong et~al.(2021)Gong, Wang, Liu, Han, You, Yang, and Tao]{gong2021instance}
Chen Gong, Qizhou Wang, Tongliang Liu, Bo~Han, Jane You, Jian Yang, and Dacheng Tao.
\newblock Instance-dependent positive and unlabeled learning with labeling bias estimation.
\newblock \emph{IEEE Transactions on Pattern Analysis and Machine Intelligence}, 2021.

\bibitem[Griffiths(2020)]{griffiths2020understanding}
Thomas~L Griffiths.
\newblock Understanding human intelligence through human limitations.
\newblock \emph{Trends in Cognitive Sciences}, 2020.

\bibitem[Han et~al.(2018)Han, Yao, Yu, Niu, Xu, Hu, Tsang, and Sugiyama]{han2018co}
Bo~Han, Quanming Yao, Xingrui Yu, Gang Niu, Miao Xu, Weihua Hu, Ivor Tsang, and Masashi Sugiyama.
\newblock Co-teaching: Robust training of deep neural networks with extremely noisy labels.
\newblock In \emph{NeurIPS}, 2018.

\bibitem[Havrilla and Iyer(2024)]{havrilla2024understanding}
Alex Havrilla and Maia Iyer.
\newblock Understanding the effect of noise in llm training data with algorithmic chains of thought.
\newblock \emph{arXiv preprint arXiv:2402.04004}, 2024.

\bibitem[He et~al.(2020)He, Fan, Wu, Xie, and Girshick]{he2020momentum}
Kaiming He, Haoqi Fan, Yuxin Wu, Saining Xie, and Ross Girshick.
\newblock Momentum contrast for unsupervised visual representation learning.
\newblock In \emph{CVPR}, 2020.

\bibitem[He et~al.(2024)He, Zeng, Huang, Chen, Xiao, He, Zhou, Liang, and Xiao]{he2024can}
Qianyu He, Jie Zeng, Wenhao Huang, Lina Chen, Jin Xiao, Qianxi He, Xunzhe Zhou, Jiaqing Liang, and Yanghua Xiao.
\newblock Can large language models understand real-world complex instructions?
\newblock In \emph{AAAI}, 2024.

\bibitem[Huang et~al.(2022)Huang, Gu, Hou, Wu, Wang, Yu, and Han]{huang2022selfimprove}
Jiaxin Huang, Shixiang~Shane Gu, Le~Hou, Yuexin Wu, Xuezhi Wang, Hongkun Yu, and Jiawei Han.
\newblock Large language models can self-improve.
\newblock \emph{arXiv preprint arXiv:2210.11610}, 2022.

\bibitem[Huang et~al.(2024)Huang, Chen, Mishra, Zheng, Yu, Song, and Zhou]{huang2023large}
Jie Huang, Xinyun Chen, Swaroop Mishra, Huaixiu~Steven Zheng, Adams~Wei Yu, Xinying Song, and Denny Zhou.
\newblock Large language models cannot self-correct reasoning yet.
\newblock In \emph{ICLR}, 2024.

\bibitem[Huang et~al.(2023)Huang, Yu, Ma, Zhong, Feng, Wang, Chen, Peng, Feng, Qin, et~al.]{huang2023survey}
Lei Huang, Weijiang Yu, Weitao Ma, Weihong Zhong, Zhangyin Feng, Haotian Wang, Qianglong Chen, Weihua Peng, Xiaocheng Feng, Bing Qin, et~al.
\newblock A survey on hallucination in large language models: Principles, taxonomy, challenges, and open questions.
\newblock \emph{arXiv preprint arXiv:2311.05232}, 2023.

\bibitem[Janis(2008)]{janis2008groupthink}
Irving~L Janis.
\newblock Groupthink.
\newblock \emph{IEEE Engineering Management Review}, 2008.

\bibitem[Jia and Liang(2017)]{jia2017adversarial}
Robin Jia and Percy Liang.
\newblock Adversarial examples for evaluating reading comprehension systems.
\newblock In \emph{EMNLP}, 2017.

\bibitem[Jiang et~al.(2024)Jiang, Sablayrolles, Roux, Mensch, Savary, Bamford, Chaplot, Casas, Hanna, Bressand, et~al.]{jiang2024mixtral}
Albert~Q Jiang, Alexandre Sablayrolles, Antoine Roux, Arthur Mensch, Blanche Savary, Chris Bamford, Devendra~Singh Chaplot, Diego de~las Casas, Emma~Bou Hanna, Florian Bressand, et~al.
\newblock Mixtral of experts.
\newblock \emph{arXiv preprint arXiv:2401.04088}, 2024.

\bibitem[Jiang et~al.(2023)Jiang, Wang, and Wang]{jiang2023selfevolve}
Shuyang Jiang, Yuhao Wang, and Yu~Wang.
\newblock Selfevolve: A code evolution framework via large language models.
\newblock \emph{arXiv preprint arXiv:2306.02907}, 2023.

\bibitem[Khosla et~al.(2020)Khosla, Teterwak, Wang, Sarna, Tian, Isola, Maschinot, Liu, and Krishnan]{khosla2020supervised}
Prannay Khosla, Piotr Teterwak, Chen Wang, Aaron Sarna, Yonglong Tian, Phillip Isola, Aaron Maschinot, Ce~Liu, and Dilip Krishnan.
\newblock Supervised contrastive learning.
\newblock In \emph{NeurIPS}, 2020.

\bibitem[Kim et~al.(2023)Kim, Baldi, and McAleer]{kim2023language}
Geunwoo Kim, Pierre Baldi, and Stephen McAleer.
\newblock Language models can solve computer tasks.
\newblock \emph{arXiv preprint arXiv:2303.17491}, 2023.

\bibitem[Kleijn and van~der Vaart(2012)]{kleijn2012bernstein}
BJK Kleijn and AW~van~der Vaart.
\newblock The bernstein-von-mises theorem under misspecification.
\newblock \emph{Electronic Journal of Statistics}, 2012.

\bibitem[Koh(2022)]{koh2022inversion}
Sin~Yee Koh.
\newblock The inversion of majority/minority at the de/reterritorialised urban higher education enclave: Xiamen university malaysia.
\newblock \emph{Urban Studies}, 2022.

\bibitem[Kojima et~al.(2022)Kojima, Gu, Reid, Matsuo, and Iwasawa]{kojima2022large}
Takeshi Kojima, Shixiang~Shane Gu, Machel Reid, Yutaka Matsuo, and Yusuke Iwasawa.
\newblock Large language models are zero-shot reasoners.
\newblock In \emph{NeurIPS}, 2022.

\bibitem[Koo et~al.(2023)Koo, Park, Lee, Seo, Eo, Moon, and Lim]{koo2023uncovering}
Seonmin Koo, Chanjun Park, Seolhwa Lee, Jaehyung Seo, Sugyeong Eo, Hyeonseok Moon, and Heuiseok Lim.
\newblock Uncovering the risks and drawbacks associated with the use of synthetic data for grammatical error correction.
\newblock \emph{IEEE Access}, 2023.

\bibitem[Lake and Baroni(2018)]{lake2018generalization}
Brenden Lake and Marco Baroni.
\newblock Generalization without systematicity: On the compositional skills of sequence-to-sequence recurrent networks.
\newblock In \emph{ICML}, 2018.

\bibitem[Lei et~al.(2023)Lei, Li, Wei, He, Huang, Zhao, and Liu]{lei2023s}
Fangyu Lei, Xiang Li, Yifan Wei, Shizhu He, Yiming Huang, Jun Zhao, and Kang Liu.
\newblock S $\hat{3}$ hqa: A three-stage approach for multi-hop text-table hybrid question answering.
\newblock \emph{arXiv preprint arXiv:2305.11725}, 2023.

\bibitem[Li et~al.(2023{\natexlab{a}})Li, Peng, and Zhang]{li2023self}
Miaoran Li, Baolin Peng, and Zhu Zhang.
\newblock Self-checker: Plug-and-play modules for fact-checking with large language models.
\newblock \emph{arXiv preprint arXiv:2305.14623}, 2023{\natexlab{a}}.

\bibitem[Li et~al.(2023{\natexlab{b}})Li, Zhou, Zhu, Yao, Liu, and Han]{li2023deepinception}
Xuan Li, Zhanke Zhou, Jianing Zhu, Jiangchao Yao, Tongliang Liu, and Bo~Han.
\newblock Deepinception: Hypnotize large language model to be jailbreaker.
\newblock \emph{arXiv preprint arXiv:2311.03191}, 2023{\natexlab{b}}.

\bibitem[Li et~al.(2023{\natexlab{c}})Li, Zhang, and Sun]{li2023metaagents}
Yuan Li, Yixuan Zhang, and Lichao Sun.
\newblock Metaagents: Simulating interactions of human behaviors for llm-based task-oriented coordination via collaborative generative agents.
\newblock \emph{arXiv preprint arXiv:2310.06500}, 2023{\natexlab{c}}.

\bibitem[Liang et~al.(2023)Liang, Liu, Zhou, Tu, Wen, Yang, Dong, and Liu]{liang2023knowledge}
Ke~Liang, Yue Liu, Sihang Zhou, Wenxuan Tu, Yi~Wen, Xihong Yang, Xiangjun Dong, and Xinwang Liu.
\newblock Knowledge graph contrastive learning based on relation-symmetrical structure.
\newblock \emph{IEEE Transactions on Knowledge and Data Engineering}, 2023.

\bibitem[Liang et~al.(2024)Liang, Meng, Liu, Liu, Tu, Wang, Zhou, Liu, Sun, and He]{liang2024survey}
Ke~Liang, Lingyuan Meng, Meng Liu, Yue Liu, Wenxuan Tu, Siwei Wang, Sihang Zhou, Xinwang Liu, Fuchun Sun, and Kunlun He.
\newblock A survey of knowledge graph reasoning on graph types: Static, dynamic, and multi-modal.
\newblock \emph{IEEE Transactions on Pattern Analysis and Machine Intelligence}, 2024.

\bibitem[Liu et~al.(2023)Liu, Wu, Michael, Suhr, West, Koller, Swayamdipta, Smith, and Choi]{liu2023we}
Alisa Liu, Zhaofeng Wu, Julian Michael, Alane Suhr, Peter West, Alexander Koller, Swabha Swayamdipta, Noah~A Smith, and Yejin Choi.
\newblock We're afraid language models aren't modeling ambiguity.
\newblock \emph{arXiv preprint arXiv:2304.14399}, 2023.

\bibitem[Liu et~al.(2021{\natexlab{a}})Liu, Han, Liu, Gong, Niu, Zhou, Sugiyama, et~al.]{liu2021probabilistic}
Feng Liu, Bo~Han, Tongliang Liu, Chen Gong, Gang Niu, Mingyuan Zhou, Masashi Sugiyama, et~al.
\newblock Probabilistic margins for instance reweighting in adversarial training.
\newblock In \emph{NeurIPS}, 2021{\natexlab{a}}.

\bibitem[Liu et~al.(2021{\natexlab{b}})Liu, Shen, Zhang, Dolan, Carin, and Chen]{liu2021makes}
Jiachang Liu, Dinghan Shen, Yizhe Zhang, Bill Dolan, Lawrence Carin, and Weizhu Chen.
\newblock What makes good in-context examples for gpt-$3 $?
\newblock \emph{arXiv preprint arXiv:2101.06804}, 2021{\natexlab{b}}.

\bibitem[Lu et~al.(2021)Lu, Bartolo, Moore, Riedel, and Stenetorp]{lu2021fantastically}
Yao Lu, Max Bartolo, Alastair Moore, Sebastian Riedel, and Pontus Stenetorp.
\newblock Fantastically ordered prompts and where to find them: Overcoming few-shot prompt order sensitivity.
\newblock \emph{arXiv preprint arXiv:2104.08786}, 2021.

\bibitem[Madaan et~al.(2023)Madaan, Tandon, Gupta, Hallinan, Gao, Wiegreffe, Alon, Dziri, Prabhumoye, Yang, et~al.]{madaan2023self}
Aman Madaan, Niket Tandon, Prakhar Gupta, Skyler Hallinan, Luyu Gao, Sarah Wiegreffe, Uri Alon, Nouha Dziri, Shrimai Prabhumoye, Yiming Yang, et~al.
\newblock Self-refine: Iterative refinement with self-feedback.
\newblock \emph{arXiv preprint arXiv:2303.17651}, 2023.

\bibitem[Mayer(1977)]{mayer1977thinking}
Richard~E Mayer.
\newblock \emph{Thinking and problem solving: An introduction to human cognition and learning.}
\newblock Scott, Foresman, 1977.

\bibitem[Min et~al.(2022)Min, Lyu, Holtzman, Artetxe, Lewis, Hajishirzi, and Zettlemoyer]{min2022rethinking}
Sewon Min, Xinxi Lyu, Ari Holtzman, Mikel Artetxe, Mike Lewis, Hannaneh Hajishirzi, and Luke Zettlemoyer.
\newblock Rethinking the role of demonstrations: What makes in-context learning work?
\newblock In \emph{EMNLP}, 2022.

\bibitem[Morewedge and Kahneman(2010)]{morewedge2010associative}
Carey~K Morewedge and Daniel Kahneman.
\newblock Associative processes in intuitive judgment.
\newblock \emph{Trends in cognitive sciences}, 2010.

\bibitem[Norris(2000)]{norris2000emotional}
Paul Norris.
\newblock \emph{Emotional reasoning}.
\newblock University of Massachusetts Amherst, 2000.

\bibitem[Pan et~al.(2023)Pan, Saxon, Xu, Nathani, Wang, and Wang]{pan2023automatically}
Liangming Pan, Michael Saxon, Wenda Xu, Deepak Nathani, Xinyi Wang, and William~Yang Wang.
\newblock Automatically correcting large language models: Surveying the landscape of diverse self-correction strategies.
\newblock \emph{arXiv preprint arXiv:2308.03188}, 2023.

\bibitem[Pandia and Ettinger(2021)]{pandia2021sorting}
Lalchand Pandia and Allyson Ettinger.
\newblock Sorting through the noise: Testing robustness of information processing in pre-trained language models.
\newblock In \emph{EMNLP}, 2021.

\bibitem[Paul et~al.(2023)Paul, Ismayilzada, Peyrard, Borges, Bosselut, West, and Faltings]{paul2023refiner}
Debjit Paul, Mete Ismayilzada, Maxime Peyrard, Beatriz Borges, Antoine Bosselut, Robert West, and Boi Faltings.
\newblock Refiner: Reasoning feedback on intermediate representations.
\newblock \emph{arXiv preprint arXiv:2304.01904}, 2023.

\bibitem[Perez et~al.(2021)Perez, Kiela, and Cho]{perez2021true}
Ethan Perez, Douwe Kiela, and Kyunghyun Cho.
\newblock True few-shot learning with language models.
\newblock In \emph{NeurIPS}, 2021.

\bibitem[Radford et~al.(2019)Radford, Wu, Child, Luan, Amodei, Sutskever, et~al.]{radford2019language}
Alec Radford, Jeffrey Wu, Rewon Child, David Luan, Dario Amodei, Ilya Sutskever, et~al.
\newblock Language models are unsupervised multitask learners.
\newblock \emph{OpenAI blog}, 2019.

\bibitem[Robey et~al.(2023)Robey, Wong, Hassani, and Pappas]{robey2023smoothllm}
Alexander Robey, Eric Wong, Hamed Hassani, and George~J Pappas.
\newblock Smoothllm: Defending large language models against jailbreaking attacks.
\newblock \emph{arXiv preprint arXiv:2310.03684}, 2023.

\bibitem[Rokicki et~al.(2014)Rokicki, Kociemba, Davidson, and Dethridge]{rokicki2014diameter}
Tomas Rokicki, Herbert Kociemba, Morley Davidson, and John Dethridge.
\newblock The diameter of the rubik's cube group is twenty.
\newblock \emph{SIAM REVIEW}, 2014.

\bibitem[Sambasivan et~al.(2021)Sambasivan, Kapania, Highfill, Akrong, Paritosh, and Aroyo]{sambasivan2021everyone}
Nithya Sambasivan, Shivani Kapania, Hannah Highfill, Diana Akrong, Praveen Paritosh, and Lora~M Aroyo.
\newblock “everyone wants to do the model work, not the data work”: Data cascades in high-stakes ai.
\newblock In \emph{CHI}, 2021.

\bibitem[Saparov and He(2022)]{saparov2022language}
Abulhair Saparov and He~He.
\newblock Language models are greedy reasoners: A systematic formal analysis of chain-of-thought.
\newblock \emph{arXiv preprint arXiv:2210.01240}, 2022.

\bibitem[Schaeffer et~al.(2023)Schaeffer, Miranda, and Koyejo]{schaeffer2023emergent}
Rylan Schaeffer, Brando Miranda, and Sanmi Koyejo.
\newblock Are emergent abilities of large language models a mirage?
\newblock In \emph{NeurIPS}, 2023.

\bibitem[Scheurer et~al.(2023)Scheurer, Campos, Korbak, Chan, Chen, Cho, and Perez]{scheurer2023training}
J{\'e}r{\'e}my Scheurer, Jon~Ander Campos, Tomasz Korbak, Jun~Shern Chan, Angelica Chen, Kyunghyun Cho, and Ethan Perez.
\newblock Training language models with language feedback at scale.
\newblock \emph{arXiv preprint arXiv:2303.16755}, 2023.

\bibitem[Shi et~al.(2023)Shi, Chen, Misra, Scales, Dohan, Chi, Sch{\"a}rli, and Zhou]{shi2023large}
Freda Shi, Xinyun Chen, Kanishka Misra, Nathan Scales, David Dohan, Ed~H Chi, Nathanael Sch{\"a}rli, and Denny Zhou.
\newblock Large language models can be easily distracted by irrelevant context.
\newblock In \emph{ICML}, 2023.

\bibitem[Shinn et~al.(2023)Shinn, Cassano, Gopinath, Narasimhan, and Yao]{shinn2023reflexion}
Noah Shinn, Federico Cassano, Ashwin Gopinath, Karthik~R Narasimhan, and Shunyu Yao.
\newblock Reflexion: Language agents with verbal reinforcement learning.
\newblock In \emph{NeurIPS}, 2023.

\bibitem[Silver et~al.(2016)Silver, Huang, Maddison, Guez, Sifre, Van Den~Driessche, Schrittwieser, Antonoglou, Panneershelvam, Lanctot, et~al.]{silver2016mastering}
David Silver, Aja Huang, Chris~J Maddison, Arthur Guez, Laurent Sifre, George Van Den~Driessche, Julian Schrittwieser, Ioannis Antonoglou, Veda Panneershelvam, Marc Lanctot, et~al.
\newblock Mastering the game of go with deep neural networks and tree search.
\newblock \emph{Nature}, 2016.

\bibitem[Sinha et~al.(2019)Sinha, Sodhani, Dong, Pineau, and Hamilton]{sinha2019clutrr}
Koustuv Sinha, Shagun Sodhani, Jin Dong, Joelle Pineau, and William~L Hamilton.
\newblock Clutrr: A diagnostic benchmark for inductive reasoning from text.
\newblock \emph{arXiv preprint arXiv:1908.06177}, 2019.

\bibitem[Srivastava et~al.(2022)Srivastava, Rastogi, Rao, Shoeb, Abid, Fisch, Brown, Santoro, Gupta, Garriga-Alonso, et~al.]{srivastava2022beyond}
Aarohi Srivastava, Abhinav Rastogi, Abhishek Rao, Abu Awal~Md Shoeb, Abubakar Abid, Adam Fisch, Adam~R Brown, Adam Santoro, Aditya Gupta, Adri{\`a} Garriga-Alonso, et~al.
\newblock Beyond the imitation game: Quantifying and extrapolating the capabilities of language models.
\newblock \emph{arXiv preprint arXiv:2206.04615}, 2022.

\bibitem[Stacey et~al.(2024)Stacey, Cheng, Torr, Guigue, Driesen, Coca, Gaynor, and Johannsen]{stacey2024lucid}
Joe Stacey, Jianpeng Cheng, John Torr, Tristan Guigue, Joris Driesen, Alexandru Coca, Mark Gaynor, and Anders Johannsen.
\newblock Lucid: Llm-generated utterances for complex and interesting dialogues.
\newblock \emph{arXiv preprint arXiv:2403.00462}, 2024.

\bibitem[Tang et~al.(2023)Tang, Wang, He, Zhang, Pan, Wang, Zeng, Zhao, Shi, He, et~al.]{tang2023fusionai}
Zhenheng Tang, Yuxin Wang, Xin He, Longteng Zhang, Xinglin Pan, Qiang Wang, Rongfei Zeng, Kaiyong Zhao, Shaohuai Shi, Bingsheng He, et~al.
\newblock Fusionai: Decentralized training and deploying llms with massive consumer-level gpus.
\newblock \emph{arXiv preprint arXiv:2309.01172}, 2023.

\bibitem[Tang et~al.(2024)Tang, Kang, Yin, Pan, Wang, He, Wang, Zeng, Zhao, Shi, Zhou, Li, He, and Chu]{tang2024fusionllm}
Zhenheng Tang, Xueze Kang, Yiming Yin, Xinglin Pan, Yuxin Wang, Xin He, Qiang Wang, Rongfei Zeng, Kaiyong Zhao, Shaohuai Shi, Amelie~Chi Zhou, Bo~Li, Bingsheng He, and Xiaowen Chu.
\newblock Fusionllm: A decentralized llm training system on geo-distributed gpus with adaptive compression.
\newblock \emph{arXiv preprint arXiv:2410.12707}, 2024.

\bibitem[Team et~al.(2023)Team, Anil, Borgeaud, Wu, Alayrac, Yu, Soricut, Schalkwyk, Dai, Hauth, et~al.]{team2023gemini}
Gemini Team, Rohan Anil, Sebastian Borgeaud, Yonghui Wu, Jean-Baptiste Alayrac, Jiahui Yu, Radu Soricut, Johan Schalkwyk, Andrew~M Dai, Anja Hauth, et~al.
\newblock Gemini: a family of highly capable multimodal models.
\newblock \emph{arXiv preprint arXiv:2312.11805}, 2023.

\bibitem[Thorstad(2023)]{thorstad2023cognitive}
David Thorstad.
\newblock Cognitive bias in large language models: Cautious optimism meets anti-panglossian meliorism.
\newblock \emph{arXiv preprint arXiv:2311.10932}, 2023.

\bibitem[Tian et~al.(2023)Tian, Zhu, Wang, Li, and Lan]{tian2023r3}
Qingyuan Tian, Hanlun Zhu, Lei Wang, Yang Li, and Yunshi Lan.
\newblock R3 prompting: Review, rephrase and resolve for chain-of-thought reasoning in large language models under noisy context prompting: Review, rephrase and resolve for chain-of-thought reasoning in large language models under noisy context.
\newblock In \emph{EMNLP}, 2023.

\bibitem[Touvron et~al.(2023)Touvron, Martin, Stone, Albert, Almahairi, Babaei, Bashlykov, Batra, Bhargava, Bhosale, et~al.]{touvron2023llama}
Hugo Touvron, Louis Martin, Kevin Stone, Peter Albert, Amjad Almahairi, Yasmine Babaei, Nikolay Bashlykov, Soumya Batra, Prajjwal Bhargava, Shruti Bhosale, et~al.
\newblock Llama 2: Open foundation and fine-tuned chat models.
\newblock \emph{arXiv preprint arXiv:2307.09288}, 2023.

\bibitem[Tversky and Kahneman(1974)]{tversky1974judgment}
Amos Tversky and Daniel Kahneman.
\newblock Judgment under uncertainty: Heuristics and biases: Biases in judgments reveal some heuristics of thinking under uncertainty.
\newblock \emph{Science}, 1974.

\bibitem[Tyen et~al.(2023)Tyen, Mansoor, Chen, Mak, and C{\u{a}}rbune]{tyen2023llms}
Gladys Tyen, Hassan Mansoor, Peter Chen, Tony Mak, and Victor C{\u{a}}rbune.
\newblock Llms cannot find reasoning errors, but can correct them!
\newblock \emph{arXiv preprint arXiv:2311.08516}, 2023.

\bibitem[Valmeekam et~al.(2023)Valmeekam, Marquez, Sreedharan, and Kambhampati]{valmeekam2023planning}
Karthik Valmeekam, Matthew Marquez, Sarath Sreedharan, and Subbarao Kambhampati.
\newblock On the planning abilities of large language models-a critical investigation.
\newblock In \emph{NeurIPS}, 2023.

\bibitem[Wang et~al.(2023)Wang, Wei, Schuurmans, Le, Chi, Narang, Chowdhery, and Zhou]{wang2023self}
Xuezhi Wang, Jason Wei, Dale Schuurmans, Quoc Le, Ed~Chi, Sharan Narang, Aakanksha Chowdhery, and Denny Zhou.
\newblock Self-consistency improves chain of thought reasoning in language models.
\newblock In \emph{ICLR}, 2023.

\bibitem[Wei et~al.(2022{\natexlab{a}})Wei, Tay, Bommasani, Raffel, Zoph, Borgeaud, Yogatama, Bosma, Zhou, Metzler, et~al.]{wei2022emergent}
Jason Wei, Yi~Tay, Rishi Bommasani, Colin Raffel, Barret Zoph, Sebastian Borgeaud, Dani Yogatama, Maarten Bosma, Denny Zhou, Donald Metzler, et~al.
\newblock Emergent abilities of large language models.
\newblock \emph{arXiv preprint arXiv:2206.07682}, 2022{\natexlab{a}}.

\bibitem[Wei et~al.(2022{\natexlab{b}})Wei, Wang, Schuurmans, Bosma, Xia, Chi, Le, Zhou, et~al.]{wei2022chain}
Jason Wei, Xuezhi Wang, Dale Schuurmans, Maarten Bosma, Fei Xia, Ed~Chi, Quoc~V Le, Denny Zhou, et~al.
\newblock Chain-of-thought prompting elicits reasoning in large language models.
\newblock In \emph{NeurIPS}, 2022{\natexlab{b}}.

\bibitem[Wei et~al.(2023)Wei, Wei, Tay, Tran, Webson, Lu, Chen, Liu, Huang, Zhou, et~al.]{wei2023larger}
Jerry Wei, Jason Wei, Yi~Tay, Dustin Tran, Albert Webson, Yifeng Lu, Xinyun Chen, Hanxiao Liu, Da~Huang, Denny Zhou, et~al.
\newblock Larger language models do in-context learning differently.
\newblock \emph{arXiv preprint arXiv:2303.03846}, 2023.

\bibitem[Wimmer and Shohamy(2012)]{wimmer2012preference}
G~Elliott Wimmer and Daphna Shohamy.
\newblock Preference by association: how memory mechanisms in the hippocampus bias decisions.
\newblock \emph{Science}, 2012.

\bibitem[Wu et~al.(2023)Wu, Qiu, Ross, Aky{\"u}rek, Chen, Wang, Kim, Andreas, and Kim]{wu2023reasoning}
Zhaofeng Wu, Linlu Qiu, Alexis Ross, Ekin Aky{\"u}rek, Boyuan Chen, Bailin Wang, Najoung Kim, Jacob Andreas, and Yoon Kim.
\newblock Reasoning or reciting? exploring the capabilities and limitations of language models through counterfactual tasks.
\newblock \emph{arXiv preprint arXiv:2307.02477}, 2023.

\bibitem[Xi et~al.(2023)Xi, Jin, Zhou, Zheng, Gao, Gui, Zhang, and Huang]{xi2023self}
Zhiheng Xi, Senjie Jin, Yuhao Zhou, Rui Zheng, Songyang Gao, Tao Gui, Qi~Zhang, and Xuanjing Huang.
\newblock Self-polish: Enhance reasoning in large language models via problem refinement.
\newblock \emph{arXiv preprint arXiv:2305.14497}, 2023.

\bibitem[Xie et~al.(2021)Xie, Raghunathan, Liang, and Ma]{xie2021explanation}
Sang~Michael Xie, Aditi Raghunathan, Percy Liang, and Tengyu Ma.
\newblock An explanation of in-context learning as implicit bayesian inference.
\newblock \emph{arXiv preprint arXiv:2111.02080}, 2021.

\bibitem[Yang et~al.(2023)Yang, Wang, Lu, Liu, Le, Zhou, and Chen]{yang2023large}
Chengrun Yang, Xuezhi Wang, Yifeng Lu, Hanxiao Liu, Quoc~V Le, Denny Zhou, and Xinyun Chen.
\newblock Large language models as optimizers.
\newblock \emph{arXiv preprint arXiv:2309.03409}, 2023.

\bibitem[Yao et~al.(2023)Yao, Yu, Zhao, Shafran, Griffiths, Cao, and Narasimhan]{yao2024tree}
Shunyu Yao, Dian Yu, Jeffrey Zhao, Izhak Shafran, Tom Griffiths, Yuan Cao, and Karthik Narasimhan.
\newblock Tree of thoughts: Deliberate problem solving with large language models.
\newblock In \emph{NeurIPS}, 2023.

\bibitem[Ye et~al.(2023)Ye, Jo, Kim, Kim, Hwang, and Seo]{ye2023selfee}
Seonghyeon Ye, Yongrae Jo, Doyoung Kim, Sungdong Kim, Hyeonbin Hwang, and Minjoon Seo.
\newblock Selfee: Iterative self-revising llm empowered by self-feedback generation.
\newblock \emph{Blog post, May}, 2023.

\bibitem[Ye and Durrett(2022)]{ye2022unreliability}
Xi~Ye and Greg Durrett.
\newblock The unreliability of explanations in few-shot prompting for textual reasoning.
\newblock In \emph{NeurIPS}, 2022.

\bibitem[Yu et~al.(2024)Yu, He, Minervini, and Pan]{yu2024evaluating}
Simon Chi~Lok Yu, Jie He, Pasquale Minervini, and Jeff~Z Pan.
\newblock Evaluating the adversarial robustness of retrieval-based in-context learning for large language models.
\newblock \emph{arXiv preprint arXiv:2405.15984}, 2024.

\bibitem[Yu et~al.(2023)Yu, Zhang, Liang, Jiang, and Sabharwal]{yu2023improving}
Wenhao Yu, Zhihan Zhang, Zhenwen Liang, Meng Jiang, and Ashish Sabharwal.
\newblock Improving language models via plug-and-play retrieval feedback.
\newblock \emph{arXiv preprint arXiv:2305.14002}, 2023.

\bibitem[Zelikman et~al.(2022)Zelikman, Wu, Mu, and Goodman]{zelikman2022star}
Eric Zelikman, Yuhuai Wu, Jesse Mu, and Noah Goodman.
\newblock Star: Bootstrapping reasoning with reasoning.
\newblock In \emph{NeurIPS}, 2022.

\bibitem[Zeng et~al.(2023)Zeng, Xu, Zheng, and Huang]{zeng2023certified}
Jiehang Zeng, Jianhan Xu, Xiaoqing Zheng, and Xuanjing Huang.
\newblock Certified robustness to text adversarial attacks by randomized [mask].
\newblock \emph{Computational Linguistics}, 2023.

\bibitem[Zhang et~al.(2022)Zhang, Feng, and Tan]{zhang2022active}
Yiming Zhang, Shi Feng, and Chenhao Tan.
\newblock Active example selection for in-context learning.
\newblock \emph{arXiv preprint arXiv:2211.04486}, 2022.

\bibitem[Zhang et~al.(2023{\natexlab{a}})Zhang, Zhou, Yao, Chu, and Han]{zhang2023adaprop}
Yongqi Zhang, Zhanke Zhou, Quanming Yao, Xiaowen Chu, and Bo~Han.
\newblock Adaprop: Learning adaptive propagation for graph neural network based knowledge graph reasoning.
\newblock In \emph{SIGKDD}, 2023{\natexlab{a}}.

\bibitem[Zhang et~al.(2023{\natexlab{b}})Zhang, Li, Cui, Cai, Liu, Fu, Huang, Zhao, Zhang, Chen, Wang, Luu, Bi, Shi, and Shi]{zhang2023hallucination}
Yue Zhang, Yafu Li, Leyang Cui, Deng Cai, Lemao Liu, Tingchen Fu, Xinting Huang, Enbo Zhao, Yu~Zhang, Yulong Chen, Longyue Wang, Anh~Tuan Luu, Wei Bi, Freda Shi, and Shuming Shi.
\newblock Siren's song in the ai ocean: A survey on hallucination in large language models.
\newblock \emph{arXiv preprint arXiv:2309.01219}, 2023{\natexlab{b}}.

\bibitem[Zhang et~al.(2023{\natexlab{c}})Zhang, Zhang, Hou, Fan, Li, Liu, Zhang, and Chang]{zhang2023certified}
Zhen Zhang, Guanhua Zhang, Bairu Hou, Wenqi Fan, Qing Li, Sijia Liu, Yang Zhang, and Shiyu Chang.
\newblock Certified robustness for large language models with self-denoising.
\newblock \emph{arXiv preprint arXiv:2307.07171}, 2023{\natexlab{c}}.

\bibitem[Zhang et~al.(2023{\natexlab{d}})Zhang, Zhang, Li, and Smola]{zhang2023automatic}
Zhuosheng Zhang, Aston Zhang, Mu~Li, and Alex Smola.
\newblock Automatic chain of thought prompting in large language models.
\newblock In \emph{ICLR}, 2023{\natexlab{d}}.

\bibitem[Zhao et~al.(2023{\natexlab{a}})Zhao, Chen, Yang, Liu, Deng, Cai, Wang, Yin, and Du]{zhao2023explainability}
Haiyan Zhao, Hanjie Chen, Fan Yang, Ninghao Liu, Huiqi Deng, Hengyi Cai, Shuaiqiang Wang, Dawei Yin, and Mengnan Du.
\newblock Explainability for large language models: A survey.
\newblock \emph{ACM Transactions on Intelligent Systems and Technology}, 2023{\natexlab{a}}.

\bibitem[Zhao et~al.(2023{\natexlab{b}})Zhao, Zhou, Li, Tang, Wang, Hou, Min, Zhang, Zhang, Dong, et~al.]{zhao2023survey}
Wayne~Xin Zhao, Kun Zhou, Junyi Li, Tianyi Tang, Xiaolei Wang, Yupeng Hou, Yingqian Min, Beichen Zhang, Junjie Zhang, Zican Dong, et~al.
\newblock A survey of large language models.
\newblock \emph{arXiv preprint arXiv:2303.18223}, 2023{\natexlab{b}}.

\bibitem[Zhao et~al.(2021)Zhao, Wallace, Feng, Klein, and Singh]{zhao2021calibrate}
Zihao Zhao, Eric Wallace, Shi Feng, Dan Klein, and Sameer Singh.
\newblock Calibrate before use: Improving few-shot performance of language models.
\newblock In \emph{ICML}, 2021.

\bibitem[Zheng and Saparov(2023)]{zheng2023noisy}
Hongyi Zheng and Abulhair Saparov.
\newblock Noisy exemplars make large language models more robust: A domain-agnostic behavioral analysis.
\newblock In \emph{EMNLP}, 2023.

\bibitem[Zheng et~al.(2023)Zheng, Chiang, Sheng, Zhuang, Wu, Zhuang, Lin, Li, Li, Xing, et~al.]{zheng2024judging}
Lianmin Zheng, Wei-Lin Chiang, Ying Sheng, Siyuan Zhuang, Zhanghao Wu, Yonghao Zhuang, Zi~Lin, Zhuohan Li, Dacheng Li, Eric Xing, et~al.
\newblock Judging llm-as-a-judge with mt-bench and chatbot arena.
\newblock In \emph{NeurIPS}, 2023.

\bibitem[Zhou et~al.(2022{\natexlab{a}})Zhou, Sch{\"a}rli, Hou, Wei, Scales, Wang, Schuurmans, Cui, Bousquet, Le, et~al.]{zhou2022least}
Denny Zhou, Nathanael Sch{\"a}rli, Le~Hou, Jason Wei, Nathan Scales, Xuezhi Wang, Dale Schuurmans, Claire Cui, Olivier Bousquet, Quoc Le, et~al.
\newblock Least-to-most prompting enables complex reasoning in large language models.
\newblock \emph{arXiv preprint arXiv:2205.10625}, 2022{\natexlab{a}}.

\bibitem[Zhou et~al.(2022{\natexlab{b}})Zhou, Muresanu, Han, Paster, Pitis, Chan, and Ba]{zhou2022large}
Yongchao Zhou, Andrei~Ioan Muresanu, Ziwen Han, Keiran Paster, Silviu Pitis, Harris Chan, and Jimmy Ba.
\newblock Large language models are human-level prompt engineers.
\newblock \emph{arXiv preprint arXiv:2211.01910}, 2022{\natexlab{b}}.

\bibitem[Zhou et~al.(2023{\natexlab{a}})Zhou, Yao, Liu, Guo, Yao, He, Wang, Zheng, and Han]{zhou2023combating}
Zhanke Zhou, Jiangchao Yao, Jiaxu Liu, Xiawei Guo, Quanming Yao, Li~He, Liang Wang, Bo~Zheng, and Bo~Han.
\newblock Combating bilateral edge noise for robust link prediction.
\newblock In \emph{NeurIPS}, 2023{\natexlab{a}}.

\bibitem[Zhou et~al.(2023{\natexlab{b}})Zhou, Zhou, Li, Yao, Yao, and Han]{zhou2023mcgra}
Zhanke Zhou, Chenyu Zhou, Xuan Li, Jiangchao Yao, Quanming Yao, and Bo~Han.
\newblock On strengthening and defending graph reconstruction attack with markov chain approximation.
\newblock In \emph{ICML}, 2023{\natexlab{b}}.

\bibitem[Zhou et~al.(2024)Zhou, Zhang, Yao, Yao, and Han]{zhou2024less}
Zhanke Zhou, Yongqi Zhang, Jiangchao Yao, Quanming Yao, and Bo~Han.
\newblock Less is more: One-shot subgraph reasoning on large-scale knowledge graphs.
\newblock In \emph{ICLR}, 2024.

\bibitem[Zhu et~al.(2024)Zhu, Moniz, Bhargava, Lu, Piraviperumal, Li, Zhang, Yu, and Tseng]{zhu2024large}
Yilun Zhu, Joel Ruben~Antony Moniz, Shruti Bhargava, Jiarui Lu, Dhivya Piraviperumal, Site Li, Yuan Zhang, Hong Yu, and Bo-Hsiang Tseng.
\newblock Can large language models understand context?
\newblock \emph{arXiv preprint arXiv:2402.00858}, 2024.

\end{thebibliography}
\bibliographystyle{plainnat}

\clearpage


\clearpage
\onecolumn
\appendix 
\etocdepthtag.toc{mtappendix}
\etocsettagdepth{mtchapter}{none}
\etocsettagdepth{mtappendix}{subsection}
\renewcommand{\contentsname}{Appendix}
\tableofcontents 
\clearpage

\section{Further Discussion}
\label{appendix: further discussion}

\textbf{Broader impact.}
This paper presents work that aims to advance the field of trustworthy machine learning and large language models. We do not find any negative societal consequences of our work.
This paper does not raise any ethical concerns. 
This study does not involve human subjects, practices, data set releases, potentially harmful insights, methodologies, applications, potential conflicts of interest and sponsorship, discrimination/bias/fairness concerns, privacy and security issues, legal compliance, or research integrity issues.

\textbf{Limitations.}
The proposed CD-CoT method is currently dependent on human-annotated, clean rationales. Future enhancements could include developing a self-supervised variant that does not rely on such examples. Exploring strategies like using contrasting noisy examples or incorporating an external knowledge base, possibly through a retrieval-augmented denoising framework, may offer significant advances in automation and robustness of reasoning.

\textbf{Extensions.}
CoT and its variants have predominantly focused on deductive reasoning, leaving inductive reasoning largely unexplored. 
Investigating the ability of LLMs to extract rules from noisy examples is a compelling area. 
Additionally, theoretical analysis of noisy ICL can offer deeper insights into the Noisy-R problem. 
Expanding the NoRa dataset to include multi-modal scenarios, particularly visual data, is also crucial for a more comprehensive understanding of the robustness of foundation models.
Further extensions include knowledge-enhanced denoising~\citep{zhang2023adaprop, liang2023knowledge, liang2024survey, zhou2024less}, generalization to out-of-distribution noisy scenarios~\citep{chen2022ciga, chen2023understanding}, and training to fundamentally improve the robustness of language models~\citep{tang2024fusionllm, tang2023fusionai}.

\section{Related Work}
\label{appendix: related work}

In this section, we provide a detailed literature review as an extension of the preliminaries (Sec.~\ref{sec: related-work}), including
in-context learning (Appendix~\ref{appendix: related work-ICL}),
self-correction methods (Appendix~\ref{appendix: related work-self-correction}),
self-consistency methods (Appendix~\ref{appendix: related work-self-consistency}),
and external supervision (Appendix~\ref{appendix: related work-external}).
We further discuss the relation between our work and literature in Appendix~\ref{appendix: related work-relation}.
We also provide Fig.~\ref{fig:problem-full} to better understand different reasoning settings.

\begin{figure}[h]
    \centering
    \includegraphics[width=\textwidth]{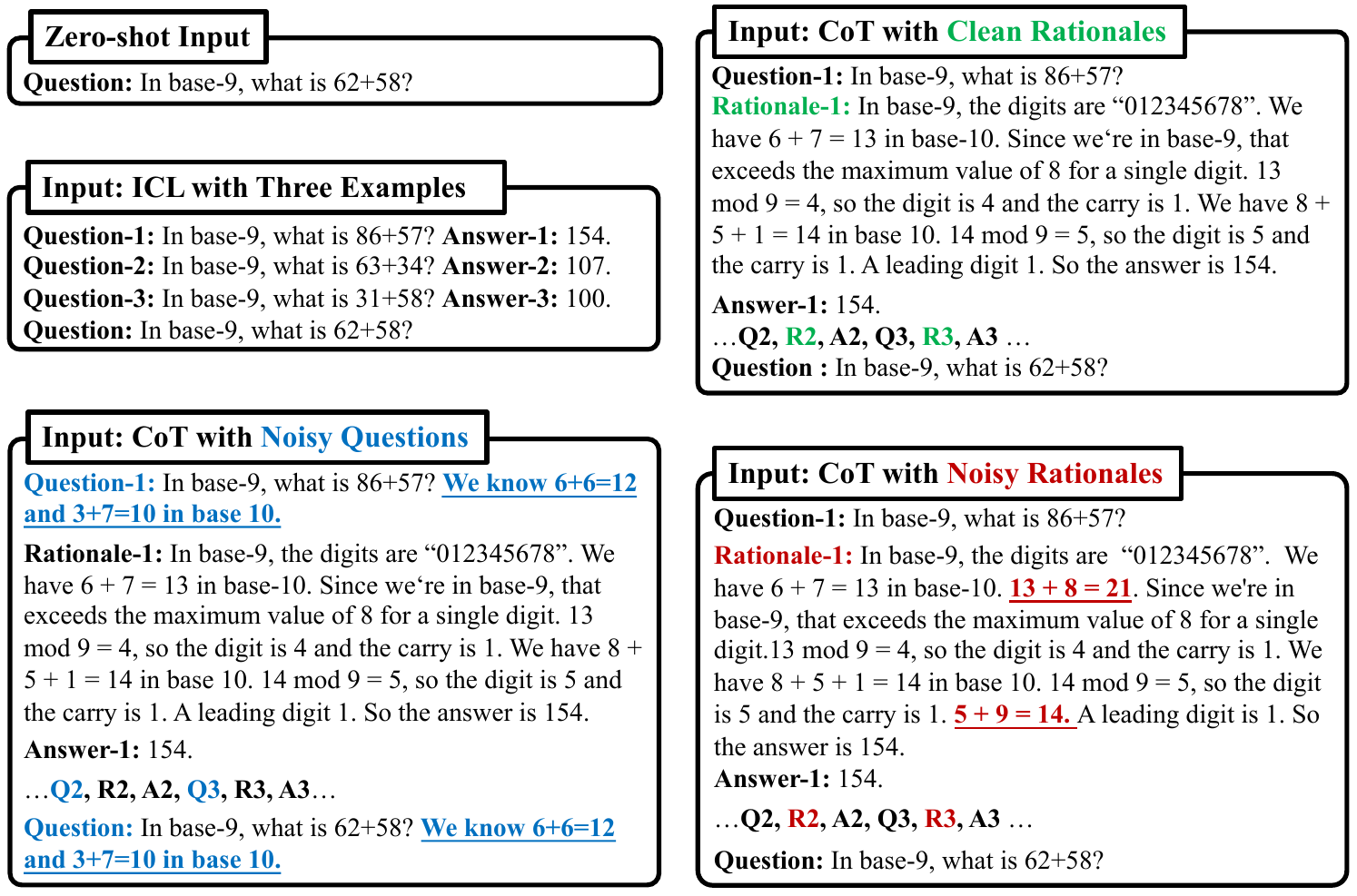}
    \vspace{-6px}
    \caption{
    Illustrations of different reasoning settings.
    }
\label{fig:problem-full}
\vspace{-5pt}
\end{figure}

\subsection{In-context Learning}
\label{appendix: related work-ICL}

With the scaling of model size and corpus size~\cite{radford2019language, brown2020language, chowdhery2023palm, zhao2023survey}, large language models (LLMs) have demonstrated remarkable performance across a variety of tasks through in-context learning (ICL), that is, learning from a few input-output examples in the context. As a training-free framework, ICL can save on expensive training costs and be easily adapted to diverse new tasks~\citep{wu2023reasoning,achiam2023gpt}.

However, standard ICL faces challenges when addressing \textit{intricate} reasoning problems. \citet{wei2022chain} proposes chain-of-thought prompting (CoT), a methodology that explicitly guides LLMs in generating sequential reasoning steps to enhance their performance on intricate reasoning tasks. In particular, CoT replaces the input-output exemplars in ICL with triplets in the form of \textit{<input, rationale, output>}, thereby enabling the model to learn explicit reasoning processes.

In the literature, various versions of chain-of-thought prompting have been proposed. To alleviate the human effort required for constructing triplet exemplars, \citet{kojima2022large} proposes zero-shot CoT, a method that explicitly encourages models to generate a chain of reasoning first and then derive the final answer by employing natural language prompts like ``Let's think step by step''. \citet{wang2023self} shows that self-consistency sampling can improve reasoning accuracy by sampling diverse reasoning paths and then taking the majority vote. \citet{zelikman2022star} proposes STaR, which leverages a small amount of human-annotated rationale data, along with a large dataset of question-label pairs without rationales. This approach iteratively generates additional rationales and enhances the model's ability to generate reasoning steps by fine-tuning.

Least-to-most prompting (LtM)~\cite{zhou2022least} enhances the reasoning capability of models by decomposing a given problem into sub-problems. In particular, LtM initially breaks down the task into a series of consecutive sub-problems and subsequently answers them one by one. During the process of responding, the answer to the preceding sub-problem is incorporated into the prompt for the succeeding one. 
Tree of Thoughts (ToT)~\cite{yao2024tree} extends LtM by exploring multiple reasoning possibilities at each step. Specifically, ToT first decomposes a given problem into several reasoning steps and generates multiple answers for each step, ultimately constructing a tree structure. Subsequently, ToT employs BFS or DFS to traverse the tree, yielding the final rationale and answer.

\textbf{Vulnerabilities of ICL.} Despite being promising, some works point out the brittleness and over-sensitivity of ICL. \citet{liu2021makes, perez2021true, zhang2022active} demonstrate that ICL performance depends heavily on the choice of exemplars. Meanwhile, \citet{zhao2021calibrate, lu2021fantastically} observe that the arrangement order of in-context examples is also crucial to the ICL performance, potentially shifting results from near state-of-the-art to a random guessing. 
Moreover, \citet{ye2022unreliability, gan2023sensitivity, zheng2023noisy, zhang2023certified} reveal LLMs' deficiencies when handling subtle perturbations within the prompts, even when such perturbations do not alter any semantic meaning. \citet{yu2024evaluating} explores the robustness of retrieval-augmented in-context learning (ICL) against demonstration attacks and test sample attacks. It focuses on perturbing the example questions (i.e., noisy questions) or labels, while our work focuses on the rationales of the examples (i.e., noisy rationales). 
In addition, previous work on safety~\citep{li2023deepinception, zhou2023mcgra}
and data noise~\citep{han2018co, liu2021probabilistic, gong2021instance, zhou2023combating} might also inspire the robust problems in ICL.
Overall, the enhancement of reasoning performance brought about by ICL is inherently unstable and susceptible to example selection, example ordering, and prompt perturbations. These observations underscore the importance the robustness of other aspects.

The aforementioned efforts primarily revolve around the idealized ICL, which utilizes high-quality prompts free from any noise or interference. Conversely, a parallel line of research has emerged, exploring the impact of noisy prompts on the performance of LLMs. 
\citet{min2022rethinking} examines the impact of in-context examples on ICL. This work observes that incorporating out-of-distribution input texts significantly diminishes the performance of standard question answering. 
\citet{wei2023larger} devises two different set-ups of ICL: ICL with flipped labels and ICL with semantically unrelated labels. Their investigation reveals that LLMs possess the capability to override semantic priors when confronted with in-context exemplars that contradict these priors. This phenomenon also suggests that larger models may be more susceptible to the influence of the noise present in examples.
\citet{shi2023large} examines the impact of irrelevant context on LLMs, and the results suggest that the inclusion of irrelevant information can significantly impair the performance of the models.
These studies further illuminate the fragility and instability inherent in the reasoning capabilities of LLMs.

However, the previous works mainly consider the noisy questions/answers in standard ICL. In contrast, we move to the under-explored noisy rationales problem in the context of CoT, as illustrated in Fig.~\ref{fig:problem-full}.



Numerous strategies have been proposed to address the vulnerabilities of LLM reasoning during in-context learning. These approaches can be categorized into \textit{self-correction} and \textit{self-consistency}, which are introduced as follows.

\vspace{-10pt}
\subsection{Self-correction} 
\label{appendix: related work-self-correction}
\vspace{-4pt}

Self-correction emerges as a promising direction to enhance LLM reasoning, where LLMs attempt to correct their initial responses based on feedback. One popular line of research involves utilizing manual labor or external systems to evaluate and refine models. However, this can be costly due to the manual labor involved. Another line of research leverages the LLM's inherent capabilities to correct its initial responses without the crutch of external feedback. This methodology is a promising way to make LLM-based solutions practical and deployable~\citep{pan2023automatically}.

\textbf{Self-correction with internal feedback.} In this line of research, the LLM is required to correct response trajectories based solely on its inherent capabilities. 
\citet{huang2022selfimprove} first demonstrates the self-improvement potential of LLMs by utilizing a pre-trained LLM to generate rationale-augmented answers for unlabeled questions using CoT and majority voting and then fine-tuning the LLM using those self-generated labels, eventually improving the general reasoning ability.

When addressing problems, people typically engage in trial and error, coupled with reflective thinking, to discern the correct solutions.
Inspired by this, \citet{madaan2023self} proposes Self-refine, a simple and direct approach to improving LLM's output. In this approach, an LLM is used to create an initial output. Then, the model provides feedback on its own output in multiple dimensions. Based on this feedback, the model refines its initial output and repeats this process until it reaches a specified limit or the LLM determines that no further adjustments are necessary. 

Encouraged by the augmented efficacy achieved through self-feedback mechanisms, \citet{ye2023selfee} releases SelFee, a new instruction-following language model that generates self-feedback on its response and self-revises based on the feedback. The development of SelFee involves the fine-tuning of LLaMA by utilizing training instances generated by ChatGPT. 

In addition, \citet{gero2023self} introduces Self-verification, suggesting that by asking LLMs to provide provenance for their own outputs and conducting checks, it is possible to alleviate LLMs' issues regarding accuracy and interpretability in crucial domains such as healthcare.
On the other hand, \citet{xi2023self} focuses on the simplicity and comprehensibility of the given questions, proposing Self-polish (SP). This method instructs the LLM to iteratively refine the test question by removing irrelevant information and rearranging the logical structure, thereby improving the reasoning performance. 

While the self-correction methodologies based on internal feedback appear promising, \cite{huang2023large} categorizes such self-correction methods as \textit{intrinsic self-correction} (ISC) and demonstrates that the model’s performance drops on all benchmarks after using ISC. This work points out that LLMs struggle to self-correct their responses without external feedback, and the corrected responses often exhibit inferior quality compared to their initial counterparts. \citet{saparov2022language} reveals that while models are able to produce valid reasoning steps with high probability when dealing with proof problems, they struggle with proof planning. In other words, when models occasionally generate incorrect proof steps, they are not able to return to the correct path.

Building upon this observation, \citet{tyen2023llms} further decomposes the self-correction process into two core components: mistake finding
and output correction. This work demonstrates that current state-of-the-art LLMs cannot find mistakes reliably, even in the most simple and unambiguous cases, and suggests this is a main contributing factor to LLMs’ inability to self-correct reasoning errors.

\textbf{Self-correction with external feedback.}
External feedback offers a valuable external perspective, proving particularly advantageous in pinpointing errors that the large language model may not inherently recognize~\citep{pan2023automatically}. The sources of external feedback can be categorized as 1) human feedback~\citep{scheurer2023training, shinn2023reflexion, kim2023language, tyen2023llms}, 2) external tools~\cite{jiang2023selfevolve, yu2023improving, li2023self}, and 3) other models~\citep{paul2023refiner, akyurek2023rl4f}. 

\citet{scheurer2023training} proposes Imitation Learning from Language Feedback (ILF), an approach leveraging informative human feedback that involves conditioning the model on input, initial output, and feedback; selecting the most feedback-incorporated refinement; and fine-tuning the model to maximize the chosen refinement's likelihood given the input. 

Similarly, \citet{shinn2023reflexion} introduces Reflexion, which fortifies language agents by relying on linguistic feedback generated by themselves rather than weight updates, resulting in noteworthy enhancements compared to a baseline agent across a spectrum of tasks. \citet{kim2023language} demonstrates the effectiveness of self-critique on computer tasks by proposing Recursively Criticizes and Improves (RCI). This method utilizes a pre-trained LLM agent to execute computer tasks, guiding it using a simple prompting scheme where the agent recursively criticizes and improves its output. 

Although Reflexion~\cite{shinn2023reflexion} and RCI~\cite{kim2023language} do not explicitly utilize human feedback, \citet{huang2023large} points out that these approaches employ the labels of test questions to determine when to halt the self-correction loop. Consequently, they are implicitly guided by human feedback. ~\citet{huang2023large} categorizes these methods as \textit{self-correction with oracle feedback} (SCO).

Due to the difficulty of LLMs relying solely on their inherent capabilities to correct responses for reasoning tasks~\cite{huang2023large}, \citet{tyen2023llms} decomposes the self-correction process and observes challenges in accurately identifying error locations for LLMs. Consequently, they propose Backtracking~(BT), which guides the correction process of models by providing the initial error location in the problem.

Furthermore, the incorporation of external tools can enhance the model's self-correction performance. For instance, to improve the code generation capabilities of the model, \citet{jiang2023selfevolve} utilizes error messages generated by an external interpreter as feedback, requiring LLM to debug the initially generated code acting as a professional programmer. 
Introducing external sources of knowledge, such as Wikipedia, can also significantly enhance the accuracy of the model in reasoning tasks~\cite{yu2023improving, li2023self}. 

In addition to leveraging external tools, certain efforts revolve around the utilization of external models to generate feedback. \citet{paul2023refiner} introduces a critic model that generates structured feedback for the responses produced by the LLM during the self-correction process. The feedback is then incorporated into the input, allowing the LLM to refine its own output iteratively. Likewise, \citet{akyurek2023rl4f} presents a collaborative framework involving multiple agents, where the critique generator is trained with the aim of maximizing the performance of the reasoning LLM.

\vspace{-10pt}
\subsection{Self-consistency}
\label{appendix: related work-self-consistency}
\vspace{-4pt}

In our study, self-consistency represents a category of robustness methods, typically involving the generation of multiple diverse answers to a given question and subsequently converging them. This approach serves to mitigate the impact of noise inherent in the question. 
Among existing methods, two primary approaches are employed to generate distinct answers: (1) generating multiple samples via randomized smoothing on input, followed by separate and independent reasoning processess~\citep{robey2023smoothllm, zhang2023certified, zeng2023certified}, and (2) leveraging the inherent randomness in the decoding process of the LLM to sample multiple distinct reasoning paths for the same input~\citep{wang2023self}.

\textbf{Sampling diverse outputs.} 
Relying on the idea that an intricate reasoning problem usually allows for various distinct thinking ways, each leading to the unique correct solution,~\citet{wang2023self} first explores a varied array of reasoning paths and subsequently identifies the most consistent answer by majority vote. 
This approach boosts the model's performance across various reasoning tasks.

\textbf{Randomized smoothing over inputs.}\label{smooth} 
Randomized smoothing serves as a straightforward yet effective mechanism commonly employed to defend against adversarial attacks by introducing multiple perturbations to the input to smooth potential noise~\cite{cohen2019certified, zeng2023certified, robey2023smoothllm, zhang2023certified}. 

To address the vulnerability of LLMs to jailbreaking attacks, ~\citet{robey2023smoothllm} proposes Smooth LLM~(SM), which is based on randomized smoothing. Utilizing the discovery that prompts generated adversarially exhibit fragility to character-level modifications, this method initiates by randomly introducing perturbations to multiple copies of a given input prompt. Subsequently, it consolidates the associated predictions to identify adversarial inputs.

However, as pointed out in~\cite{zhang2023certified}, randomized smoothing relies on introducing additional noise to input, and its certification success hinges on the model's performance with corrupted data. To address the issue, ~\citet{zhang2023certified} proposes Self-denoising~(SD), which denoises the corrupted inputs with LLMs using the multitasking nature of LLMs.

\vspace{-10pt}
\subsection{External Supervision}
\label{appendix: related work-external}
\vspace{-4pt}

In addition to self-correction and self-consistency methods, there is a line of methods that require external supervision or information.
In our study, we compare diverse denoising methods that require varying degrees of external supervision, including (1) ground truth for test questions, (2) noise positions, and (3) additional clean demonstrations.

To provide a thorough comparison, We adopt Self-correction with Oracle feedback~(SCO)~\cite{huang2023large} as a strong baseline, which uses the labels of test questions to determine when to terminate the self-correction loop, as discussed in \ref{appendix: related work-self-correction}.
We also employ Backtracking~(BT)~\citep{tyen2023llms}, which necessitates the location of the initial noise in the prompt, to demonstrate the effectiveness of our method.
In addition, we utilize
Contrastive Chain-of-Thought~(CC)~\citep{chia2023contrastive} in experiments which needs similar external supervision as our proposed method: a clean demonstration. This approach enhances reasoning by presenting the LLM with all the clean and noisy examples concurrently. 

\vspace{-10pt}
\subsection{Relation with Literature}
\label{appendix: related work-relation}
\vspace{-4pt}

There are some implicit connections with existing papers. Here, we would like to further clarify our work's major differences and unique contributions, compared with ~\cite{yao2024tree}, \cite{min2022rethinking} and~\cite{chia2023contrastive}.

As discussed in~\ref{appendix: related work-ICL},~\citet{yao2024tree} proposes Tree-of-Thought~(ToT) for deliberate reasoning tasks. ToT decomposes a given problem into several reasoning steps and generates multiple answers for each step, ultimately constructing a tree structure. Subsequently, ToT employs BFS or DFS to traverse the tree, yielding the final rationale and answer. Namely, ToT focuses on the rationale (thoughts) generation in the output space of an LLM, wherein the generated thoughts can be irrelevant or inaccurate. By contrast, Our work focuses on the input space with noisy prompting examples. From this view of input space and output space, ToT is orthogonal to our work.

Moreover, ToT focuses on problem-solving by reasoning, \textit{e.g.}, Game of 24 and Mini Crosswords, while our work investigates the robustness problem of noisy rationales. ToT does not consider imperfect scenarios with noisy inputs. This means that when ToT is prompted by examples with noisy rationales, the reasoning might also be degenerated. 
Hence, a further investigation on reasoning robustness with these up-to-date prompting techniques, possibly with a focus on the output space, can be conducted based on our work.

As for~\cite{min2022rethinking}, this work explores the question-answer mapping of prompting examples. They find that breaking the mapping has nearly no influence on LLM's in-context learning. 

However, it's important to note that~\cite{min2022rethinking} primarily focuses on the relationship between input and label, investigating the question-answer mapping of prompting examples. Their work does not delve into the role of rationales in the examples, which is the central focus of our research. In contrast, our experiments keep the input-label mapping intact and instead explore the impact of introducing noise into the rationales. This allows us to isolate the effect of rationale quality on the model's performance without altering the fundamental relationship between questions and answers.

Hence, when the form of prompting example shifts from (question, answer) to (question, rationale, answer), more aspects can be investigated. We conduct additional experiments to further explore the mapping among question, rationale, and answer in Sec.~\ref{sec: experiment}. 

\citet{chia2023contrastive} proposes Contrastive Chain-of-Thought (CC)~\cite{chia2023contrastive} to enhance LLM reasoning. In our work, we focus on improving reasoning in noisy settings and define the noisy rationale problem as “factually inaccurate or irrelevant reasoning steps paired with valid question-answer prompts.” Fig.~\ref{fig: problem} shows an example. Here, only one rationale is given in each demonstration, which can be potentially noisy but unknown to the model.
However, in CC’s setting, each ICL demonstration explicitly includes a clean rationale and a wrong rationale. Notably, this rationale is wrong instead of noisy as it induces the wrong answer. Therefore, CC’s setting and ours are totally different.
In addition, empirical results in Tab.~\ref{tab:proposed-method performance} show that our method CD-CoT significantly outperforms CC when given the same information.

The problem of noisy rationale is challenging and non-trivial. We believe that the investigation of noisy rationale is still in the budding stage. Several potential solutions can be progressively devised to address this problem. For example, accessing external databases is a promising solution. However, this retrieve augmented generation (RAG) is usually used to address the lack of knowledge of LLMs but not for data denoising. Therein, verifying the factual relevance and accuracy of a rationale is quite challenging: (1) there is a semantic gap between a rationale and the corresponding query instructions of a database; (2) how can the models realize the need for data verification and then find a suitable database for that are also cutting-edge research directions?

Hence, we believe that RAG for denoising is not a trivial solution. More research attention and effort should be paid to achieve this solution.
In this work, we verify that an LLM can achieve good denoising results when prompted with only one clean rationale and several noisy rationales. This solution, sharing a similar spirit to RAG that requires additional information, significantly mitigates the intrinsic vulnerabilities of LLMs.

\clearpage
\section{Benchmark}
\label{appendix: benchmark}

In this section, we provide a comprehensive overview of the NoRa benchmark, including 
the motivation for constructing NoRa (Appendix~\ref{appendix: benchmark-motivation}),
a further discussion on the origins and impact of noisy rationales in CoT demonstrations (Appendix~\ref{appendix: benchmark-noise-discussion}), 
real-world examples of noisy rationales from both model and human perspectives (Appendix~\ref{appendix: real-world examples}), 
the detailed process of generating noise for the NoRa dataset (Appendix~\ref{appendix: benchmark-details}), a comprehensive datasheet for the NoRa dataset (Appendix~\ref{appendix:datasheet}),
and full examples of the NoRa dataset across various reasoning tasks (Appendix~\ref{appendix: benchmark-examples}).

\subsection{Motivation}
\label{appendix: benchmark-motivation}

The robustness of LLMs against noisy inputs deserves more attention, particularly in the context of Chain-of-Thought (CoT) reasoning. While CoT methods have shown remarkable effectiveness, they face two fundamental challenges: the susceptibility to selected demonstrations and the complexity of generating high-quality demonstrations.

Recent studies have revealed that LLMs can be easily distracted by irrelevant or adversarial context, and their reasoning process is unstable under small prompt modifications. Furthermore, misleading information can cause LLMs to agree with factual errors. These vulnerabilities become particularly concerning when we consider that demonstrations in practice are not always clean or accurate.

The problem of noisy rationales in demonstrations arises from various practical scenarios:
\begin{enumerate}
\item \textbf{Human-generated noise}: Just as humans make mistakes in data annotation (which has motivated extensive research in label-noise learning), they can also provide imperfect demonstrations. Different users with varying expertise and background knowledge may contribute demonstrations of varying quality.
\item \textbf{System-generated noise}: Noise can emerge from automated systems, including dialogue systems and machine-generated data, especially in crowdsourced platforms.
\item \textbf{Demonstration selection challenges}: While incorporating human supervision in dealing with noisy rationales is possible, it is often expensive and impractical at scale. The sensitivity of LLMs to in-context learning examples makes this particularly challenging.
\end{enumerate}
Our empirical studies have revealed widespread vulnerability among prevailing LLMs to noisy rationales, with existing reasoning methods showing limited efficacy. Compared to clean rationales, we observe performance decreases of 0.4-19.8\% with irrelevant noise and more severe degradation of 2.0-66.8\% with inaccurate noise.

The NoRa benchmark addresses this crucial gap in current research. While most existing work assumes clean demonstrations for in-context learning, we argue that noisy rationales represent a practical and challenging problem that requires immediate attention. This benchmark serves two primary purposes:

\begin{enumerate}
\vspace{-5pt}

\item It provides a systematic way to evaluate LLM robustness against different types of noise in rationales.

\item It helps identify the limitations of current approaches and motivates the development of more robust reasoning strategies.
\end{enumerate}

We believe the NoRa dataset and our findings will contribute significantly to the development of more trustworthy foundation models capable of handling real-world, imperfect demonstrations.

\subsection{A Further Discussion on Noisy Rationales in CoT Demonstrations}
\label{appendix: benchmark-noise-discussion}

In this part, we further clarify the background and settings of the Noisy-R problem investigated in our work, \textit{i.e.}, noisy rationales in CoT demonstrations. Having established the presence of irrelevant and inaccurate thoughts as primary noise types, we now turn to their origins and impact within human-model interactions.

\textbf{Irrelevant and inaccurate thoughts can be generated by both the model and humans.}
On the one hand, irrelevant or inaccurate thoughts are likely to appear within the rationales of LLMs' answers. These rationales, drawn from historical dialogues between users and the model, can create a noisy context for new conversations. On the other hand, human demonstrators actively provide in-context demonstrations to enhance the reasoning capabilities of models. However, these demonstrations, whether crafted by humans or sourced from datasets, can also contain noisy rationales.

\textbf{Irrelevant thoughts - model perspective:}
LLMs tend to explain the concept of terms during reasoning.
For instance, if you ask GPT-4 to debug an error related to the 'concurrent' package, it tends to start by explaining what the 'concurrent' package is rather than immediately addressing the debug request. Such explanations, while informative, may introduce irrelevant noise into the conversation. 
Here are some key reasons for their emergence.

\begin{enumerate}
\vspace{-5pt}

\item \textbf{Cognitive bias of models.} When addressing complex queries, LLMs tend to include explanations for terms or concepts mentioned in the prompt that are unnecessary for solving the specific problem presented, as a kind of irrelevant thought. This behavior exhibits a form of cognitive bias where the model, unable to assess the inquirer's level of understanding, leads to explicating background information. This is comparable to a lecturer explaining the basics of a subject to a class without first assessing their students' existing knowledge, potentially leading to irrelevant elaborations.


\item \textbf{Lack of precise contextual understanding.}  Despite LLMs' proficiency in processing language and recognizing patterns, they cannot always fully grasp the precise context or specific requirements of a problem. This shortfall can result in the production of thoughts that, although seemingly related, do not directly contribute to resolving the question at hand. Their responses might align more with the broader theme of the discussion rather than the specific, critical details needed for a precise solution. 

\item \textbf{User query ambiguity of the dialogue mechanism.} Ambiguities in user queries can stem from the use of vague or multifaceted language, prompting the model to generate a wide array of responses. This situation is similar to a search engine returning a variety of results for a query that lacks specificity. The model, attempting to cover all potential meanings of the query, may produce responses that contain thoughts unrelated to the user's actual intent. For instance, if someone asks, ``What is force calculation?'', the model might provide information on both Newtonian mechanics and quantum mechanics. These responses, encompassing a broad range of topics, could influence the answers to subsequent physics questions, leading to a continuation of the ambiguity and further complicating the conversation.

\item \textbf{Progressive disclosure bias of the dialogue mechanism.} When engaging with LLMs, individuals often initiate the dialogue by describing simpler concepts and progressively work their way toward articulating the complex question at hand because of human limitations in language or comprehension abilities, which prevent a person from presenting the entire complexity of an issue in one go. This step-by-step approach, while natural for humans who struggle to directly convey intricate problems, can introduce extraneous content that contributes to noise within the model's contextual reasoning. As the conversation builds, the LLM will factor in these initial, possibly tangential, explanations into its understanding of the context, potentially leading to a dilution of the focus necessary for solving the specific issue. This phenomenon reflects a human cognitive strategy in communication that may not be optimally aligned with the operational mechanisms of LLMs for efficient problem-solving.






\end{enumerate}

\textbf{Irrelevant thoughts - human perspective:} 
When a human is actively demonstrating CoT reasoning, the introduction of irrelevant thoughts could be due to a variety of reasons.

\begin{enumerate}
\vspace{-5pt}

\item \textbf{Cognitive overload of humans.} Humans may introduce irrelevant information when they are trying to process too much information at once, which can lead to a loss of focus and the inclusion of tangential thoughts. For instance, a programmer is struggling with a bug in a complex piece of software and asks an LLM for help. To provide context, the programmer starts explaining the issue with a CoT rationale approach, intending to walk the LLM through their thought process. However, due to the complexity of the code and the stress of finding the bug, the programmer gets sidetracked. They include unnecessary details about the different error messages encountered in the past, unrelated functions in the code, and general thoughts on software development.

\item \textbf{Associative thinking of humans.} Humans naturally think in an associative manner, where one thought may lead to another that is only loosely related to the task at hand. This can result in straying from the main point during a CoT explanation. For example, while a programmer outlines the steps to diagnose a software issue for an LLM, they might recall a similar problem encountered in a different project. This memory could lead them to mention troubleshooting strategies, tools, or anecdotes from that experience, which, although related to the broader theme of problem-solving, do not directly contribute to the current issue.

\item \textbf{Irrelevant content in datasets.} In the future, it is likely that companies or professional organizations will increasingly utilize databases to assemble CoT prompts. However, these databases, whether privately maintained or publicly accessible, can contain irrelevant reasoning processes. This is especially true for databases sourced from crowdsourcing platforms or open forums, where the information is contributed by a diverse set of individuals with varying levels of expertise and focus. When these datasets are used to provide in-context information for CoT reasoning, the noise can originate from the inclusion of off-topic discussions, personal opinions, or overly verbose explanations that do not directly address the problem at hand. Such noise can be inadvertently introduced into the CoT process when humans provide explanations that contain unnecessary or tangential information.



\end{enumerate}

Similarly, we analyze the two sources of inaccurate thoughts as follows.

\textbf{Inaccurate thought - model perspective:}
For models, LLMs may produce erroneous thoughts during the reasoning process, especially when dealing with complex problems. For example, when tackling a base-9 math problem in a zero-shot setting, GPT-3.5 may generate some inaccurate reasoning steps. The former dialogue will become inaccurate and noisy in the context of subsequent dialogues. Here are some key reasons for their emergence.


\begin{enumerate}
\vspace{-5pt}

\item \textbf{Outdated or incomplete training data of the model.} Language models are built upon datasets that may not be current or fully comprehensive. When faced with problems that require up-to-date knowledge or complete understanding, which are absent in their training data, models may rely on outdated or incomplete information, resulting in inaccurate outputs. For example, in the field of medicine, if new research suggests a change in treatment protocol after the model's last update, it wouldn't be able to advise on the new information.

\item \textbf{Adaptation to novel reasoning contexts of the model.} New challenges may require models to reason within contexts that slightly or significantly differ from their training data. For instance, a model extensively trained on base-10 arithmetic might struggle with a base-9 math problem because it requires a shift in the underlying numerical framework. This kind of scenario demands on-the-fly adaptation to a novel reasoning context, which can lead to generating thoughts that do not accurately apply the learned principles from the base-10 system to the newly introduced base-9 system.

\item \textbf{Misinterpretation of complex subjects of dialogue mechanism.} Users often fail to clearly articulate their complete requirements at the outset of an inquiry, leading to LLMs generating misunderstandings and inaccurate thoughts that do not align with user expectations. The process of correcting these thoughts is inherently a reasoning process laden with noisy contexts. As users provide feedback to refine the model's output, the iterative nature of this interaction can introduce additional inaccuracies as the model attempts to reconcile the new information with the previously misunderstood context.

\end{enumerate}

\textbf{Inaccurate thought - human perspective:}
 Inaccurate thoughts in CoT can stem from the information provided by humans, whether it is self-made on the spot or sourced from a database for in-context learning by LLMs. These CoT demos can include inaccurate noise due to various factors.

\begin{enumerate}
\vspace{-5pt}

\item \textbf{Personal knowledge limitations of human.} Individuals may possess incomplete or outdated knowledge on a given subject, leading to the provision of incorrect information when creating a CoT. For instance, a person without expertise in mathematics might attempt to construct a CoT for a complex math problem and inadvertently introduce incorrect steps or conclusions. Their understanding may be based on heuristics or educational background that hasn't been updated to reflect more recent methodologies or discoveries in the field.

\item \textbf{Cognitive biases of human.} Human reasoning can be influenced by a range of cognitive biases, such as confirmation bias, where an individual tends to search for, interpret, and remember information in a way that confirms their preconceptions, neglecting contrary information. Or the oversimplification of complex issues might lead to inaccurate reasoning steps within a CoT. These biases can skew the logic flow and result in conclusions that do not hold up under scrutiny or are based on flawed premises.

\item \textbf{Data quality issues of database.} The databases that humans rely on for creating CoTs might contain errors or biases introduced during data collection and processing. If this flawed data is used for in-context learning by LLMs, it can impart incorrect patterns of thought or factual inaccuracies. For example, a dataset with biased sampling methods might lead to generalizations that do not accurately represent the broader population or situation.

\item \textbf{Contextual misplacement of databases.} Information from databases may be stripped of its original context, leading to misinterpretation when reused. When humans include such decontextualized information in a CoT, they might not properly align it with the new context, introducing misunderstandings or inaccuracies. This is particularly problematic in nuanced fields where context heavily influences the meaning and applicability of information, such as legal precedents or cultural studies.

\end{enumerate}

Given the convenience and adaptability of CoT reasoning, broader adoption in LLM applications is expected in the future. This structured approach enables LLMs to break down complex problems and explain their reasoning in a way that resembles how humans think, proving essential for sophisticated problem-solving. Nonetheless, we are bound to face the noisy reasoning challenges, stemming from both model-generated and human-contributed contexts, as mentioned above. 

To address these challenges, we must focus on continuously improving training methods, keeping models updated with the latest information, enhancing their ability to parse context and ambiguity, and refining algorithms to diminish biases and logical inaccuracies.


\textbf{Differences between inaccurate and irrelevant thoughts.}
Given the context and question, a thought is either relevant or irrelevant. Within the relevant thoughts, those accurate ones are desirable, and those inaccurate, termed as "inaccurate thoughts" in this work, are proven to be harmful to LLMs. Hence, the extreme case of "inaccurate and irrelevant thought" is not covered in this work. 
For clarity and simplicity, we investigate these two kinds of noise separately. Specifically, in definitions,

\begin{itemize}[label=\textbullet]
\vspace{-5pt}

\item {Irrelevant} thoughts refer to incorporating extraneous details that are unhelpful for solving the question. Redundant information may be introduced by the LLM’s diverse response generation or by humans when clarifying concepts in problem-solving examples;
\item {Inaccurate} thoughts refer to factual errors in rationales that are common in mathematical calculation or transcription. The emergence of noise can be due to algorithmic limitations, errors in training data, misinterpretations of context or instructions, and logical fallacies.
\vspace{-5pt}
\end{itemize}
Note that the "accuracy" and "relevance" of thoughts are related to the context of the given question. Basically, the question provides the context, and an LLM generates thought conditional on the context. Taking the examples of the Base-9 dataset in Tab.~\ref{tab:dataset-examples},

\begin{itemize}[label=\textbullet]
\vspace{-5pt}
\item {Irrelevant} thought is "There are five oceans on Earth: the Atlantic, Pacific, Indian, Arctic, and Southern.", which is accurate but not relevant and not useful;
\item {Inaccurate} thought here is "5 + 9 = 14", which is relevant to the question and previous thought but is inaccurate for the base-9 calculation.
\vspace{-5pt}
\end{itemize}

In the NoRa dataset, only the relevant thought with factual error will be classified as an "inaccurate thought". There is no irrelevant and inaccurate thought in NoRa, which should be rare in practice. 

We focus on the two major types of noise in this work. And empirically, inaccurate thoughts bring severe degradations. Compared with clean rationales, a 1.4\%-19.8\% decrease with irrelevant noise and a more drastic 2.2\%-40.4\% decrease with inaccurate noise.

\subsection{Real-world Examples}
\label{appendix: real-world examples}


The emphasis on noisy rationale is due to its practical challenges, with examples drawn from diverse sources such as crowdsourced platforms, dialogue systems, and AI-generated data. Here, the Noisy-R mainly originates from (1) the inherent imperfections, inconsistencies, and inaccuracy of humans’ cognitive processes and (2) the diversity, unpredictability, and hallucination of the LLMs’ generative mechanisms, as discussed in more detail in Appendix \ref{appendix: benchmark-noise-discussion}. 

Briefly, irrelevant and inaccurate thoughts can be generated by both the model and humans. From model perspective, the generated rationales by the model can be noisy. From human perspective, the annotated rationales by humans can be noisy.

In the following four parts, we elaborate on the evidence (including sources and examples) of irrelevant and inaccurate thoughts from both the model perspective and the human perspective.

\subsubsection{Irrelevant Thoughts from Model Perspective}

\textbf{Cognitive bias of models:} 
\begin{itemize}
\item 
\textbf{Source.}
The presence of cognitive biases in LLMs is demonstrated, and the implications of using biased reasoning under the guise of expertise are discussed, calling for stronger education, risk management, and continued research as the widespread adoption of this technology increases~\cite{koh2022inversion}. 

\citet{thorstad2023cognitive} examines the inherent cognitive biases within large language models, suggesting that while there is reason for cautious optimism regarding the potential to mitigate these biases, a realistic, anti-Panglossian approach is necessary to drive meaningful improvements. It advocates for a balanced perspective that neither idealizes the perfectibility of AI nor dismisses the progress made in addressing its limitations. 

\item 
\textbf{Examples.} 
Tab.~\ref{tab: Cognitive-bias-examples} demonstrates examples of cognitive bias of models.

\begin{table*}[ht]
    \centering
    \fontsize{8}{8}\selectfont
    \begin{tabular}{p{0.97\textwidth}}
        \toprule
        \multicolumn{1}{c}{Example 1} \\
        \midrule
         \textbf{Question:} What rubber toy filled with helium will make a child happy? Answer Choices: (A) loved (B) learn \(C\) eat cake (D) balloon (E) become adult
    
         \textbf{Rationale:} Let's think step by step. A rubber toy filled with helium will make a child happy if the child loves the toy. If the child loves the toy, he or she will learn to play with it. \noisy{If the child learns to play with it, he or she will eat cake. If the child eats cake, he or she will become an adult.}
    
        \textbf{Answer:} The answer is D. 
    
        \textbf{Explanation:} There is information related to the answer in reasoning, but there is also noise information that is irrelevant to the final answer, and it has nothing to do with whether the child is happy or not because he becomes an adult because he eats the cake.\\
        \midrule
        
        \multicolumn{1}{c}{Example 2} \\
        \midrule
        \textbf{Question:} What might a couple have a lot of when they are deciding on stopping being married to each other? Answer Choices: (A) pleasure (B) detachment \(C\) exercise (D) bankruptcy (E) fights\\
        \textbf{Rationale:} Let's think step by step. \noisy{A couple might have a lot of pleasure when they are deciding on stopping being married to each other if they are both in agreement with the decision. A couple might have a lot of detachment when they are deciding on stopping being married to each other if they are both indifferent to the decision. A couple might have a lot of exercise when they are deciding on stopping being married to each other if they are both physically active and need to work out their aggression. A couple might have a lot of bankruptcy when they are deciding on stopping being married to each other if they are both in financial trouble.} A couple might have a lot of fights when they are \\
        \textbf{Answer:} Therefore, among A through E, the answer is E.\\
        \textbf{Explanation:} Only the last sentence mentions that couples tend to quarrel in the event of divorce, and the previous explanations and reasoning about the other options are nonsensical\\
        \bottomrule
    \end{tabular}
    \caption{Cognitive bias of models of the model examples.}
\vspace{-5pt}
\label{tab: Cognitive-bias-examples}
\end{table*}

\end{itemize}

\textbf{Lack of precise contextual understanding:}
\begin{itemize}
\item
\textbf{Source.} \citet{zhu2024large} introduces a context understanding benchmark by adapting existing datasets to suit the evaluation of generative models’ ability to understand the context and finds that pre-trained dense models struggle with understanding more nuanced contextual features when compared to state-of-the-art fine-tuned models.

\citet{he2024can} designs eight features for complex instructions and construct a comprehensive evaluation dataset from real-world scenarios, and compares the performance of representative Chinese-oriented and English-oriented models in following complex instructions through extensive experiments. 


\item
\textbf{Examples.} Tab.~\ref{tab:lack-of-precise-contextual-understanding-examples} demonstrates the lack of precise contextual understanding in LLMs.
\end{itemize}

\begin{table*}[ht]
    \centering
    \fontsize{8}{8}\selectfont
    \begin{tabular}{p{0.97\textwidth}}
        \toprule
        \multicolumn{1}{c}{Example 1} \\
        \midrule
        \textbf{Question:} He was trying to procreate with many individuals, this led to a what? Answer Choices: (A) moaning (B) die \(C\) kiss (D) std (E) sanity\\
        \textbf{Rationale:} Let's think step by step. He was trying to procreate with many individuals. This would lead to many offspring, which would be a good thing, right? \noisy{But it also led to a lot of competition among the offspring for resources. This would lead to a lot of fighting and eventually to the death of many individuals. So while he may have been trying to create more life, he also caused a lot of death.}\\
        \textbf{Answer:} Therefore, among A through E, the answer is D.\\
        \textbf{Explanation:} The reasoning is mainly about the fact that many people will have a lot of offspring, and the offspring will bring competition for resources and wars, which will cause a lot of deaths, but the final result of the reasoning is that it is indeed a sexually transmitted disease, and there is no information related to the sexually transmitted disease in the reasoning. So it is irrelevant thought.\\
        \midrule
        \multicolumn{1}{c}{Example 2} \\
        \midrule
        \textbf{Question:} Can Hulk's alter ego explain atomic events?\\
        \textbf{Rationale:} Let's think step by step. First, we need to understand what an atom is. \noisy{Atoms are the basic units of matter and the defining structure of elements. The term "atom" comes from the Greek word for indivisible, because it was once thought that atoms were the smallest things in the universe and could not be divided. The structure of an atom is a central nucleus composed of protons and neutrons with electrons orbiting around this nucleus.} Now that we know what an atom is, we can move on to understanding what an atomic event is.  An atomic event is a process that involves.\\
        \textbf{Answer:} Yes, Hulk's alter ego can explain atomic events.\\
        \textbf{Explanation:} The model's rationale focuses on defining atoms and atomic events but fails to connect this information to the context of the question, which is about the ability of Hulk's alter ego, Bruce Banner, to explain these concepts. The model does not demonstrate an understanding of the alter ego's scientific expertise and instead provides a general explanation of atomic structure. This lack of precise contextual understanding leads to an unsupported conclusion.\\
        \bottomrule
    \end{tabular}
    \caption{Examples demonstrating the lack of precise contextual understanding in LLMs.}
    \vspace{-5pt}
    \label{tab:lack-of-precise-contextual-understanding-examples}
\end{table*}

\subsubsection{Irrelevant Thoughts from Human Perspective}



\textbf{Associative thinking of humans:}
\begin{itemize}
\item 
\textbf{Source.} \citet{morewedge2010associative} investigates the underlying cognitive mechanisms of intuition, focusing on how unconscious associative processes can lead to rapid and efficient problem-solving and decision-making. It aims to elucidate the often-overlooked influence of automatic, associative thinking on our intuitive judgments, challenging the traditional emphasis on deliberate and logical reasoning in cognitive psychology. 

\citet{wimmer2012preference} presents evidence that the hippocampus is involved in decision-making by enabling the transfer of value between memories, affecting choices among novel alternatives. Using functional brain imaging, the study shows that monetary rewards activate a network of memories, influencing individuals to later choose associated, but unrewarded, items, with this bias being predicted by hippocampal activity and its connections with memory and reward brain regions. 

\item 
\textbf{Examples.} Please refer to Tab.~\ref{tab:associative-thinking-examples} for examples.
\end{itemize}

\begin{table*}[ht]
    \centering
    \fontsize{8}{8}\selectfont
    \begin{tabular}{p{0.97\textwidth}}
        \toprule
        \multicolumn{1}{c}{Example 1} \\
        \midrule
        \textbf{Question:} How can the relationship between John and Mary be determined within the family tree, given that they share a grandfather but have different fathers?

        \textbf{Rationale:}
    
        1. Identifying the Common Grandfather: The first step is to confirm that John and Mary share a grandfather, indicating a direct blood relation within the family tree.

        2. Analyzing the Fathers' Identity: Since they have different fathers, this implies their fathers are brothers.

        3. An Interesting Anecdote: \noisy{During a family gathering, John and Mary participated in a "Who's Who" family game where they were mistakenly identified as siblings rather than cousins.}

        4. Considering the shared grandfather and that their fathers are brothers, it's established that John and Mary are first cousins.

        \textbf{Answer:} John and Mary are first cousins.
    
        \textbf{Explanation:} When analyzing the relationship, an anecdote that is loosely related to the task is considered. Although the anecdote is about the relationship between John and Mary, it does not help in reasoning.\\
        \midrule
        \multicolumn{1}{c}{Example 2} \\
        \midrule
        \textbf{Question:} Calculate the length of the hypotenuse in a right-angled triangle with one leg of length 3 and the other leg of length 4.\\

        \textbf{Rationale:}\\
        1. Using Pythagoras' Theorem: In a right-angled triangle, the sum of the squares of the lengths of the legs is equal to the square of the length of the hypotenuse, \textit{i.e.}, $a^2 + b^2 = c^2$. \noisy{Pythagoras, an ancient Greek mathematician, is credited with this discovery, but his contributions to mathematics extend beyond this, including music, philosophy, and astronomy.}\\
        2. Substituting Known Values to Solve: In our problem, the lengths of the legs are 3 and 4, thus $3^2 + 4^2 = c^2$, which is $9 + 16 = c^2$. This leads to $c^2 = 25$.\\
        3. Solving the Equation to Find the Length of the Hypotenuse: Solving $c^2 = 25$ yields $c = 5$. \noisy{This solution not only reveals the geometric property of the triangle but also hints at the special status of the 3-4-5 triplet as one of the simplest Pythagorean triples. Pythagorean triples are sets of three positive integers that satisfy Pythagoras' theorem, with 3-4-5 being among the most well-known sets, often used in construction and engineering to create right angles.}\\
    \textbf{Answer:} The length of the hypotenuse is $5$.\\
    \textbf{Explanation:} Introduction of Pythagoras and the special status of the 3-4-5 triplet do not help in calculation of the length of the hypotenuse, which is unnecessary detail about the problem.\\
        \bottomrule
    \end{tabular}
    \caption{Examples of associative thinking of humans}
    \vspace{-2pt}
    \label{tab:associative-thinking-examples}
\end{table*}

\clearpage
\textbf{Irrelevant content in datasets:}
\begin{itemize}
\item 
\textbf{Source.} \citet{havrilla2024understanding} examines how different types and intensities of noise in training data affect the performance of LLMs using a novel framework called Traced Integer (TInt) for generating customizable noised execution traces. The study reveals that while fine-tuned models display high robustness to static noise, they are more susceptible to dynamic noise, and it offers insights into the importance of filtering out training samples with dynamic noise to improve model performance.

\item 
\textbf{Examples.} Please refer to Tab.~\ref{tab:irrelevant-content-datasets-examples} for examples.
\end{itemize}

\begin{table*}[ht]
    \centering
    \fontsize{8}{8}\selectfont
    \begin{tabular}{p{0.97\textwidth}}
        \toprule
        \multicolumn{1}{c}{Example 1} \\
        \midrule
        \textbf{Question:} Calculating the Area of a Circle with Radius 7\\
        \textbf{Rationale:} \\
        1. Using the Area Formula: The area of a circle can be calculated using the formula $A = \pi r^2$, where $A$ is the area and $r$ is the radius of the circle. $\pi$ (pi) is a mathematical constant approximately equal to 3.14159, representing the ratio of a circle's circumference to its diameter. \noisy{$\pi$ is not only widely used in mathematics but also serves as a fundamental element in physics, engineering, and other scientific fields.}\\
        2. Substituting the Given Value: Substituting the radius $r = 7$, we get $A = \pi \times 7^2 = 49\pi$. \noisy{This result not only provides the area of the circle but also reminds us of the irrational nature of $\pi$ – an infinite, non-repeating decimal, representing a form of endless beauty in mathematics.}\\
        3. Practical Application of the Area: In practice, calculating the area of a circle has many applications, such as designing fountains in parks, making round tables, or calculating the surface area of planets. This simple formula underlies limitless exploration and potential applications.\\
        \textbf{Answer:} The area of the circle is $49\pi$.\\
        \textbf{Explanation:} Application about $\pi$ does not help in calculation of area of a circle, which is unnecessary detail about the problem.\\
        \midrule
        \multicolumn{1}{c}{Example 2} \\
        \midrule
        \textbf{Question:} Calculate the length of the hypotenuse in a right-angled triangle with one leg of length 3 and the other leg of length 4.\\
        \textbf{Rationale:}\\
            1. Using Pythagoras' Theorem: In a right-angled triangle, the sum of the squares of the lengths of the legs is equal to the square of the length of the hypotenuse, \textit{i.e.}, $a^2 + b^2 = c^2$. \noisy{Pythagoras, an ancient Greek mathematician, is credited with this discovery, but his contributions to mathematics extend beyond this, including music, philosophy, and astronomy.}\\
            2. Substituting Known Values to Solve: In our problem, the lengths of the legs are 3 and 4, thus $3^2 + 4^2 = c^2$, which is $9 + 16 = c^2$. This leads to $c^2 = 25$.\\
            3. Solving the Equation to Find the Length of the Hypotenuse: Solving $c^2 = 25$ yields $c = 5$. \noisy{This solution not only reveals the geometric property of the triangle but also hints at the special status of the 3-4-5 triplet as one of the simplest Pythagorean triples. Pythagorean triples are sets of three positive integers that satisfy Pythagoras' theorem, with 3-4-5 being among the most well-known sets, often used in construction and engineering to create right angles.}\\
        \textbf{Answer:} The length of the hypotenuse is $5$.\\
        \textbf{Explanation:} Introduction of Pythagoras and the special status of the 3-4-5 triplet do not help in calculation of the length of the hypotenuse, which is unnecessary detail about the problem.\\
        \bottomrule
    \end{tabular}
    \caption{Examples of irrelevant content in datasets}
    \vspace{-10pt}
    \label{tab:irrelevant-content-datasets-examples}
\end{table*}

\subsubsection{Inaccurate Thought from Human Perspective}

\textbf{Personal knowledge limitations of human:}
\begin{itemize}
\item 
\textbf{Source.} The array of problems addressed by the human mind constitutes merely a specific subset within the spectrum of challenges tackled by AI systems \cite{griffiths2020understanding}, resulting from human time limitation and human computation limitation, etc.

Due to the limited lifespan, whatever a human learns must be acquired in less than a million hours of real-time experience. The AlphaGo system that beat the best human players in history has the benefit of multiple human lifetimes of simulated play \cite{silver2016mastering}.

There exists a disparity between human intelligence  and machine intelligence in addressing the Rubik's Cube conundrum. Machine intelligence has successfully determined the shortest path from any initial configuration to an unscrambled state, requiring no more than 20 moves \cite{rokicki2014diameter}. Conversely, the most adept human solvers necessitate approximately 50 moves to achieve the same outcome.

\item 
\textbf{Examples.} Please refer to Tab.~\ref{tab:human-knowledge-examples} for examples.
\end{itemize}

\begin{table*}[ht]
    \centering
    \fontsize{8}{8}\selectfont
    \begin{tabular}{p{0.97\textwidth}}
        \toprule
        \multicolumn{1}{c}{Example 1} \\
        \midrule
        \textbf{Question:} What should be the next move in the game of Go?\\
        \textbf{Rationale:} In order to win this game, I need to capture more of the opponent's pieces. \noisy{Since capturing as many of the opponent's pieces as possible is advantageous,} it is now time to launch an offensive and encircle their pieces.\\
        \textbf{Answer:} Therefore, the next move is at position A.\\
        \textbf{Explanation:} In the context of playing Go, a deficiency in holistic board assessment and the incorporation of offensive strategies leads to inaccurate thought. This narrow focus solely on offensive tactics detrimentally impacts the overall outcome of the game, hindering the player's ability to achieve victory across the entirety of the board.\\
        \midrule
        \multicolumn{1}{c}{Example 2} \\
        \midrule
        \textbf{Question:} How long will it take to solve the Rubik's Cube at the fastest?\\
        \textbf{Rationale:} One rotates a Rubik's Cube in 0.2 second, only needing to know how many steps it takes to solve one. \noisy{Someone solve a Rubik's Cube requires a minimum of 50 steps.} It takes 4 seconds to solve the Rubik's Cube at the fastest.\\
        \textbf{Answer:} Therefore, the answer is 4 seconds.\\
        \textbf{Explanation:} In assessing the completion time of the Rubik’s Cube conundrum, due to a lack of knowledge regarding shortest path algorithms and planning, an individual relies solely on experiential intuition, presuming that a minimum of 50 steps is necessary, whereas in reality, the puzzle can be completed within 20 steps or fewer.\\
        \bottomrule
    \end{tabular}
    \caption{Examples of human knowledge limitation.}
    \vspace{-2pt}
    \label{tab:human-knowledge-examples}
\end{table*}




\textbf{Cognitive biases of human:}
\begin{itemize}

\item
\textbf{Source.} \citet{tversky1974judgment} define human cognitive biases: systematic patterns of deviation from rational judgment. Here we list three types of cognitive biases of human.

Confirmation bias is a common cognitive bias, referring to the tendency of people to seek, remember, or emphasize information that supports their existing views or beliefs, while ignoring or disregarding information that contradicts their viewpoints~\cite{del2017modeling}.

When members within a group suppress their individual capacity for independent thinking in order to maintain team harmony or consensus, groupthink occurs~\cite{janis2008groupthink}.

Emotional reasoning happens when people judge the authenticity or value of things based on their own emotional state~\cite{norris2000emotional}.

\item
\textbf{Examples.} Please refer to Tab.~\ref{tab:cognitive-bias-examples} for examples.
\end{itemize}
\begin{table*}[ht]
    \centering
    \fontsize{8}{8}\selectfont
    \begin{tabular}{p{0.97\textwidth}}
        \toprule
        \multicolumn{1}{c}{Example 1} \\
        \midrule
        \textbf{Question:} Alice likes going to the coffee shop. Does Alice like drinking tea?\\
        \textbf{Rationale:} Alice likes going to the coffee shop. \noisy{People who enjoy going to coffee shops must enjoy drinking coffee.} However, individuals who enjoy drinking coffee may not necessarily enjoy drinking tea. Therefore, it cannot be inferred whether Alice enjoys drinking tea.\\
        \textbf{Answer:} Therefore, not sure if Alice likes drinking tea.\\
        \textbf{Explanation:} Not all individuals who enjoy frequenting coffee shops necessarily enjoy drinking coffee. There is no inherent correlation between enjoying coffee shop visits and liking coffee. Here, the introduction of confirmation bias plays a pivotal role.\\
        \midrule
        \multicolumn{1}{c}{Example 2} \\
        \midrule
        \textbf{Question:} Which one is the middle-aged woman? (A) The beautifully dressed slender woman, (B) The dowdily dressed slightly overweight woman,(C)Not sure.\\
        \textbf{Rationale:} This woman is middle-aged. \noisy{She needs to take care of her family and children, and she has no time to dress up or manage her figure.} She could be either overweight or slim, inclined towards dressing up or adopting a more casual appearance; hence, we cannot ascertain her physical appearance.\\
        \textbf{Answer:} Therefore, the answer is (C).\\
        \textbf{Explanation:} In the above example, groupthink leads to the maximization of societal and individual family interests, hence the collective belief that a middle-aged woman should be overweight and dowdily dressed. However, the reality is that middle-aged women are diverse and cannot be stereotyped.\\
        \bottomrule
    \end{tabular}
    \caption{Examples of cognitive bias.}
    \vspace{-2pt}
    \label{tab:cognitive-bias-examples}
\end{table*}

\subsection{Noise Generation}
\label{appendix: benchmark-details}

As detailed in Sec.~\ref{sec: dataset}, we have introduced both irrelevant and inaccurate noises into our dataset as insertions. These insertional noises are carefully integrated into the rationales, ensuring they neither modify the existing reasoning pathways nor affect the final answers. Furthermore, they are purposefully crafted based on the intermediate steps of the reasoning sequence, which guarantees that each piece of noise is contextually related to the stage of reasoning it accompanies. To maintain consistency, we treat each sentence in the reasoning sequence as a single thought; accordingly, each noise we introduce is also fashioned as only one sentence. For irrelevant noise, we draw from unrelated scientific or social facts. Alternatively, inaccurate noise involves the extension of an intermediate thought, which is redundant and incorrect.

\textbf{NoRa-Math.} In NoRa-Math problems, an intermediate reasoning result typically manifests as a numerical value. For instance, from the equation "We have 4 + 2 = 6 in base-10. ", we take the derived number $6$ and craft a sentence of noise to follow this particular reasoning step. Take the numeral "6"; we might introduce an unrelated fact such as "According to the Standard Model of physics, there are six types of quarks, the fundamental constituents of matter." For inaccurate thoughts, we generate a decimal addition related to $6$, such as "6 + 5 = 11. ", which is inaccurate in base-9 representation. 

\textbf{NoRa-Symbolic.} NoRa-Symbolic problems are handled similarly; here, intermediate results are individual elements of a navigation instruction. We insert a sentence of noise that aligns with these specific components. For example, if the element is the directive "right," we fabricate a direction-related noise. One irrelevant thought can be "Turning
right in countries that drive on the right side of the road typically does not intersect with oncoming traffic.". For inaccurate thoughts, to ensure that the noise does not contradict the previous reasoning, we select another related instruction on the same term domain. For example, when explaining the instruction "right," we might generate noise related to the instruction "left", such as "left means I\_TURN\_RIGHT.", which is incorrect and does not contradict the previous thought, and this thought itself is redundant. 

\textbf{NoRa-Commonsense.} When it comes to NoRa-Commonsense problems, we introduce noise that mirrors the relationships discerned during the reasoning process. For example, after establishing that "mother’s sister is aunt," we craft a noise sentence associated with the concept of "aunt." An irrelevant noise insertion could be: "Aunts often play pivotal roles in the social development of primates, akin to their influence in human societies." To generate inaccurate thoughts, we might produce a statement like "an aunt's mother is also an aunt," which is logically inaccurate.

In addition to inserting a fixed number of noisy thoughts per chain of thought shot, as detailed in the main text, we further introduce variability in noise addition by randomly inserting noise following each thought. This randomness follows a Bernoulli distribution $\text{Bern}(\epsilon) \in \{0,1\}$, where noise is added after a thought only if the Bernoulli trial results in 1. Consequently, while the fixed number approach guarantees a set amount of noise within a CoT shot, the random addition allows for the possibility of varying amounts of noise in each shot, dictated by the probability parameter $\epsilon$.

\subsection{NoRa Datasheet}
\label{appendix:datasheet}

Following the framework of \citet{gebru2021datasheets}, we provide a comprehensive datasheet for the NoRa dataset in Tab.~\ref{tab:datasheet}. This datasheet documents the motivation, composition, collection process, preprocessing details, intended uses, distribution plan, and maintenance strategy of our dataset.

{\scriptsize
\begin{longtable}{p{0.35\textwidth}|p{0.6\textwidth}}
\toprule
\multicolumn{2}{c}{\textsc{\textbf{Motivation}}} \\
\midrule
    \textbf{For what purpose was the dataset created?} 
    & 
    The NoRa dataset was created to evaluate the robustness of large language models when faced with noisy rationales in chain-of-thought prompting. It addresses the under-explored challenge of LLMs' reasoning ability when presented with in-context examples containing irrelevant or inaccurate steps.  
    \\ \midrule
    \textbf{Who created the dataset and on behalf of which entity?} & The authors of this paper created the dataset.\\ \midrule
    \textbf{Who funded the creation of the dataset?} & Unavailable. \\ \midrule
    \textbf{Any other comment?} & The NoRa dataset is the first of its kind specifically designed to evaluate the robustness of LLMs against noisy rationales. It covers a wide range of reasoning tasks, including mathematical, symbolic, and commonsense reasoning, The dataset includes varying levels of noise in rationales to systematically access LLMs' performance. \\ \midrule
    \multicolumn{2}{c}{\textsc{\textbf{Composition}}} \\ \midrule
    \textbf{What do the instances that comprise the dataset represent?} & Each instance in the NoRa dataset consists of a question, its correct answer, and in-context demonstrations containing noisy reasoning steps.  \\ \midrule
    \textbf{How many instances are there in total?} & The NoRa dataset contains 26,391 questions in total (see Section 3.2). Table 2 contains the sizes of our datasets. \\ \midrule
    \textbf{Does the dataset contain all possible instances or is it a sample (not necessarily random) of instances from a larger set?} & The NoRa dataset is a carefully curated sample selected from a larger pool of potential reasoning tasks. \\ \midrule
    \textbf{What data does each instance consist of?} & Each instance in the NoRa dataset contains a question, its correct answer, and in-context demonstrations with noisy reasoning steps.\\ \midrule
    \textbf{Is there a label or target associated with each instance?} & Yes, each instance has a correct answer that serves as its label. \\ \midrule
    \textbf{Is any information missing from individual instances?} & No, each instance contains complete information. \\ \midrule
    \textbf{Are relationships between individual instances made explicit?} & Yes, instances are grouped by reasoning type: mathematical, symbolic, and commonsense. \\ \midrule
    \textbf{Are there recommended data splits?} & No, this dataset is designed for training-free in-context learning. We provide recommended in-context demonstrations for each question. \\ \midrule
    \textbf{Are there any errors, sources of noise, or redundancies in the dataset?} & No, the dataset itself is clean and does not contain any errors, unintended noise, or redundancies. \\ \midrule
    \textbf{Is the dataset self-contained, or does it link to or otherwise rely on external resources?} & Yes, the dataset is self-contained and does not rely on any external resources. \\ \midrule
    \textbf{Does the dataset contain data that might be considered confidential?} & 
    No.
    \\ \midrule
    \textbf{Does the dataset contain data that, if viewed directly, might be offensive, insulting, threatening, or might otherwise cause anxiety?} & No, the dataset does not include any offensive, insulting, threatening, or anxiety-inducing content.
    \\ \midrule
    \multicolumn{2}{c}{\textsc{\textbf{Collection}}} \\ \midrule
    \textbf{How was the data associated with each instance acquired?} &  The data was created based on the reference of three rational datasets. (see sec 3.2)\\ \midrule
    \textbf{What mechanisms or procedures were used to collect the data?} & We systematically added irrelevant or inaccurate thoughts to clean rationales from existing datasets. The irrelevant or inaccurate thoughts are generated from GPT4. \\ \midrule
    \textbf{If the dataset is a sample from a larger set, what was the sampling strategy?} & Whenever we use subsets, we randomly sample from the original data. \\ \midrule
    \textbf{Who was involved in the data collection process and how were they compensated?} & The authors of this paper. \\ \midrule
    \textbf{Over what timeframe was the data collected?} &  NoRa was created in 2023. The original datasets of NoRa (Base Calculation, SCAN, and CLUTRR) were created between 2018-2023. We added noisy rationales to these existing datasets to create NoRa. \\ \midrule
    \textbf{Were any ethical review processes conducted?} & No, as the dataset does not involve human subjects or sensitive information, no formal ethical review was required. However, we adhered to ethical guidelines for AI research throughout the process \\ \midrule
    \multicolumn{2}{c}{\textsc{\textbf{Preprocessing}}} \\ \midrule
    \textbf{Was any preprocessing/cleaning/labeling of the data done?} & Yes. We preprocess the original datasets by adding noisy rationales (irrelevant or inaccurate thoughts) generated by GPT-4 and ensure that the questions and answers remain unchanged.  (detailed in Section 3.1)   \\ \midrule
    \textbf{Was the “raw” data saved in addition to the preprocessed/cleaned/labeled data?} & Yes. We retained the original datasets (Base Calculation, SCAN, and CLUTRR) in their raw form, as well as our modified NoRa dataset with added noisy rationales. \\ \midrule
    \textbf{ Is the software that was used to preprocess/clean/label the data available?} & No. \\ \midrule
    \multicolumn{2}{c}{\textsc{\textbf{Uses}}} \\ \midrule
    \textbf{Has the dataset been used for any tasks already?} & Yes, the NoRa dataset has been used in our paper to benchmark the reasoning robustness of large language models against noisy rationales in chain-of-thought prompting. \\ \midrule
    \textbf{Is there a repository that links to any or all papers or systems that use the dataset?} 
    &  
    Yes. 
    The repository is publicly available at:
    \url{https://github.com/tmlr-group/NoisyRationales}.
    \\ \midrule
    \textbf{What (other) tasks could the dataset be used for?} & The NoRa dataset could be used for:
1. Evaluating and improving LLMs' robustness to noisy information
2. Developing new denoising techniques for reasoning tasks
3. Understanding the impact of In-context noise on LLMs.\\ \midrule
    \textbf{Is there anything about the composition of the dataset or the way it was collected and preprocessed/cleaned/labeled that might impact future uses?} & The dataset contains artificially introduced noise, which might not perfectly reflect real-world scenarios of noisy reasoning. \\ \midrule
    \textbf{Are there tasks for which the dataset should not be used?} & The dataset should not be used for tasks that require completely noise-free data or for evaluating models' performance on clean, standard reasoning tasks without noise. \\ \midrule
    \multicolumn{2}{c}{\textsc{\textbf{Distribution}}} \\ \midrule
    \textbf{Will the dataset be distributed to third parties outside of the entity on behalf of which the dataset was created?} & Yes, we plan to make the dataset available to the research community. \\ \midrule
    \textbf{How will the dataset will be distributed?} & The dataset will be made available through a public repository (e.g., GitHub) along with necessary documentation for its use. \\ \midrule
    \textbf{When will the dataset be distributed?} & We aim to release the dataset upon publication of our paper. \\ \midrule
    \textbf{Will the dataset be distributed under a copyright or other intellectual property (IP) license, and/or under applicable terms of use (ToU)?} &  Not to our knowledge. \\ \midrule
    \textbf{Have any third parties imposed IP-based or other restrictions on the data associated with the instances?} & Not to our knowledge.  \\ \midrule
    \textbf{Do any export controls or other regulatory restrictions apply to the dataset or to individual instances?} & Not to our knowledge. \\ \midrule
    \multicolumn{2}{c}{\textsc{\textbf{Maintenance}}} \\ \midrule
    \textbf{Who will be supporting/hosting/maintaining the dataset?} & The authors of this paper. \\ \midrule
    \textbf{How can the owner/curator/manager of the dataset be contacted?} & The authors of this paper.\\ \midrule
    \textbf{Is there an erratum?} & No. \\ \midrule
    \textbf{Will the dataset be updated?} & No. \\ \midrule
    \textbf{If others want to extend/augment/build on/contribute to the dataset, is there a mechanism for them to do so?} & No. \\ 
    \bottomrule
    
\caption{Datasheet for the NoRa dataset, following the framework introduced by \citet{gebru2021datasheets}.
}
\label{tab:datasheet}
\end{longtable}}


\subsection{Full Examples of the NoRa Dataset}
\label{appendix: benchmark-examples}

In this part, we show the full examples of the NoRa dataset.

NoRa-Math:
\begin{itemize}[label=\textbullet]
\setlength\itemsep{-2pt}
\vspace{-8pt}
\item
Base-9 Irrelevant (Tab.~\ref{table:Irrelevant Base-9})
\item
Base-9 Inaccurate (Tab.~\ref{table:Inaccurate Base-9})
\item
Base-11 Irrelevant (Tab.~\ref{table:Irrelevant Base-11})
\item
Base-11 Inaccurate (Tab.~\ref{table:Inaccurate Base-11})
\end{itemize}

NoRa-Symbolic:
\begin{itemize}[label=\textbullet]
\setlength\itemsep{-2pt}
\vspace{-8pt}
\item
Equal Irrelevant  (Tab.~\ref{table:Irrelevant Symblic-Equal})
\item
Equal Inaccurate  (Tab.~\ref{table:Inaccurate Symblic-Equal})
\item
Longer Irrelevant  (Tab.~\ref{table:Irrelevant Symbolic-Longer})
\item
Longer Inaccurate  (Tab.~\ref{table:Inaccurate Symbolic-Longer})
\end{itemize}

NoRa-Commonsense:
\begin{itemize}[label=\textbullet]
\setlength\itemsep{-2pt}
\vspace{-8pt}
\item
Irrelevant (Tab.~\ref{table:Irrelevant Commonsense})
\item
Inaccurate (Tab.~\ref{table:Inaccurate Commonsense})

\end{itemize}

\begin{table*}[ht]
\centering
\fontsize{8}{8}\selectfont
\begin{tabular}{p{\textwidth}} 
\toprule[1.5pt] 
\midrule 
\textbf{Example:} 

User: In base-9, what is 64+62? Please reason it step by step. End the response with the result in "Answer:\textbackslash boxed\{result\}".

Assistant: In base-9, the digits are "012345678".  We have 4 + 2 = 6 in base-10.  \noisy{\uline{According to the Standard Model of physics, there are six types of quarks, the fundamental constituents of matter.}} Since we're in base-9, that doesn't exceed the maximum value of 8 for a single digit.  \noisy{\uline{In music, an octave consists of eight notes in the Western musical scale.}} 6 mod 9 = 6, so the digit is 6 and the carry is 0. We have 6 + 6 + 0 = 12 in base 10. 12 mod 9 = 3, so the digit is 3 and the carry is 1. A leading digit is 1. So the answer is 136. Answer:\textbackslash box\{136\}
\\
\midrule 
\bottomrule[1.5pt]
\caption{Example of NoRa-Math Base-9 with irrelevant rationales}
\label{table:Irrelevant Base-9}
\end{tabular}
\end{table*}


\begin{table*}[ht]
\centering
\fontsize{8}{8}\selectfont
\begin{tabular}{p{\textwidth}} 
\toprule[1.5pt] 
\midrule 

\textbf{Example:}

User: In base-9, what is 64+62? Please reason it step by step. End the response with the result in "Answer:\textbackslash box\{result\}".

Assistant: In base-9, the digits are "012345678". \noisy{\uline{9 + 8 = 17.}}  We have 4 + 2 = 6 in base-10. Since we're in base-9, that doesn't exceed the maximum value of 8 for a single digit. \noisy{\uline{8 + 3 = 11.}} 6 mod 9 = 6, so the digit is 6 and the carry is 0. We have 6 + 6 + 0 = 12 in base 10. 12 mod 9 = 3, so the digit is 3 and the carry is 1. A leading digit is 1. So the answer is 136. Answer:\textbackslash box\{136\}. 
\\
\midrule 
\bottomrule[1.5pt]
\caption{Example of NoRa-Math Base-9 with inaccurate rationales}
\label{table:Inaccurate Base-9}
\end{tabular}
\end{table*}

\begin{table*}[ht]
\centering
\fontsize{8}{8}\selectfont
\begin{tabular}{p{\textwidth}} 
\toprule[1.5pt] 
\midrule 

\textbf{Example:}

User: In base-11, what is 97+74? Please reason it step by step. End the response with the result in "Answer:\textbackslash boxed\{result\}".

Assistant: In base-11, the digits are "0123456789A".  We have 7 + 4 = 11 in base-10. \noisy{\uline{In mathematics, eleven is a prime number, meaning it has no factors other than 1 and itself. Since we're in base-11, that exceeds the maximum value of A for a single digit.}} 11 mod 11 = 0, so the digit is 0 and the carry is 1. \noisy{\uline{In game theory, zero represents the total benefit in a 'zero-sum game', where the gains and losses of every player add up to zero.}} We have 9 + 7 + 1 = 17 in base 10. 17 mod 11 = 6, so the digit is 6 and the carry is 1. A leading digit is 1. So the answer is 160. Answer:\textbackslash box\{160\}
\\
\midrule 
\bottomrule[1.5pt]
\caption{Example of NoRa-Math Base-11 with irrelevant rationales}
\label{table:Irrelevant Base-11}
\end{tabular}
\end{table*}

\begin{table*}[ht]
\centering
\fontsize{8}{8}\selectfont
\begin{tabular}{p{\textwidth}} 
\toprule[1.5pt] 
\midrule 

\textbf{Example:}

User: In base-11, what is 97+74? Please reason it step by step. End the response with the result in "Answer:\textbackslash boxed\{result\}".

Assistant: In base-11, the digits are "0123456789A".  We have 7 + 4 = 11 in base-10. Since we're in base-11, that exceeds the maximum value of A for a single digit. \noisy{\uline{10 + 5 = 15.}} 11 mod 11 = 0, so the digit is 0 and the carry is 1. We have 9 + 7 + 1 = 17 in base 10. 17 mod 11 = 6, so the digit is 6 and the carry is 1. A leading digit is 1. So the answer is 160. \noisy{\uline{0 + 11 = 11.}} Answer:\textbackslash box\{160\}. 
%
%
%
%
%
%
%
%
%
\\
\midrule 
\bottomrule[1.5pt]
\caption{Example of NoRa-Math Base-11 with inaccurate rationales}
\label{table:Inaccurate Base-11}
\end{tabular}
\end{table*}

\begin{table*}[ht]
\centering
\fontsize{9}{9}\selectfont 
\begin{tabular}{p{\textwidth}} 
\toprule[1.5pt] 
\midrule 
\textbf{Example:}

User: With IN:walk right twice and look right twice, what is OUT?\\Please reason it step by step, and provide the final action sequence as the answer. End the response with "So, final answer is OUT: <action sequence>"

Assistant: Since command is walk right twice and look right twice, we should consider Step1: "walk right twice" firstly, "walk right" means the agent needs to turn right and walk. The 'right' corresponds to the command I\_TURN\_RIGHT. \noisy{\uline{Turning right in countries that drive on the right side of the road typically does not intersect with oncoming traffic.}} Subsequently, 'walk' translates to I\_WALK. Therefore, the action sequence is I\_TURN\_RIGHT I\_WALK. Since we need do twice in command "walk right twice", this entire sequence is repeated 2 times, so the action sequence to "walk right twice" is :I\_TURN\_RIGHT I\_WALK I\_TURN\_RIGHT I\_WALK. \\Now, we consider Step2:"look right twice", "look right" means the agent needs to turn right and look. \noisy{\uline{Motorists in many areas are allowed to turn right at a red traffic signal after coming to a complete stop and ensuring the road is clear.}} The 'right' corresponds to the command I\_TURN\_RIGHT. \noisy{\uline{Certain traffic circles or roundabouts require all entering traffic to turn right regardless of the direction they intend to take.}} Subsequently, 'look' translates to I\_LOOK. Therefore, the action sequence is I\_TURN\_RIGHT I\_LOOK. Since we need do twice in command "look right twice", this entire sequence is repeated 2 times, so the action sequence to "look right twice" is :I\_TURN\_RIGHT I\_LOOK I\_TURN\_RIGHT I\_LOOK. \\Above all -- So, final answer is OUT:I\_TURN\_RIGHT I\_WALK I\_TURN\_RIGHT I\_WALK I\_TURN\_RIGHT I\_LOOK I\_TURN\_RIGHT I\_LOOK. 
\\
\midrule 
\bottomrule[1.5pt]
\caption{Example of NoRa-Symbolic Equal with irrelevant rationales}
\label{table:Irrelevant Symblic-Equal}
\end{tabular}
\end{table*}

\begin{table*}[ht]
\centering
\fontsize{9}{9}\selectfont 
\begin{tabular}{p{0.95\textwidth}} 
\toprule[1.5pt] 
\midrule 
\textbf{Example:}

User: With IN:walk right twice and look right twice, what is OUT?\\Please reason it step by step, and provide the final action sequence as the answer. End the response with "So, final answer is OUT: <action sequence>"

Assistant: Since command is walk right twice and look right twice, we should consider Step1: "walk right twice" firstly, "walk right" means the agent needs to turn right and walk. The 'right' corresponds to the command I\_TURN\_RIGHT. Subsequently, 'walk' translates to I\_WALK. \noisy{\uline{The term 'jump' corresponds to the command I\_RUN.}} Therefore, the action sequence is I\_TURN\_RIGHT I\_WALK. Since we need do twice in command "walk right twice", this entire sequence is repeated 2 times, so the action sequence to "walk right twice" is :I\_TURN\_RIGHT I\_WALK I\_TURN\_RIGHT I\_WALK. \noisy{\uline{Thrice means the entire sequence is repeated 6 times.}} \\Now, we consider Step2:"look right twice", "look right" means the agent needs to turn right and look. The 'right' corresponds to the command I\_TURN\_RIGHT. \noisy{\uline{The term 'left' translates to  I\_TURN\_RIGHT.}} Subsequently, 'look' translates to I\_LOOK. Therefore, the action sequence is I\_TURN\_RIGHT I\_LOOK. Since we need do twice in command "look right twice", this entire sequence is repeated 2 times, so the action sequence to "look right twice" is :I\_TURN\_RIGHT I\_LOOK I\_TURN\_RIGHT I\_LOOK. \\Above all -- So, final answer is OUT:I\_TURN\_RIGHT I\_WALK I\_TURN\_RIGHT I\_WALK I\_TURN\_RIGHT I\_LOOK I\_TURN\_RIGHT I\_LOOK. 
\\
\midrule 
\bottomrule[1.5pt]
\caption{Example of NoRa-Symbolic Equal with inaccurate rationales}
\label{table:Inaccurate Symblic-Equal}
\end{tabular}
\end{table*}

\begin{table*}[ht]
\centering
\fontsize{7}{7}\selectfont 
\begin{tabular}{p{\textwidth}} 
\toprule[1.5pt] 
\midrule 
\textbf{Example:} 

User: With IN:jump around left and jump around left, what is OUT?\\Please reason it step by step, and provide the final action sequence as the answer. End the response with "So, final answer is OUT: <action sequence>"

Assistant: Since command is jump around left and jump around left, we should consider Step1: "jump around left" firstly, "jump around left" means the agent needs to turn left and jump, and repeat this action sequence four times to complete a 360-degree loop. Therefore, the action sequence is I\_TURN\_LEFT I\_JUMP I\_TURN\_LEFT I\_JUMP I\_TURN\_LEFT I\_JUMP I\_TURN\_LEFT I\_JUMP. \\Now, we consider Step2:"jump around left", "jump around left" means the agent needs to turn left and jump, and repeat this action sequence four times to complete a 360-degree loop. Therefore, the action sequence is I\_TURN\_LEFT I\_JUMP I\_TURN\_LEFT I\_JUMP I\_TURN\_LEFT I\_JUMP I\_TURN\_LEFT I\_JUMP. \noisy{\uline{In cinematography, an action sequence is a part of the film where intense physical actions such as fights or chases take place.}} \\Above all -- So, final answer is OUT:I\_TURN\_LEFT I\_JUMP I\_TURN\_LEFT I\_JUMP I\_TURN\_LEFT I\_JUMP I\_TURN\_LEFT I\_JUMP I\_TURN\_LEFT I\_JUMP I\_TURN\_LEFT I\_JUMP I\_TURN\_LEFT I\_JUMP I\_TURN\_LEFT I\_JUMP. 
\\
\midrule 
\bottomrule[1.5pt]
\caption{Example of NoRa-Symbolic Longer with irrelevant rationales}
\label{table:Irrelevant Symbolic-Longer}
\end{tabular}
\end{table*}

\begin{table*}[ht]
\centering
\fontsize{7}{7}\selectfont 
\begin{tabular}{p{\textwidth}} 
\toprule[1.5pt] 
\midrule 

\textbf{Example:}

User: With IN:jump around left and jump around left, what is OUT?\\Please reason it step by step, and provide the final action sequence as the answer. End the response with "So, final answer is OUT: <action sequence>"

Assistant: Since command is jump around left and jump around left, we should consider Step1: "jump around left" firstly, "jump around left" means the agent needs to turn left and jump, and repeat this action sequence four times to complete a 360-degree loop. \noisy{\uline{The term 'opposite' implies a 180-degree turn, which requires I\_TURN\_RIGHT, I\_TURN\_LEFT.}} Therefore, the action sequence is I\_TURN\_LEFT I\_JUMP I\_TURN\_LEFT I\_JUMP I\_TURN\_LEFT I\_JUMP I\_TURN\_LEFT I\_JUMP. \\Now, we consider Step2:"jump around left", "jump around left" means the agent needs to turn left and jump, and repeat this action sequence four times to complete a 360-degree loop. Therefore, the action sequence is I\_TURN\_LEFT I\_JUMP I\_TURN\_LEFT I\_JUMP I\_TURN\_LEFT I\_JUMP I\_TURN\_LEFT I\_JUMP. \\Above all -- So, final answer is OUT:I\_TURN\_LEFT I\_JUMP I\_TURN\_LEFT I\_JUMP I\_TURN\_LEFT I\_JUMP I\_TURN\_LEFT I\_JUMP I\_TURN\_LEFT I\_JUMP I\_TURN\_LEFT I\_JUMP I\_TURN\_LEFT I\_JUMP I\_TURN\_LEFT I\_JUMP. 
\\
\midrule 
\bottomrule[1.5pt]
\caption{Example of NoRa-Symbolic Longer with inaccurate rationales}
\label{table:Inaccurate Symbolic-Longer}
\end{tabular}
\end{table*}

\begin{table*}[ht]
\centering
\fontsize{8}{8}\selectfont 
\begin{tabular}{p{\textwidth}} 
\toprule[1.5pt] 
\midrule 
\textbf{Example:} 

User: In a family tree, if Chris is Frances's husband's son's son.   Question: Chris is Frances's what? Please reason it step by step, and provide a single word answer describing the relationship. End the response in the format "Answer: relation "

Assistant: Chris is Frances's husband's son's son, so the relations path is husband, son, son. For son's son, we have son's son is grandson. \noisy{\uline{The parent-son relationship plays a crucial role in a son's emotional and social development.}} So the relations path are reduced to husband, grandson. \noisy{\uline{In genetics, a grandson shares approximately 25\% of his genes with each of his biological grandparents.}} For husband's grandson, we have husband's grandson is grandson. So the relations path are reduced to grandson. Therefore, Answer: grandson.
\\
\midrule 
\bottomrule[1.5pt]
\caption{Example of NoRa-Commonsense with irrelevant rationales}
\label{table:Irrelevant Commonsense}
\end{tabular}
\end{table*}

\begin{table*}[ht]
\centering
\fontsize{8}{8}\selectfont 
\begin{tabular}{p{\textwidth}} 
\toprule[1.5pt] 
\midrule 
\textbf{Example:} 

User: In a family tree, if Chris is Frances's husband's son's son.   Question: Chris is Frances's what? Please reason it step by step, and provide a single word answer describing the relationship. End the response in the format "Answer: relation "

Assistant: Chris is Frances's husband's son's son, so the relations path is husband, son, son. For son's son, we have son's son is grandson. So the relations path are reduced to husband, grandson. \noisy{\uline{We have grandson's daughter is brother.}} For husband's grandson, we have husband's grandson is grandson. So the relations path are reduced to grandson. \noisy{\uline{We have grandson's mother is grandfather}} Therefore, Answer: grandson. 
\\
\midrule 
\bottomrule[1.5pt]
\caption{Example of NoRa-Commonsense with inaccurate rationales}
\label{table:Inaccurate Commonsense}
\end{tabular}
\end{table*}

\clearpage
\section{Theoretical Analysis}
\label{app:theo}

In this part, we provide the theoretical results of ICL with the noisy demonstrations based on the \textit{distinguishability}~\citep{xie2021explanation}, which acts as a crucial condition in understanding ICL as the implicit Bayesian inference. We reveal that the noisy demonstrations can enlarge the 0-1 risk of the in-context predictor through the mismatch between the noisy prompt distribution and the clean pre-trained distributions. 

First, we briefly review some definitions and preliminary setups, as well as the distinguishability condition established in~\citep{xie2021explanation}, which is used to formalize when in-context learning occurs. 

\paragraph{Clean and noisy demonstrations.} As stated previously, the ICL techniques prompt the LLM with a few supporting examples $S_{n} \! = \! \{(x_{i}, y_{i})\}_{i=1}^{n}$ collected in the current context, each composed of a question $x_{i}$ and answer $y_{i}$, and then constructing the new input of $x_{\text{ICL}}$ as
\begin{align}
\vspace{-2pt}
x_{\text{ICL}} = \left[S_{n}, x_{\text{test}}\right] &= \left[x_{1}, y_{1}, \ldots, x_{n}, y_{n}, \; x_{\text{test}}\right]. 
\vspace{-2pt}
\nonumber
\end{align}
With $S_{n}$, the
$f_{\theta}(x_{\text{ICL}}) \! \mapsto \! y_{\text{test}}$
is easier than
$f_{\theta}(x_{\text{test}}) \! \mapsto \! y_{\text{test}}$.
Then, the CoT further refines $x_{\text{ICL}}$ by constructing the step-by-step \textit{rationale} $\mathcal{T}_i$, consisting several thoughts $T_{i}^{(j)}$, \textit{i.e.},
\begin{align}
\vspace{-2pt}
x_{\text{CoT}} = \left[x_{1}, \mathcal{T}_1, y_{1}, \ldots, x_{n}, \mathcal{T}_n, y_{n}, \; x_{\text{test}}\right], 
\nonumber \quad 
\text{where \;}
\mathcal{T}_i = \left[T_{i}^{(1)}, T_{i}^{(2)}, T_{i}^{(3)}, \ldots, T_{i}^{(k)}\right].
\vspace{-2pt}
\label{eqn: clean CoT_app}
\end{align}

However, as mentioned, the thoughts in CoT (Eqn.~\ref{eqn: clean CoT}) can be noisy in practice.
This noise can be attributed to (1) \textit{irrelevant thoughts}, which are irrelevant but correct, or (2) \textit{inaccurate thoughts}, which are relevant but factually wrong.
Here, we uniformly formalize these two kinds of noise as 
\begin{align} 
\vspace{-2pt}
{\hat{\mathcal{T}}_i}= \left[ T_{i}^{(1)}, {\hat{T}_{i}^{(1)}}, T_{i}^{(2)}, {\hat{T}_{i}^{(2)}}, 
\ldots, T_{i}^{(k)}, {\hat{T}_{i}^{(k)}} \right],
\vspace{-2pt}
\end{align}
where ${\hat{T}_{i}^{(j)}}$ represents a noisy thought (irrelevant or inaccurate)
that is coherent with previous clean thought $T_{i}^{(j)}$ (relevant and correct) in Eqn.~\ref{eqn: clean CoT_app}, and it make the $x_{\text{CoT}}$ be noisy.

\paragraph{In-context learning setups.} Keeping the same setups in~\citep{xie2021explanation}, we analyze the latent concept $\theta$ from a family of concepts $\Theta$ that defines a distribution over the observed tokens $o$ from a vocabulary $\mathcal{O}$. Following~\citep{xie2021explanation}, we assume that the pertaining document is a length $T$ sequence can be generated with the prior $p(\theta)$ as $p(o_1,\cdots,o_T)=\int_{\theta\in\Theta}p(o_1,\cdots,o_T|\theta)p(\theta)d\theta$, where $p(o_1,\cdots,o_T|\theta)$ is defined by the hidden markov model. We have the prompt distribution $p_{prompt}$ that generates the prompts for in-context learning, in which all the demonstrations share the prompt concept $\theta^*$.

We build the derivation by analyzing the in-context predictor $f_n(x_\text{test})=\arg\max_y p(y|S_n, x_\text{test})$, which indicates the most likely prediction over the pertaining distribution conditioned on the prompt from the prompt distribution, assuming that greedy decoding is adopted in practice. Specifically, we study the expected 0-1 error with $n$ examples $L_\text{0-1}(f_n)=\mathbb{E}[\mathbbm{1}[f_n(x_\text{CoT}\neq y_\text{test})]]$.

Given the output distribution $p_\theta^i(o):=p(\mathcal{O}[i]=o|\mathcal{O}[1:i-1],\theta)$ of the $i$-th token with previous tokens and the analogous distribution $p^{i}_{prompt}:=p_{prompt}(\mathcal{O}[i]=o|\mathcal{O}[1:i-1])$ under the prompt distribution, the distinguishability condition is dependent on a KL divergence between the previous two distributions and the error terms $\epsilon_\theta$ resulting from the distribution mismatch between the prompt and the pertaining distributions for each example. Letting $p_\theta^i(o)$ and $p^{i}_{prompt}$ correspond to the concepts $\theta$ and $\theta^*$, the KL divergence are defined as follows,
\begin{equation}
\text{KL}_i(\theta^*||\theta):=\mathbb{E}_{\mathcal{O}[1:i-1]\sim p_{prompt}}[\text{KL}(p^{i}_{prompt}||p_\theta^i)].
\end{equation}
Then, the distinguishability is provided as,
\begin{condition}[distinguishability~\citep{xie2021explanation}] $\theta^*$ is distinguishable if for all $\theta\in\Theta$, $\theta\neq\theta^*$,
    \begin{align}
        \sum_{i=1}^k \text{KL}_{i}(\theta^*||\theta)>\epsilon_\theta,
    \end{align}
\end{condition}
where $\epsilon_\theta$ is a small error term corresponding to the latent concept $\theta$. It has been empirically verified~\citep{wei2023larger,wei2022chain,xi2023self} that the larger example lengths $k$ can increase the left-hand side and improve distinguishability.
Intuitively, the larger example lengths increase the proportion of the prompt sampled from the pre-trained distribution, providing more evidence for Bayesian inference.
According to the noisy demonstration introduced in Eqn.~\eqref{eqn: irrelevant CoT}, we can naturally have the following condition on the distinguishability under the noisy scenario.
\begin{condition}[noisy distinguishability] With the irrelevant or inaccurate noise injection, the distribution mismatch between the prompt and the pre-trained distributions is also enlarged, resulting in a varied requirement with error term $\xi_\theta(\delta)$ (where $\delta$ indicates the noise ratio of the noisy rational in the whole demonstration) for $\theta^*$ being distinguishable if for all $\theta\in\Theta$, $\theta\neq\theta^*$,
    \begin{align}
        \sum_{i=1}^k \text{KL}_{i}(\theta^*||\theta)>\epsilon_\theta+\xi_\theta(\delta),\quad {\rm where}\quad \xi_\theta(\delta)\propto \delta.
    \end{align}
    \label{cond:2}
\end{condition}
Given the above Condition~\ref{cond:2}, we can analyze the expected 0-1 error when it fails according to the noisy scenarios considered in our work. It indicates that ICL cannot induce the correct output with the noise demonstration. First, we provide the following lemma based on the results of~\citet{devore1993constructive} and \citet{xie2021explanation}, which relate the bound on KL divergence with 0-1 loss given the Condition~\ref{cond:2}.
\begin{lemma}[noisy-relaxed bound in~\citep{xie2021explanation}]
\label{lemma:1}
    let $\mathcal{B}$ denotes the set of $\theta$ which does not satisfy Condition~\ref{cond:2}. We assume that $\text{KL}(p_{prompt}(y_\text{test}|x_\text{test}))||p(y_\text{test}|x_\text{test},\theta)$ is bounded for all $\theta$ and $\theta^*$ minimizes the multi-class logistic risk as,
    \begin{align}
\begin{split}
L_\text{CE}(\theta)=-\mathbb{E}_{x_\text{test}\sim p_{prompt}}[p_{prompt}(y_\text{test}|x_\text{test})\cdot\log p(y_\text{test}|x_\text{test},\theta)].
\end{split}
\end{align}
We can have if
\begin{align}
\mathbb{E}_{x_\text{test}\sim p_{prompt}}[\text{KL}(p_{prompt}(y_\text{test}|x_\text{test}, \delta)
|| p(y_\text{test}|x_\text{test},\theta))]\leq \epsilon_{\theta}+\xi_\theta(\delta),\quad \forall \quad \theta\in\mathcal{B},
\end{align}
then
\begin{align}
\lim_{n\rightarrow\infty} L_{0-1}(f_n) \leq \inf_{f} L_{0-1}(f) + g^{-1}\bigg(\sup_{\theta\in\mathcal{B}}(\epsilon_\theta+\xi_\theta(\delta))\bigg),
\end{align}
where $g(\tau) = \frac{1}{2}\big((1-\tau)\log(1-\tau)+(1+\tau)\log(1+\tau)\big)$ is the calibration function for the multiclass logistic loss for $\tau\in[0,1]$.
\end{lemma}
Based on the previous work~\citep{kleijn2012bernstein}, we assume that there is continuity between the concept parameter $\theta$ and its corresponding output distribution $p_{\theta}$, where the KL divergence is assumed to have the 2nd-order Taylor expansion. Then, we have the following theorem and proof.

\begin{theorem}\label{theo:1}
Given the previous assumption on multi-class logistic risk, we can have that 
$k \geq 2$ and $n \rightarrow \infty$, the risk of an in-context learning predictor $f_n$ is bounded as
\begin{align}
\begin{split}
\lim_{n\rightarrow\infty}L_{0-1}(f_n)\leq\inf_{f}L_{0-1}(f)+g^{-1}\bigg(O(\frac{\gamma_{\theta}\cdot\sup_{\theta\in\Theta}(\epsilon_{\theta}+\xi_\theta(\delta))}{k-1})\bigg),
\end{split}
\end{align}
where $g(\tau)=\frac{1}{2}\big((1-\tau)\log(1-\tau)+(1+\tau)\log(1+\tau)\big)$ is the calibration function for multi-class logistic loss for $\tau\in[0,1)$. 
The overall result indicates that the upper bound of the risk is loosed by $\xi_\theta(\delta)$ of the noisy demonstration.
\end{theorem}

\begin{proof}[Proof of Theorem~\ref{theo:1}]
Assume KL divergence have a 2nd-order Taylor expansion around $\theta^*$ as,
\begin{align}
    \forall i>1, \text{KL}_{i}(\theta^*||\theta) = \frac{1}{2}(\theta-\theta^*)^\mathrm{T} I_{i,\theta^*}(\theta-\theta^*)+O(||\theta-\theta^*||^3),
\end{align}
where $I_{i,\theta^*}$ is the Fisher information matrix of the $i$-th token distribution \textit{w.r.t.} $\theta^*$. Let $\gamma_{\theta^*}=\frac{\max_{i}\lambda_\text{max}(I_{i,\theta^*})}{\min_{i}\lambda_\text{min}(I_{i,\theta^*})}$ and $\lambda_\text{max}, \lambda_\text{min}$ is the largest and smallest eigenvalues. We can have the following  relationship for $\theta\in\mathcal{B}$ that,
\begin{align}
    \sum_{i=2}^k \text{KL}_{i}(\theta^*||\theta) &\geq \frac{1}{2}(\theta-\theta^*)^\mathrm{T}I_{i,\theta^*}(\theta-\theta^*)+O(||\theta-\theta^*||^3)\\
    &\geq \frac{1}{2}(k-1)\lambda_\text{min}(I_{i,\theta^*})||\theta-\theta^*||^2,\\
    ||\theta-\theta^*||^2 &\leq \frac{\epsilon_\theta+\xi_\theta(\delta)}{\frac{1}{2}(k-1)(\min_i \lambda_\text{min}(I_{i,\theta^*}))}.
\end{align}
Then we can bound the last KL term as,
\begin{align}
    \text{KL}_{k}(\theta^*||\theta) &= \frac{1}{2}(\theta-\theta^*)^\mathrm{T}I_{i,\theta^*}(\theta-\theta^*)+O(||\theta-\theta^*||^3)\\
    &\leq\frac{1}{2}(\max_{i}\lambda_\text{max}(I_{i,\theta^*})||\theta-\theta^*||^2+O||\theta-\theta||^2)\\
    &\leq\frac{((\epsilon_\theta+\xi_\theta(\delta))\max_{i}\lambda_\text{max}(I_{i,\theta^*})+O(1))}{(k-1)\min_i \lambda_\text{min}(I_{i,\theta^*})}.
\end{align}
Denoting $\text{KL}_{k}(\theta^*||\theta)$ with $\mathbb{E}_{x_\text{test}\sim p_{prompt}}[\text{KL}(p_{prompt}(y_\text{test}|x_\text{test}(\delta))
|| p(y_\text{test}|x_\text{test},\theta))]$, we have
\begin{align}
    \mathbb{E}_{x_\text{test}\sim p_{prompt}}[\text{KL}(p_{prompt}(y_\text{test}|x_\text{test}(\delta))
|| p(y_\text{test}|x_\text{test},\theta))]\leq \frac{((\epsilon_\theta+\xi_\theta(\delta))\max_{i}\lambda_\text{max}(I_{i,\theta^*})+O(1))}{(k-1)\min_i \lambda_\text{min}(I_{i,\theta^*})}.
\end{align}
Then, with the previous Lemma~\ref{lemma:1}, we can have the theorem result: for $k\geq2$ and $n\rightarrow\infty$, the risk of an ICL predictor $f_n$ is bounded as,
    \begin{align}\label{eq:18}
\begin{split}
\lim_{n\rightarrow\infty}L_{0-1}(f_n)\leq\inf_{f}L_{0-1}(f)+
g^{-1}\bigg(O(\frac{\gamma_{\theta}\cdot\sup_{\theta\in\Theta}(\epsilon_{\theta}+\xi_\theta(\delta))}{k-1})\bigg),
\end{split}
\end{align}
where $g(\tau)=\frac{1}{2}\big((1-\tau)\log(1-\tau)+(1+\tau)\log(1+\tau)\big)$ is the calibration function for multi-class logistic loss for $\tau\in[0,1)$, and $g^{-1}$ is roughly linear in $\epsilon$ as analyzed in~\citep{xie2021explanation}.

\end{proof}

With the above results based on \citep{xie2021explanation}, we can find that when the demonstration length is fixed, the multi-class logistic loss can have a loosened upper bound controlled by our noise injection. Intuitively, when the noise rate is large, the performance of ICL can be significantly affected, which can be empirically verified by the Easy-to-Hard performance trend evaluated in Tab.~\ref{tab:vanilia_LLM}. On the other hand, considering the fixed noisy injection in the demonstration, our method can reduce the effect of the mismatching error by rephrasing the demonstration with the pre-trained distribution of the LLM itself and then decreasing the error term in RHS of Eqn.~\eqref{eq:18}, which increase the distinguishability of ICL to induce the correct outputs successfully.

\clearpage
\section{Implementation Details}
\label{app:implementation}

In this section, we provide a comprehensive overview of the implementation details for our experiments, including the specifics of the baseline methods (Appendix~\ref{app:implementation-baselines}), the full algorithm of our proposed CD-CoT method (Appendix~\ref{app:implementation-CD-CoT}), and the hyper-parameter optimization process (Appendix~\ref{app: hyper-parameter pptimization}). 
For reproduction, the source code and full dialogs are provided at:
\url{https://github.com/tmlr-group/NoisyRationales}.

To conduct the full range of experiments on GPT-3.5-turbo-0613 mentioned in this paper, we utilized a total of \textbf{2.03B} tokens. 
This consisted of 
1.21B 
input tokens and 
0.82B
output tokens.

\subsection{Baseline Methods}
\label{app:implementation-baselines}

In this part, we introduce the implementation details of the baseline methods.

\selectfont\textbf{Self-correction Methods:}
\begin{itemize}
\vspace{-8pt}
\item 
\textbf{Intrinsic Self-correction (ISC)}~\cite{huang2023large} asks LLM to endeavor to rectify its initial responses based solely on its inherent capabilities. We employ the prompts from~\cite{huang2023large}, instructing LLMs to review and revise their answers to NoRa tasks with "Review your previous answer and find problems with your answer. Based on the problems you found, improve your answer. Please reiterate your answer."

\item 
\textbf{Self-polish (SP)}~\cite{xi2023self} teaches LLM to eliminate noisy information, rearrange the logic structure, and organize local conditions into new ones in parallel. We implement this method by (1) prompting LLMs to individually refine each noisy CoT examplar without additional information, repeating the process three times, and (2) combining these rephrased demos to form the context for the task.

\end{itemize}

\textbf{Self-consistency Methods:}

\begin{itemize}
\vspace{-8pt}
\item 
\textbf{SmoothLLM (SM)}~\cite{robey2023smoothllm} enhances robustness by injecting perturbations into the prompts and utilizing self-consistency to mitigate these effects.
We apply the described disturbance methods to noisy rationale demonstrations and feed them into LLMs for reasoning tasks. This process is repeated five times, with the most common answer across iterations selected as the voted answer.

\item 
\textbf{Self-denoise (SD)}~\cite{zhang2023certified} improves LLM robustness by preprocessing prompts with random masks; the LLMs then work to reconstruct the masked content, reducing noise and aiding incoherent reasoning. Our implementation involves (1) applying the masking method to the noisy rationales on each shot, (2) prompting the LLMs to infer and fill the mask sections of each demonstration, and (3) using the reconstructed CoT demonstrations for task reasoning. This process is also repeated five times, and the most common answer is selected.

\item 
\textbf{Self-consistency (SC)}~\cite{wang2023self} boosts reasoning performance by sampling multiple outputs and conducting majority voting without engaging in any input processing. To apply this method, we run the same task 5 times and vote for the maximum number of the same answers. 
\end{itemize}

\textbf{External Supervision Methods:}

\begin{itemize}
\vspace{-8pt}
\item 
\textbf{Self-correction with Oracle Feedback (SCO)}~\cite{huang2023large} enhances the ISC framework by providing a definitive ground truth. To implement this method, we use ISC to guide LLMs in refining their reasoning outcomes. The correction process is iterative, allowing for up to two adjustments, but ceases as soon as the correct answer is obtained.
 
\item 
\textbf{Backtracking (BT)}~\citep{tyen2023llms} is a method of providing the position where the first mistake occurs to prompt LLMs to correct mistakes. To implement this method, (1) we prompt LLMs to correct the noisy rationales in each CoT example by providing the first noise sentence in it, and (2) we combine these rephrased examples to form the in-context for the task reasoning.

\item 
\textbf{ContrastiveCoT (CC)}~\citep{chia2023contrastive} uses contrastive CoT demonstration to improve LLMs' generalization abilities. We implement this method by providing an extra example with clean rationales as a positive demonstration, and the original prompts with noisy rationales are treated as negative examples. We then prompt the task queries with these positive and negative examples using the contrastive structure prompt mentioned in the paper to guide the LLMs in the reasoning process.

\end{itemize}

\subsection{Contrastive Denoising with Noisy Chain-of-Thought}
\label{app:implementation-CD-CoT}

\textbf{Step-1: Rephrasing via Supervised Contrasting ($1$to$N$).}
First, we establish a general prompt to construct a pair of contrastive examples, as shown in Sec.~\ref{implementation}, steering the model towards learning from the clean example and then rephrasing as well as rectifying the noisy examples. 
It is worth noting that for each dataset, the clean demonstration utilized is fixed, aiming to minimize the usage of additional supervision. 
In the $K$-shot in-context learning, we formulate a rephrasing prompt for each noisy example in the context, following the template in Sec.~\ref{implementation}. The rephrasing time of each noisy example is controlled by the parameter $N$. In our implementation, we set $N=5$ to strike a balance between the diversity of rephrased results and computational costs.

\textbf{Step-2: Rationale Selection ($N$to$M$, $N \! \geq \! M$).}
Given that rephrased results may still contain noise, we employ answer matching for further denoising. Specifically, for each noisy example and its corresponding rephrased rationales, as each rationale can deduce an answer to the question, we select the rationales that the corresponding answers match the given (true) answer of this demonstration. This is called the “answer matching” and does not require an LLM for inference. The selected rationales form the refined candidate pool for that noisy example. 

For example, $(x, \mathcal{T}, y)$ indicates the question, rationale, and answer of a given noisy demonstration. The rationale rephrasing (step 1) obtains three rephrased demonstrations: $(x_{1}, \mathcal{T}_{1}, y_{1}), (x_{2}, \mathcal{T}_{2}, y_{2})$, and $(x_{3}, \mathcal{T}_{3}, y_{3})$. Then, if $y_{1}=y_{2}=y$ and $y_{3} \neq y$, we will select the first two rephrased rationales, $\mathcal{T}_{1}$ and $\mathcal{T}_{2}$. 
Namely, only rephrased results with consistent answers are retained, forming the refined candidate pool for that noisy demonstration of in-context learning. Notably, in cases where no consistent answer is found, the algorithm randomly selects one of the rephrased results as the candidate pool for the noisy example. 
Subsequently, we assemble $M$ contexts using the $K$ candidate pools. In particular, we randomly select one candidate from each pool and concatenate the $K$ selected candidates, forming a context. This process is repeated until $M$ contexts are obtained.


\textbf{Step-3: Rationale Exploration ($M$to$D$, $M \! \leq \! D$).}
For each of the $M$ distinct contexts, we perform multiple reasoning to explore various rationales, resulting in a total number of reasoning times equal to $D$. To maintain a fair comparison with other baseline methods and to introduce more randomness in generating diverse reasoning paths, we employ the default temperature parameter, \textit{i.e.}, 1. Notably, given the fixed values for $M$ and $D$, distinct allocation strategies of reasoning times can be generated, leading to performance and cost differences, as illustrated in Tab.~\ref{tab:different_instantiation_performance} and Tab.~\ref{tab:different_instantiation_efficiency}.



\textbf{Step-4: Answer Voting ($D$to1).}
Given $D$ answers from step 3, we equally vote them into a final answer. For example, if the answer set is ${1,1,1,2,3}$ where $D=5$, the answer voting will select “1” as the final answer for its highest frequency. 








\subsection{Hyper-parameter Optimization}
\label{app: hyper-parameter pptimization}

We would like to further clarify the selection of hyper-parameters.
To ensure a fair comparison with other baselines (for example, in self-consistency where we have the LLM reason 5 times and then vote), we fixed both $D$ and $N$ to be 5. Specifically, $D$ represents reasoning 5 times over the rephrased results, and $N$ indicates performing contrastive rephrasing on the input sample 5 times. This allows us to study the effects of reasoning with different rephrased inputs.

$M$ represents the number of selected rephrased results from the generated ones, therefore $M$ must be less than or equal to $N$, and the value of $M$ depends on the breakdown combination of $D$. 

When $D=5$, the possible combinations are $[(1,1,1,1,1), (1,1,1,2), (1,2,2), (3,2), (5)]$, with corresponding $M$s of $[5,4,3,2,1]$. For example, $(1,1,1,1,1)$ means inputting 5 different rephrased results and reasoning once for each, totaling 5 times of reasoning. It's worth mentioning that the greater the value of $M$, the larger the computational cost.

Finally, to balance the effectiveness and the cost of the algorithm, we test the results on a subset of the dataset with $M=1,2,3,5$. We find that the accuracy exhibits subtle variations when employing different algorithm instances, with the highest average accuracy observed at 51.3\% and the lowest average accuracy at 49.3\%. 
However, the difference in computational cost is relatively significant. Therefore, we ultimately select the algorithm instance with $M=2$ as the default setting and test it across the entire dataset. 

We present the impact of parameters $M, D,$ and $C$ on token usage in Tab.~\ref{tab:different_instantiation_efficiency}. Here, we conduct additional experiments to figure out the effect of varying $N$ on the computational cost. Specifically, we maintain a constant number of reasoning repetitions $D=5$ while adjusting $N$ and other parameters. Here are the configurations for testing:
\begin{enumerate}
    \item N=1, M=1, C=[5], D=5
    \item N=2, M=2, C=[3,2], D=5
    \item N=3, M=2, C=[3,2], D=5
    \item N=4, M=2, C=[3,2], D=5
    \item N=5, M=2, C=[3,2], D=5 (default)
\end{enumerate}

These experiments are conducted on the NoRa-Math ase-9 task with irrelevant hard noise. The Tab.~\ref{tab: ablation study of parameter N} shows the total number of tokens consumed by CD-CoT for complete reasoning on 300 test samples. This includes tokens used for both rephrasing and reasoning steps.

\begin{table*}[ht]
    \centering
    \fontsize{8}{8}\selectfont
    \begin{tabular}{c|c|c|c|c|c}
        \toprule
         & $N=1$ & $N=2$ & $N=3$ & $N=4$ & $N=5$\\
        \midrule
        tokens & 1071560 & 1408845 & 1532606 & 1656617 & 1780095\\
        \bottomrule
        
    \end{tabular}
    \caption{Ablation study of parameter $N$.}
    \label{tab: ablation study of parameter N}
\end{table*}

As we can observe, the number of tokens generated increases as $N$ increases. This growth in token count directly correlates with computational needs. However, it's important to note that the increase is not linear, suggesting that the computational cost does not scale linear proportionally with $N$.

\clearpage
\section{Full Experiments}
\label{app:exp}

In this section, we provide additional experimental details and analyses to supplement the main results presented in the paper. We conduct experiments on various LLMs and reasoning tasks to comprehensively evaluate the effectiveness of our proposed Contrastive Denoising Chain-of-Thought (CD-CoT) method in handling noisy rationales. Specifically,

\begin{itemize}
    
\item 
In Appendix \ref{app:exp-setup}, we describe the detailed experimental setups, including the base language models used, evaluation settings, and CoT settings. 

\item 
In Appendix \ref{app:exp-quantitative}, we present additional experimental results on different language models, and computational costs, and introduce the Normalized Difference in Accuracy (NDA) metric for evaluating the efficacy of methods under noisy scenarios. 

\item 
In Appendix \ref{appendix: ablation study on CD-CoT}, we conduct an ablation study to further justify the denoising power of the proposed CD-CoT method and highlight the importance of its two key components: contrastive denoising with rationale rephrasing and repeated reasoning with voting. 

\item 
In Appendix \ref{appendix: noise_semantic_difficulty}, we explore the impact of different semantic difficulty levels of irrelevant noise on model performance and the effectiveness of CD-CoT under these scenarios.

\item 
In Appendices \ref{appendix: the Number of Noisy Thoughts} and \ref{app: number of noisy examples}, we investigate how the number of noisy thoughts per example and the number of noisy examples affect the model's reasoning performance. 

\item 
In Appendix \ref{app:variant of CD-CoT}, we discuss the variants of CD-CoT, including a self-supervised variant that does not rely on manually crafted clean examples.
 
\item 
In Appendix \ref{appendix: Large-scale Real Problems}, we further explore CD-CoT's robustness by introducing new datasets. Moreover, we investigate the noisy rationale problem in large-scale real-world scenarios by evaluating the impact of noisy context on model performance in multi-turn conversational QA tasks using the MT-Bench~\cite{zheng2024judging} dataset.

\item 
In Appendix \ref{app:exp-qualitative}, we provide additional qualitative results, showcasing the denoised examples generated by different robust methods across various tasks under medium-level noise.


\end{itemize}

These additional experiments and analyses aim to provide a comprehensive understanding of the noisy rationale problem, the effectiveness of the proposed CD-CoT method, and the factors influencing model performance under noisy scenarios.

\subsection{Detailed Setups of the Experiments}
\label{app:exp-setup}

We employ GPT-3.5-turbo-0613~\citep{floridi2020gpt} as our base LLM~(denoted as Base) for the analyses presented in this study. In addition, we conduct evaluations on three supplementary models, including Gemini-Pro (Jan. 2024)~\citep{team2023gemini}, Llama2-70B~\citep{touvron2023llama}, and Mixtral-8x7B~\citep{jiang2024mixtral}.
While evaluating baseline methods on various, we consistently keep the temperature parameter $\tau$ and the top-p setting at their default value of 1, along with all other hyperparameters of models set to defaults.
We conducted experiments on the first 300 questions for each task and repeated reasoning 5 times for each question. 

We assume all CoT experiments with clean rationales or noisy rationales are conducted in a 3-shot setting unless specified otherwise.
Furthermore, all CoT examples are constructed by randomly drawing from all available questions, except for the NoRa-Symbolic Longer task, which has predefined demonstrating and testing scopes.

\subsection{Supplementary Results of the Main Experiments}
\label{app:exp-quantitative}


\textbf{Different LLMs.}
Fig.~\ref{fig:gpt3.5_nora_performance} displays the result of the GPT-3.5-Turbo model's evaluation on the NoRa Dataset. 
It corresponds to base model results in Tab.~\ref{tab:vanilia_LLM}.
We have also conducted comprehensive experiments on the Gemini model to evaluate various types of noise. Fig.~\ref{fig:gemini_nora_performance} shows the full performance evaluation of Gemini on the NoRa dataset.

\begin{figure}[h]
    \centering
    \includegraphics[width=\textwidth]{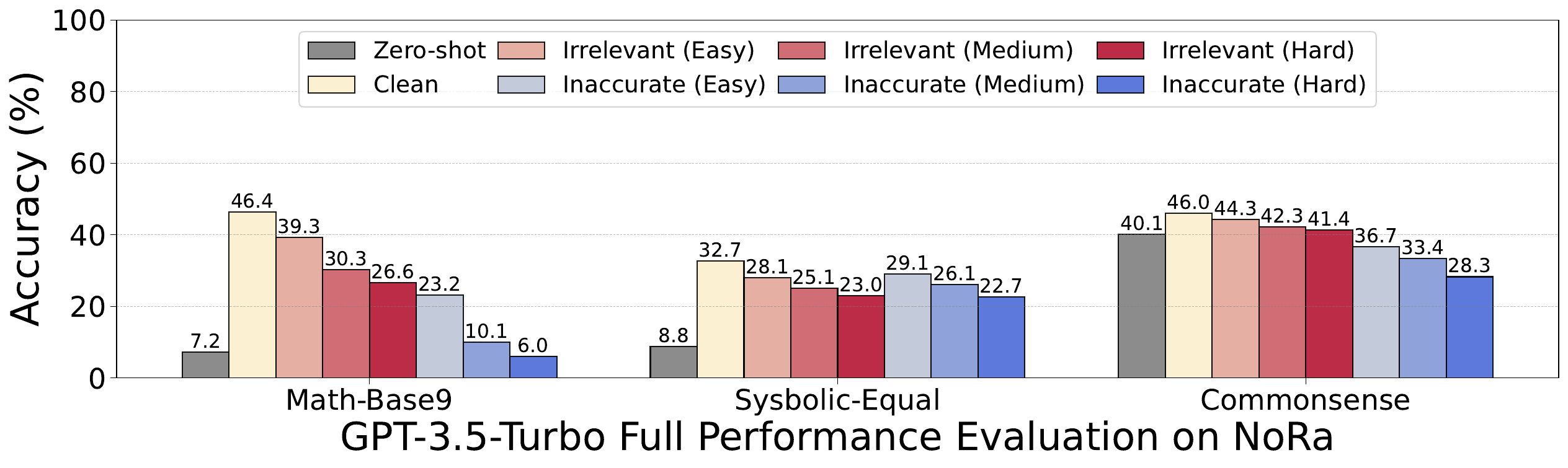}
    \vspace{-6px}
    \caption{
     \textbf{GPT-3.5-Turbo} Full Performance Evaluation on the NoRa Dataset.
    }
\label{fig:gpt3.5_nora_performance}
\vspace{-5pt}
\end{figure}

\begin{figure}[h]
    \centering
    \includegraphics[width=\textwidth]{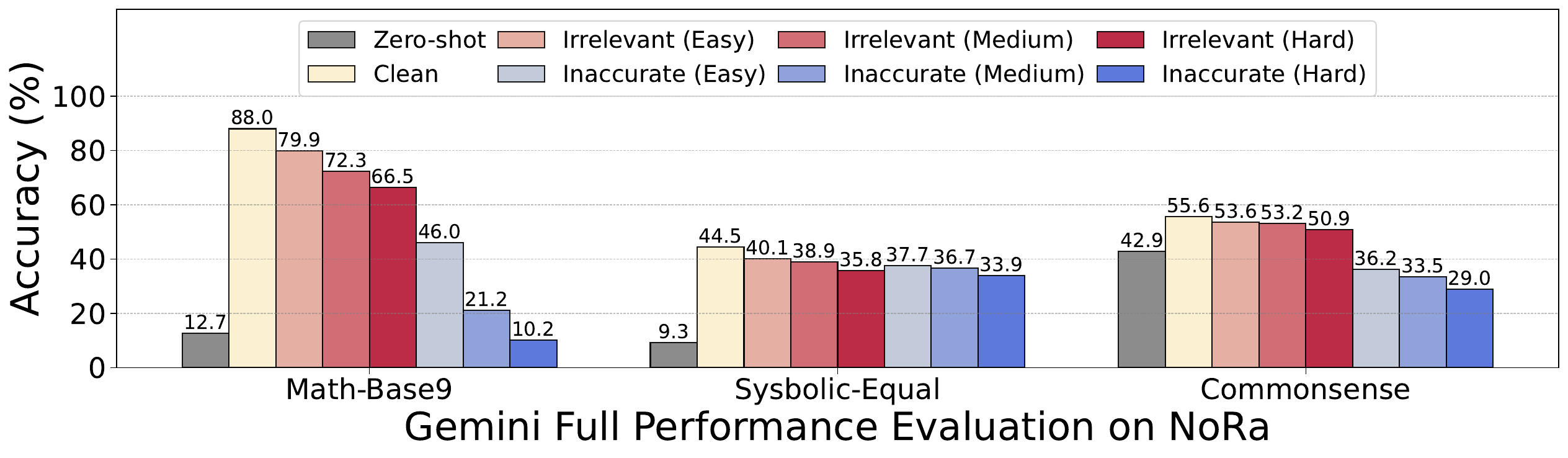}
    \vspace{-6px}
    \caption{
    \textbf{Gemini} Full Performance Evaluation on the NoRa Dataset.
    }
\label{fig:gemini_nora_performance}
\vspace{-5pt}
\end{figure}

\textbf{Computation Cost.}
Tab.~\ref{tab:cost_comparison} shows that the cost of CD-CoT is within an acceptable range compared with other methods.

\begin{table*}[ht]
\definecolor{mygray}{gray}{.9}
    \centering
    \fontsize{8}{10}\selectfont
    \setlength{\tabcolsep}{3pt}
    \renewcommand\arraystretch{0.9}
    \resizebox{\linewidth}{!}{\begin{tabular}{c|l|c|cccc|cccc}
        \toprule
        \multirow{2}{*}{Task} & \multicolumn{1}{c|}{\multirow{2}{*}{Method $\mathcal{M}$}} &\multirow{2}{*}{\makecell[c]{
        \#Tokens per \\
        clean sample
        }} &  \multicolumn{4}{c|}{
        \#Tokens per irrelevant sample 
        } &  \multicolumn{4}{c}{
        \#Tokens in inaccurate sample
        }\\
            &   &   &    Easy & Medium & Hard & Avg. & Easy & Medium & Hard & Avg. \\
        \midrule
        \multirow{10}{*}{\makecell[c]{Math \\ Base-9}}     
            &  \makecell[c]{Base}          & \gray 702.9  
            & \gray 858.2 & \gray 1027.9   & \gray 1209.5 & \gray 1031.9   
            & \gray 763.9 & \gray 826.5   & \gray 893.8 & \gray  828.1
            \\  
            &  w/ ISC ~\citep{huang2023large}        & 1375.9  
            & 1635.8   & 1922.6    & 2224.8   & 1927.7  
            & 1522.1   & 1633.2   & 1735.9   & 1630.4   
            \\  
            &  w/ SP  ~\cite{xi2023self}   & 1664.9  
            &  1739.6  & 1815.9    & 1900.1   & 1818.5   
            & 1742.4   & 1793.2   & 1824.4  & 1786.7   
            \\  
            &  w/ SM ~\citep{robey2023smoothllm}    & 3872.4  
            & 5029.2    & 5434.0    & 6197.6   & 5553.6   
            & 4226.8   &  4530.7   & 4831.1   & 4529.5  
            \\    
            &  w/ SD ~\citep{zhang2023certified}    & 5882.6   
            & 7365.2   & 9453.7  & 11123.8    & 9314.2  
            & 6033.3  & 10819.6    & 11625.9   & 9492.9 
            \\
            &  w/ SC ~\cite{wang2023self}  & 2002.2    
            & 2317.4    & 2783.7  & 3362.3   & 2821.1    
            & 2089.0    & 2228.4 & 2416.3    & 2244.6   
            \\ 
            &  w/ SCO ~\citep{huang2023large}        & 3772.1 
            & 4715.5   & 6122.8    & 7250.7   & 6029.7   
            & 5349.2   & 6536.2   & 6942.8   & 6276.1   
            \\ 
            &  w/ BT
            ~\citep{tyen2023llms}        & 701.2   
            & 1706.9  & 1859.8   & 1983.5  & 1850.1   
            & 1603.0   & 1680.1   & 1735.0  & 1672.7  
            \\ 
            &  w/ CC ~\citep{chia2023contrastive}        & 1047.7   
            & 1178.1   & 1315.5    & 1452.0   & 1315.2  
            & 966.2   & 1020.1   & 1075.0   & 1020.4  
            \\ 
            &  w/ CD-CoT (ours)        & \ours 5518.9   
            & \ours 5661.9   & \ours 5803.9    & \ours 5933.7   & \ours 5799.8   
            & \ours 5589.2   & \ours 5614.1   & \ours 5647.8   & \ours 5617.0 
            \\ 
        \midrule
        \multirow{10}{*}{\makecell[c]{Math \\ Base-11}}     
            &  \makecell[c]{Base}          & \gray 710.1   
            & \gray 877.3  & \gray 1057.9   & \gray 1226.0 & \gray 1053.7   
            & \gray 773.1  & \gray 834.4   & \gray 897.3 & \gray 834.9
            \\  
            &  w/ ISC ~\citep{huang2023large}        & 2498.2  
            & 2988.3   & 3531.7  & 4010.4   & 3510.1
            & 2687.9   & 2872.6    & 3026.4   & 2862.3      
            \\  
            &  w/ SP  ~\cite{xi2023self}   & 2837.6   
            & 2972.3   & 3136.5   & 3199.0   & 3102.6   
            & 2971.2   & 3055.2   & 3119.4   & 3048.6   
            \\  
            &  w/ SM ~\citep{robey2023smoothllm}    & 3945.6   
            & 5170.4   & 6406.1     & 7480.0   & 6352.2 
            & 4272.4  &  4567.2   & 4847.2  & 4562.3   
            \\    
            &  w/ SD ~\citep{zhang2023certified}    &  9578.2   
            & 11613.6   & 13848.5  & 15856.9    & 13773.0   
            & 10094.5  & 10842.7 & 11631.1   & 10856.1 
            \\
            &  w/ SC ~\cite{wang2023self}  & 2054.5    
            & 2385.7    & 2872.0  & 3405.4    & 2887.7    
            & 2161.0    & 2289.5    & 2433.8    & 2294.8   
            \\ 
            &  w/ SCO ~\citep{huang2023large}        & 4932.4  
            & 6073.5   & 7403.8    & 8485.9   & 7321.1  
            & 5557.4   & 6554.4   & 6986.8   & 6366.2   
            \\ 
            &  w/ BT
            ~\citep{tyen2023llms}        &    1603.0
            & 1706.9   & 1859.8   & 1983.5   & 1850.1   
            & 1603.0   & 1680.1   & 1735.0   & 1672.7  
            \\ 
            &  w/ CC ~\citep{chia2023contrastive}        & 913.8   
            & 1064.2   & 1211.2  & 1337.0   & 1204.1   
            & 1127.4 & 1188.1    & 1249.9   & 1188.5  
            \\ 
            &  w/ CD-CoT (ours)        & \ours5536.5   
            & \ours5681.3   & \ours5834.9     & \ours5968.1   & \ours5828.1   
            & \ours5647.2   & \ours5710.6     & \ours5772.9   & \ours5710.2   
            \\ 
        \midrule
        \multirow{10}{*}{\makecell[c]{Symbolic \\ Equal}}     
            &  \makecell[c]{Base}          & \gray  1623.7 
            & \gray 1834.9  & \gray 1991.0   & \gray 2221.4 & \gray 2015.8   
            & \gray 1765.3  & \gray 1868.8   & \gray 2022.4 & \gray 1885.5
            \\  
            &  w/ ISC ~\citep{huang2023large}        & 5246.0  
            & 5868.1   & 6345.2   & 7007.1   & 6406.8   
            & 5691.5   & 5998.2   & 6449.5   & 6046.4   
            \\  
            &  w/ SP  ~\cite{xi2023self}   & 5343.0   
            & 5539.8   & 5661.8   & 5858.0   & 5686.5   
            & 5526.3   & 5635.7   & 5800.3   & 5654.1   
            \\  
            &  w/ SM ~\citep{robey2023smoothllm}    & 10272.9   
            & 11936.8   & 13156.2    & 14860.2   & 13317.7   
            & 11315.3   & 12029.4   & 13093.6   & 12146.1   
            \\    
            &  w/ SD ~\citep{zhang2023certified}    & 6267.1    
            & 6782.5   & 6965.8  & 7004.4    & 6917.6   
            & 6763.1    & 7006.6    & 7135.4   & 6968.4 
            \\
            &  w/ SC ~\cite{wang2023self}  & 4089.3    
            & 4429.2 & 4724.0  & 5267.5    & 4806.9   
            & 4262.6 & 4404.0    & 4691.1    & 4452.6   
            \\ 
            &  w/ SCO ~\citep{huang2023large}        & 8937.4   
            & 10677.4   & 11253.8    & 12599.2   & 11510.1   
            & 10307.6   & 10970.7   & 11443.9   & 10907.4  
            \\ 
            &  w/ BT
            ~\citep{tyen2023llms}        & 1614.4  
            & 3058.9  & 3154.7  & 3320.2  & 3177.9   
            & 3011.1   & 3097.6   & 3203.4   & 3104.0  
            \\ 
            &  w/ CC ~\citep{chia2023contrastive}        & 2394.7  
            & 2592.5   & 2734.3  & 2933.8  & 2753.5  
            & 2529.7   & 2625.9   & 2761.8   & 2639.1   
            \\ 
            &  w/ CD-CoT (ours)        & \ours13602.0  
            & \ours13686.6   & \ours13609.4     & \ours13574.0   & \ours13623.3   
            & \ours13751.0   & \ours12818.5   & \ours13741.6   & \ours13437.0   
            \\ 
         \midrule
        \multirow{10}{*}{\makecell[c]{Symbolic \\ Longer}}     
            &  \makecell[c]{Base}          & \gray 1687.1   
            & \gray 1826.1   & \gray 1862.4   & \gray 2017.9  & \gray 1902.1  
            & \gray 1802.9   & \gray 1832.4   & \gray 1952.3 & \gray 1862.5
            \\  
            &  w/ ISC ~\citep{huang2023large}        & 5601.1  
            & 5957.7   & 6052.7    & 6466.0   & 6158.8   
            & 5932.0   & 6073.3   & 6382.8   & 6129.4   
            \\  
            &  w/ SP  ~\cite{xi2023self}   & 5687.9  
            & 5765.3   & 5815.2   & 5893.1   & 5824.5  
            & 5907.2   & 5911.1   & 6051.5   & 5956.6   
            \\  
            &  w/ SM ~\citep{robey2023smoothllm}    & 10487.3        
            & 11492.5   & 11838.6   & 12922.2      & 12084.4    
            & 11277.9   & 11566.7   & 12346.6   & 11730.4  
            \\    
            &  w/ SD ~\citep{zhang2023certified}    & 3087.1    
            & 3194.6  & 3277.5  & 3351.7    & 3274.6   
            & 3214.7  & 3416.0    & 3417.6   & 3349.4 
            \\
            &  w/ SC ~\cite{wang2023self}  & 4934.9
            & 5153.9    & 5175.4   & 5441.8    & 5257.0    
            & 5146.2    & 5146.8    & 5524.5    & 5272.5   
            \\ 
            &  w/ SCO ~\citep{huang2023large}        & 11888.2   
            & 13232.5   & 13765.5   & 14389.6   & 13795.9   
            & 13255.4   & 12982.1   & 10293.9   & 12177.1  
            \\ 
            &  w/ BT
            ~\citep{tyen2023llms}        & 1690.1   
            & 3066.5   & 3091.0   & 3187.7   & 3115.1  
            & 3095.8   & 3154.1   & 3253.5   & 3167.8   
            \\ 
            &  w/ CC ~\citep{chia2023contrastive}        & 2491.5 
            & 2618.7   & 2668.3  & 2801.0   &2696.0   
            & 2598.5   & 2626.9   & 2737.5   & 2654.3   
            \\ 
            &  w/ CD-CoT (ours)        & \ours14282.3  
            & \ours14270.6   & \ours14296.3   &  \ours14286.5   & \ours14284.5   
            & \ours14798.8   & \ours14974.2   & \ours15393.0   & \ours15055.3   
            \\ 
         \midrule 
        \multirow{10}{*}{Commonsense}
            &  \makecell[c]{Base}          & \gray  553.2
            & \gray 712.0   & \gray 789.5   & \gray 867.6 & \gray 789.7   
            & \gray 605.9   & \gray 635.6   & \gray 669.9 & \gray 637.1
            \\  
            &  w/ ISC ~\citep{huang2023large}        & 1873.2  
            & 2334.8   & 2559.2   & 2786.7   & 2560.2  
            & 2021.0   & 2110.4   & 2207.5   & 2113.0   
            \\  
            &  w/ SP  ~\cite{xi2023self}   & 4728.7   
            & 5502.0  & 5907.6  & 6308.6  & 5906.1   
            & 4998.5   & 5131.9  & 5265.1   & 5131.8   
            \\  
            &  w/ SM ~\citep{robey2023smoothllm}    & 3532.2   
            & 4755.0   & 5377.2  & 5999.0   & 5377.1   
            & 3943.7   & 4153.7  & 4352.8   & 4150.1  
            \\    
            &  w/ SD ~\citep{zhang2023certified}    & 5007.8   
            & 5882.5  & 6357.3  & 6816.1    & 6352.0   
            & 5285.0   & 5430.4    & 5561.9   & 5425.8
            \\
            &  w/ SC ~\cite{wang2023self}  & 853.5   
            & 1125.7    & 1252.5  & 1382.0    & 1253.4    
            & 900.2  & 941.2    & 1007.3    & 949.6   
            \\ 
            &  w/ SCO ~\citep{huang2023large}        & 1776.8   
            & 2382.5  & 2688.8  & 2827.4  & 2632.9   
            & 2243.5 & 2372.9   & 2581.2   & 2399.2   
            \\ 
            &  w/ BT
            ~\citep{tyen2023llms}        & 553.7   
            & 1089.7   & 1149.5   & 1213.7   & 1151.0   
            & 984.9   & 1006.2   & 1031.1   & 1007.4   
            \\ 
            &  w/ CC ~\citep{chia2023contrastive}        & 721.8 
            & 771.0   & 796.2   & 821.8   & 796.3 
            & 854.2   & 925.6   & 993.6   & 924.5   
            \\ 
            &  w/ CD-CoT (ours)        & \ours3060.7   
            & \ours3069.9   & \ours3055.1   & \ours3098.9   & \ours3074.6  
            & \ours3129.5   & \ours3093.2   & \ours3090.0   & \ours3104.2   
            \\ 
        \bottomrule
    \end{tabular}}
    \caption{Computation cost (\#tokens) of all methods.}
    \label{tab:cost_comparison}
\vspace{-2pt}
\end{table*}

\textbf{The Normalized Difference in Accuracy (NDA) Metric.}
We propose a new evaluation score, Normalized Difference in Accuracy (NDA), to quantify the efficacy of $\mathcal{M}$ under the noisy scenario, 

\begin{equation}
    \operatorname{NDA}(\mathcal{M}, \mathcal{Q}, \mathcal{P}) \! = \!
    \frac{\operatorname{Acc}(\mathcal{M}, \mathcal{Q},  \mathcal{P}_{\text{noisy}}) 
    \! - \! \operatorname{Acc}(\mathcal{M}, \mathcal{Q},  \emptyset)}
    {\operatorname{Acc}(\mathcal{M},  \mathcal{Q}, \mathcal{P}_\text{clean}) \! - \! \operatorname{Acc}(\mathcal{M} ; , \mathcal{Q},  \emptyset)},
\end{equation}
where $\operatorname{Acc}(\mathcal{M},  \mathcal{Q}, \mathcal{P}_\text{clean})$, $\operatorname{Acc}(\mathcal{M}, \mathcal{Q},  \mathcal{P}_{\text{noisy}})$ and $\operatorname{Acc}(\mathcal{M} ; , \mathcal{Q},  \emptyset)$ represent the accuracy of method $\mathcal{M}$ with clean rationales, noisy rationales (irrelevant or inaccurate), and without CoT demos.
Fig.~\ref{fig:NDA} is the illustration of the NDA metric. 

Please note that: (1) The main evaluation metric used throughout this paper is the accuracy score as introduced in Section 4; (2) The NDA metric serves as an auxiliary tool for analyzing empirical results, specifically designed to quantify how effectively a given LLM and denoising method perform under noisy scenarios.


Tab.~\ref{tab:NDA} presents a comparison of the accuracy and NDA across all methods. A negative value in NDA indicates that the accuracy of noisy rationales falls below that of 0-shot. We observe that CD-CoT consistently excels in NDA, nearing 100\% in most tasks and even surpassing it in certain instances.

\begin{figure}[ht]
    \centering
    \includegraphics[width=6.0cm]{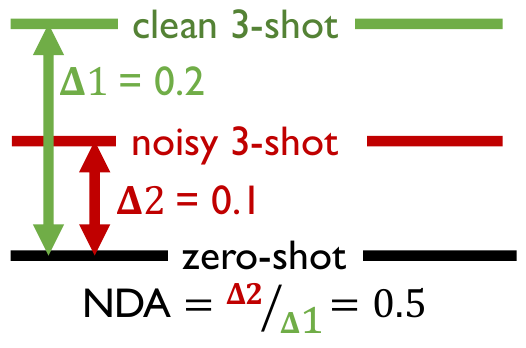}
    \vspace{-6px}
    \caption{
    Illustration of the NDA metric
    }
\label{fig:NDA}
\vspace{-5pt}
\end{figure}

\begin{table*}[ht]
    \centering
    \fontsize{7}{7}\selectfont
    \setlength{\tabcolsep}{1.8pt}
    \resizebox{\linewidth}{!}{
    \begin{tabular}{c|c|c|c|cccc|cccc}
        \toprule
        \multirow{2}{*}{Dataset} & \multicolumn{1}{c|}{\multirow{2}{*}{Method $\mathcal{M}$}} & \multirow{2}{*}{\makecell[c]{$\operatorname{Acc}(\mathcal{M}, \mathcal{Q},  \emptyset)$}} &\multirow{2}{*}{\makecell[c]{$\operatorname{Acc}(\mathcal{M}, \mathcal{Q}, \mathcal{P}_{\text{clean}})$}} &  \multicolumn{4}{c|}{$\operatorname{Acc}(\mathcal{M}, \mathcal{Q}, \mathcal{P}_{\text{Irrelevant}})$ (NDA)} &  \multicolumn{4}{c}{$\operatorname{Acc}(\mathcal{M}, \mathcal{Q}, \mathcal{P}_{\text{Inaccurate}})$ (NDA)}\\
            &   &  &  &    Easy & Medium & Hard & Avg. & Easy & Medium & Hard & Avg. \\
        \midrule
        \multirow{10}{*}{\makecell[c]{Math \\ Base-9}}     
            &  \makecell[c]{Base}                           
            & \gray 7.2  & \gray46.4           & \gray39.3 (81.9)           &  \gray30.3 (58.9)       & \gray26.6 (49.5)& \gray32.1 (63.5)       & \gray23.2 (40.8)     & \gray10.1 (7.4)      & \gray6.0 (-3.0)         & \gray13.1  (15.1)   
            \\
            &  w/ ISC                
            & 8.7 & 24.3 & 17.7 (57.7) & 14.7 (38.5) & 12.7 (25.6) & 15.0 (40.4) & 18.4 (62.2) & 13.7 (32.1) & 12.3 (23.1) & 14.8 (39.1) 
            \\
            &  w/ SP                     
            & 7.2 & 26.2 & 25.5 (\underline{96.3}) & 25.5 (\underline{96.3}) & 21.9 (77.4) & 24.3 (90.0) & 20.0 (67.4) & 18.4 (58.9) & 14.3 (37.4) & 17.6 (54.7)  
            \\ 
            &  w/ SM           
            & 12.3 & 37.4 &  30.0 (70.5) & 22.7 (41.4) & 16.5 (16.7) & 23.1 (43.0) & 24.7 (49.4) & 19.2 (27.5) & 12.4 (0.4) & 18.8 (25.9)   
            \\    
            &  w/ SD           
           & 12.3 & \underline{47.9} & 37.2 (69.9) & 25.4 (36.8) & 24.7 (34.8) & 29.1 (47.1) & 29.3 (47.8) & 12.5 (0.6) & 8.7 (-10.1) & 16.8 (12.6) 
            \\
            &  w/ SC                    
            & 12.3 & \textbf{61.5} & \underline{51.1} (78.9) & 39.0 (48.6) & 36.2 (54.3) & 42.1 (60.6) & 32.7 (41.5) & 15.3 (6.1) & 7.5 (-9.8) & 18.5 (12.6)   
            \\ 
            &  w/ SCO
            & 12.5 & 53.6 & 46.3 (82.2) & 39.6 (65.9) & 36.4 (58.2) & 40.8 (68.9) & 34.7 (54.0) & 22.0 (23.1) & 17.7 (12.7) & 24.8 (29.9)  
            \\
            &  w/ BT
            & 7.2 & 47.2 & 39.2 (80.0) & 34.2 (67.5) & 29.9 (56.8) & 34.4 (68.0) & 30.1 (57.2) & 18.4 (28.0) & 14.1 (17.2) & 20.9 (34.2) 
            \\
            &  w/ CC   
            & 7.2 & 44.9 & 43.3 (95.8) & \underline{44.6} (\underline{99.2}) & \underline{45.5} (\textbf{101.6}) & \underline{44.5} (\textbf{98.9}) & \underline{37.2} (\underline{79.4}) & \underline{31.7} (\underline{64.8}) & \underline{30.7} (\underline{62.2}) & \underline{33.2} (\underline{68.8}) 
            \\
            &  w/ CD-CoT         
            & \ours 12.3 & \ours 60.7 & \ours \textbf{59.7} (\textbf{97.9}) & \ours \textbf{60.7} (\textbf{100.0}) & \ours \textbf{57.2} (\underline{92.8}) & \ours \textbf{59.2} (\underline{96.9}) & \ours \textbf{54.0} (\textbf{86.2}) & \ours \textbf{58.7} (\textbf{95.9}) & \ours \textbf{48.4} (\textbf{74.6}) & \ours \textbf{53.7} (\textbf{85.5})   \\
        \midrule
        \multirow{10}{*}{\makecell[c]{Math \\ Base-11}}     
            &  \makecell[c]{Base}                           
            & \gray 5.5 & \gray 23.9 & \gray 19.1 (73.9) & \gray 13.6 (44.0) & \gray 10.7 (28.3) & \gray 14.5 (48.9) & \gray 14.0 (46.2) & \gray 6.7 (6.5) & \gray 3.6 (-10.3) & \gray 8.1 (14.1) 
            \\
            &  w/ ISC 
            & 7.4 & 11.2 & 8.3 (23.7) & 7.8 (10.5) & 6.0 (-36.8) & 7.4 (0.0) & 6.5 (-23.7) & 5.2 (-57.9) & 4.7 (-71.1) & 5.5 (-50.0)  
            \\
            &  w/ SP
            & 5.5 & 20.7 & 17.5 (78.9) & 16.7 (73.7) & 14.0 (55.9) & 16.0 (69.1) & 14.1 (56.6) & 10.7 (34.2) & 10.8 (34.9) & 11.9 (42.1) 
            \\    
            &  w/ SM             
            & 8.0 & 16.3 & 12.0 (48.2) & 6.0 (-24.1) & 5.7 (-27.7) & 7.9 (-1.2) & 12.0 (48.2) & 9.3 (15.7) & 7.7 (-3.6) & 9.7 (20.5)  
            \\      
            &  w/ SD   
            & 8.0 & 17.9 & 12.3 (43.4) & 12.0 (40.4) & 13.3 (53.5) & 12.5 (45.5) & 17.0 (\underline{90.9}) & 8.7 (7.1) & 5.3 (-27.3) & 10.3 (23.2)   
            \\ 
            &  w/ SC
            & 8.0 & \textbf{33.7} & 25.3 (67.3) & 16.3 (32.3) & 15.0 (27.2) & 18.9 (42.4) & 19.7 (45.5) & 9.3 (5.1) & 3.3 (-18.3) & 10.8 (10.9)  
            \\
            &  w/ SCO 
            & 22.9 & \underline{33.0} & \underline{29.2} (62.4) & \underline{24.0} (10.9) & \underline{20.0} (-28.7) & \underline{24.4} (14.9) & \textbf{29.2} (62.4) & \underline{20.0} (-28.7) & \underline{17.2} (-56.4) & \underline{22.1} (-7.9) 
            \\
            &  w/ BT
            & 5.5 & 24.3 & 17.9 (66.0) & 17.2 (62.2) & 13.7 (43.6) & 16.3 (57.4) & 12.8 (38.8) & 9.2 (19.7) & 6.8 (6.9) & 9.6 (21.8) 
            \\
            & w/ CC
            & 5.5 & 22.3 & 19.1 (\underline{81.0}) & 18.4 (\underline{76.8}) & 18.2 (\underline{75.6}) & 18.6 (\underline{78.0}) & 19.0 (80.4) & 15.3 (\underline{58.3}) & 14.6 (\underline{54.2}) & 16.3 (\underline{64.3}) 
            \\
            & w/ CD-CoT   
            & \ours 8.0 & \ours 31.0 & \ours \textbf{33.7} (\textbf{111.7}) & \ours \textbf{32.7} (\textbf{107.4}) & \ours \textbf{34.7} (\textbf{116.1}) & \ours \textbf{33.7} (\textbf{111.7}) & \ours \underline{29.0} (\textbf{91.3}) & \ours \textbf{30.7} (\textbf{98.7}) & \ours \textbf{25.3} (\textbf{75.2}) & \ours \textbf{28.3} (\textbf{88.3}) 
            \\
        \midrule
        \multirow{10}{*}{\makecell[c]{Symbolic \\ Equal}}     
            &  \makecell[c]{Base}                            
            & \gray 8.8 & \gray 32.7 & \gray 28.1 (80.8) & \gray 25.1 (68.2) & \gray 23.0 (59.4) & \gray 25.4 (69.5) & \gray 29.1 (84.9) & \gray 26.1 (72.4) & \gray 22.7 (58.2) & \gray 26.0 (72.0) 
            \\
            &  w/ ISC                
            & 5.7 & 23.9 & 20.0 (78.6) & 16.3 (58.2) & 15.5 (53.8) & 17.3 (63.7) & 19.2 (74.2) & 18.3 (69.2) & 18.1 (68.1) & 18.5 (70.3) 
            \\ 
            &  w/ SP                      
            & 8.8 & 23.2 & 23.0 (\underline{98.6}) & 22.6 (\underline{95.8}) & 22.7 (\underline{96.5}) & 22.8 (\underline{97.2}) & 23.7 (\textbf{103.5}) & 22.5 (\underline{95.1}) & 23.5 (\textbf{102.1}) & 23.2 (\textbf{100.0}) 
            \\    
            &  w/ SM             
            & 9.7 & 25.0 & 20.7 (71.9) & 19.7 (65.4) & 16.7 (45.8) & 19.0 (60.8) & 21.0 (73.9) & 20.3 (69.3) & 20.0 (67.3) & 20.4 (69.9)  
            \\    
            &  w/ SD             
            &9.7 & 9.9  & 10.1 (---)   & 10.9 (---)  & 10.3 (---) & 10.4 (---) &  10.1 (---) & 10.9 (---) & 10.4 (---) & 10.5 (---) 
            \\
            &  w/ SC          
            & 9.7 & 35.3 & 31.0 (83.2) & 28.3 (72.7) & 27.0 (67.6) & 28.8 (74.6) & 33.3 (92.2) & 30.7 (82.0) & 26.0 (63.7) & 30.0 (79.3)  
            \\ 
            & w/ SCO
            & 12.5 & 38.5 & \underline{34.9} (86.2) & \underline{33.4} (80.4) & \underline{32.7} (77.7) & \underline{33.7} (81.5) & \underline{34.0} (82.7) & \underline{34.1} (83.1) & \underline{34.5} (84.6) & \underline{34.2} (83.5) 
            \\
            & w/ BT    
            & 8.8 & 31.8 & 26.0 (74.8) & 22.7 (60.4) & 22.6 (60.0) & 23.8 (65.2) & 26.3 (76.1) & 22.7 (60.4) & 22.9 (61.3) & 24.0 (66.1)
            \\
            & w/ CC 
            & 8.8 & \underline{37.8} & 33.8 (86.2) & 32.7 (82.4) & 32.0 (80.0) & 32.8 (82.8) & 31.3 (77.6) & 33.0 (83.4) & 29.9 (72.8) & 31.4 (77.9) 
            \\
            & w/ CD-CoT 
            & \ours 9.7 & \ours \textbf{42.7} & \ours \textbf{44.7} (\textbf{106.1}) & \ours \textbf{42.7} (\textbf{100.0}) & \ours \textbf{44.0} (\textbf{103.9}) & \ours \textbf{43.8} (\textbf{103.3}) & \ours \textbf{42.6} (\underline{99.7}) & \ours \textbf{41.3} (\textbf{95.8}) & \ours \textbf{42.7} (\underline{100.0}) & \ours \textbf{42.2} (\underline{98.5})     
            \\ 
         \midrule
        \multirow{10}{*}{\makecell[c]{Symbolic \\ Longer}}     
            &  \makecell[c]{Base}                             
            & \gray 0.0 & \gray 9.2 & \gray 6.3 (68.5) & \gray 7.2 (78.3) & \gray 6.0 (65.2) & \gray 6.5 (70.7) & \gray 7.0 (76.1) & \gray 6.8 (73.9) & \gray 6.0 (65.2) & \gray 6.6 (71.7) 
            \\
            &  w/ ISC                 
            & 0.1 & 4.9 & 4.6 (93.7) & 2.7 (54.2) & 3.7 (75.0) & 3.7 (75.0) & 3.4 (68.7) & 4.3 (\textbf{87.5}) & 3.3 (66.7) & 3.7 (75.0) 
            \\
            &  w/ SP                    
            & 0.0 & 5.1 & 4.3 (84.3) & 4.1 (80.4) & 3.9 (76.5) & 4.1 (80.4) & 4.9 (\underline{96.1}) & 4.0 (78.4) & 4.5 (\underline{88.2}) & 4.5 (\underline{88.2}) 
            \\    
            &  w/ SM 
            & 0.0&  1.7   &  0.7 (---) & 0.7 (---)  & 1.3 (---)  & 1.0 (---)  & 1.3 (---)  & 0.7 (---)  & 0.3 (---)   & 0.8 (---)
            \\    
            &  w/ SD   
            & 0.0&  0.1  &  0.1 (---)  & 0.1 (---)   & 0.2 (---)  & 0.1 (---)  & 0.1 (---) & 0.3 (---) & 0.0 (---) & 0.1 (---)
            \\
            &  w/ SC             
            & 0.0 & \underline{13.0} & 7.7 (59.2) & 9.0 (69.2) & 6.3 (48.5) & 7.7 (59.2) & 8.0 (61.5) & 8.0 (61.5) & 8.7 (66.9) & 8.2 (63.1) 
            \\ 
            & w/ SCO
            & 0.0 & \textbf{18.7} & \textbf{12.1} (64.7) & \underline{10.5} (56.1) & \underline{11.3} (60.4) & \underline{11.3} (60.4) & \textbf{15.2} (81.3) & \textbf{15.9} (85.0) & \underline{9.8} (52.4) & \textbf{13.6} (72.7) 
            \\
            &  w/ BT   
            & 0.0 & 7.2 & 3.4 (47.2) & 3.5 (48.6) & 2.5 (34.7) & 3.1 (43.1) & 3.8 (52.8) & 3.6 (50.0) & 3.6 (50.0) & 3.7 (51.4) 
            \\
            & w/ CC 
            & 0.0 & 9.4 & 9.8 (\textbf{104.3}) & 7.9 (\underline{84.0}) & 7.9 (\underline{84.0}) & 8.5 (\underline{90.4}) & 8.5 (90.4) & 7.4 (78.7) & 6.5 (69.1) & 7.5 (79.8) 
            \\
            & w/ CD-CoT     
            & \ours 0.0 & \ours 12.3 & \ours \underline{12.0} (\underline{97.6}) & \ours \textbf{12.0} (\textbf{97.6}) & \ours \textbf{13.0} (\textbf{105.7}) & \ours \textbf{12.3} (\textbf{100.0}) & \ours \underline{12.3} (\textbf{100.0}) & \ours \underline{10.0} (\underline{81.3}) & \ours \textbf{11.0} (\textbf{89.4}) & \ours \underline{11.1} (\textbf{90.2})    
            \\ 
        \bottomrule
    \end{tabular}}
    \caption{
    Comparing accuracy and NDA metrics across all methods. The \textbf{boldface} numbers mean the best results; \underline{underlines} indicate the second-best. "---" denotes methods with poor results, where zero-shot performance approximates clean 3-shot outcomes, rendering NDA calculation meaningless.
    }
    \label{tab:NDA}
\end{table*}

\textbf{Standard Deviation Comparison.}
Tab.~\ref{tab:std} provides a comprehensive comparison of the standard deviations ($\sigma$) across various methods on the NoRa-Math Base-9 task, using the GPT-3.5-Turbo model. The table presents the $\sigma$ values for each method under different conditions, including clean rationale, irrelevant rationale (easy, medium, and hard), and inaccurate rationale (easy, medium, and hard). These results supplement the main experimental findings by offering insights into the variability and consistency of the methods' performance. 

\begin{table*}[ht]
    \centering
    \fontsize{7}{7}\selectfont
    \setlength{\tabcolsep}{1.8pt}
    \resizebox{\linewidth}{!}{
    \begin{tabular}{c|c|ccc|ccc}
        \toprule
        \multicolumn{1}{c|}{\multirow{2}{*}{Method $\mathcal{M}$}} &\multirow{2}{*}{\makecell[c]{ Clean (Acc $\mu \pm \sigma$)}} &  \multicolumn{3}{c|}{Irrelevant (Acc $\mu \pm \sigma$)} &  \multicolumn{3}{c}{Inaccurate (Acc $\mu \pm \sigma$)}\\
         &  &  Easy & Medium & Hard  & Easy & Medium & Hard \\
         \midrule
         \makecell[c]{Base}                          
         & \gray $46.4 \pm 2.54 $  
         & \gray $39.3 \pm 2.22 $ & \gray $30.3 \pm 2.62$ & \gray $26.6 \pm 2.17$  
         & \gray $23.2 \pm 2.92$  & \gray $10.1 \pm 1.34$  & \gray $6.0 \pm 1.42$
            \\
        w/ ISC                     
         & $24.3 \pm 2.00$  
             & $17.7 \pm 1.72$  & $14.7 \pm 1.04$  & $12.7 \pm 1.19$   
         &  $18.4 \pm 1.75$  &  $13.7 \pm 1.22$  & $12.3 \pm 0.82$   
            \\
        w/ SP                        
         &  $26.2 \pm 2.06 $
         &  $25.5 \pm 0.87 $ & $25.5 \pm  1.48$  &  $21.9 \pm 0.69$
         & $20.0 \pm 1.65 $  & $18.4 \pm 1.32 $  &   $14.3 \pm 2.17 $ 
            \\
        w/ SM                        
         &  $37.4 \pm 6.58 $
         & $30.0 \pm 3.82 $  & $22.7 \pm 4.24 $  &  $16.5 \pm 7.77 $
         & $24.7 \pm 3.32$  & $19.2 \pm 2.62$  &    $12.4 \pm 1.91 $
            \\
        w/ SD                        
        &  $47.9 \pm 1.63$ 
        &  $37.2 \pm 0.71$ & $25.4 \pm 3.68$  & $24.7 \pm 1.41$   
        &  $29.3 \pm 0.58$ & $12.5 \pm 0.28$  & $8.7 \pm 2.33$
        \\  
        w/ SC                      
        & $61.5 \pm 0.68$  
        & $51.1 \pm 1.25$  & $39.0 \pm 2.14$ & $36.2 \pm 2.31 $
        & $32.7 \pm 1.19$  & $15.3 \pm 0.91$  &  $7.5 \pm 1.37$
        \\
        w/ SCO                      
        &  $53.6 \pm 0.72$
        & $46.3 \pm 1.60$  & $39.6 \pm 1.61$  &   $36.4 \pm 0.15$
        & $34.7 \pm 0.88$  & $22.0 \pm 1.18$  &$17.7 \pm 1.99$
        \\
        w/ BT                      
        &  $47.2 \pm 1.35$
        &  $39.2 \pm 1.35$ &  $34.2 \pm 2.74$ & $29.9 \pm 2.10$   
        &  $30.1 \pm 1.59$ &  $18.4 \pm 2.03$ & $14.1 \pm 2.90$
        \\
        w/ CC                  
        &  $44.9 \pm 3.80$
        &  $43.3 \pm 1.36$ & $44.6 \pm 3.23$  & $45.5 \pm 2.88$  
        &  $37.2 \pm 1.85$ & $31.7 \pm 2.15$  & $30.7 \pm 2.94$
        \\
        w/ CD-CoT                      
        &  \ours$60.7 \pm 0.49$
        & \ours$59.7 \pm 0.49$  & \ours$60.7 \pm 3.18$  & \ours$57.2 \pm 1.63$   
        & \ours$54.0 \pm 1.41$  & \ours$58.7 \pm 4.53$  & \ours $48.4 \pm 1.91$
        \\
        \bottomrule
        \end{tabular}}
    \caption{
    Comparing the mean values and standard deviation values of accuracy across all methods. 
    }
    \label{tab:std}
\end{table*}
\subsection{The Superior Performance and Denoising Effectiveness}
\label{appendix: ablation study on CD-CoT}

The proposed CD-CoT method denoises noisy rationales, which leads to better reasoning performance. To be specific:
\begin{itemize}
    
    \item The first two steps of CD-CoT are for explicit data denoising. First, it rephrases the noisy example by contrasting it with the clean example. Then, with the obtained rephrased examples, it selects qualified candidates by checking the validity of the rephrased answers.
    \item The rephrased (denoise) rationales by different robust methods are shown in Tab.~\ref{tab:comparison-rephrased-rationales}. As can be seen, CD-CoT significantly removes noise and also ensures format alignment with the original rationale. More examples of denoising can be found in Appendix.~\ref{app:exp-qualitative}. Hence, these empirical results adequately justify the denoising power of CD-CoT.
\end{itemize}
The two important aspects of the proposed method are (1) contrastive denoising with rationale rephrasing (+ CD Rephrase) and (2) repeated reasoning with voting (+ Vote). Here, we conduct an extra ablation study on the NoRa-Math Base-9 dataset to justify the denoising power of CD-CoT. 

In this ablation study, we compare the performance of CD-CoT with different combinations of clean and noisy rationales on the NoRa-Math Base-9 dataset. The "Base model" column represents the accuracy of the base model without any additional components. The "+ CD Rephrase" column shows the accuracy when the contrastive denoising with rationale rephrasing (CD Rephrase) is applied. The "+ Vote" column represents the accuracy when the final answer is determined by a majority vote of five runs. Finally, the "CD-CoT" column shows the accuracy of the complete CD-CoT method, which includes both CD rephrasing and voting.

Since CD-CoT requires 3-shot noisy rationales and 1-shot clean rationale for contrastive rephrasing, we compare its performance with the following setups. The results are summarized in the Tab.~\ref{tab: Ablation study of CD rephrase and vote components}.

\begin{itemize}

\item  3-shot clean rationales (without CD-CoT)
\item  3-shot noisy rationales (irrelevant or inaccurate) of varying difficulty levels + 1-shot clean rationale (with CD-CoT)
\item 4-shot clean rationales (without CD-CoT)
\end{itemize}

\begin{table*}[ht]
    \centering
    \fontsize{8}{8}\selectfont
    \begin{tabular}{c|c|c|c|c}
        \toprule
        CoT shots & Base model & + CD Rephrase & + Vote & CD-CoT \\
        \midrule
        3-shot clean shots &  46.4& --- &\textbf{62.3}& ---\\
        \midrule
        3-shot irrelevant-easy + 1-shot clean &  46.0 & 47.2 & 57.3 & \textbf{60.0}\\
        \midrule
        3-shot irrelevant-medium + 1-shot clean &  41.3 & 47.2 & 54.7 & \textbf{60.7}\\
        \midrule
        3-shot irrelevant-hard + 1-shot clean &  36.2  & 46.5 & 47.3 & \textbf{58.3}\\
        \midrule
        3-shot inaccurate-easy + 1-shot clean &  36.5 & 42.6 & 44.7 & \textbf{53.0}\\
        \midrule
        3-shot inaccurate-medium + 1-shot clean &  23.4 & 43.0 & 32.3 & \textbf{58.7}\\
        \midrule
        3-shot inaccurate-hard + 1-shot clean &  16.1  & 37.0 & 24.3 & \textbf{47.0}\\
        \midrule
        4-shot clean rationale & \textbf{50.8} & --- & \textbf{67.3} & ---\\
        \bottomrule
        
    \end{tabular}
    \caption{Ablation study of CD rephrase and vote components on the NoRa-Math Base-9 dataset.}
    \label{tab: Ablation study of CD rephrase and vote components}
\end{table*}

From the experimental results, we can observe that:

\begin{observation}[CD-CoT significantly improves reasoning performance across various noisy scenarios]
CD-CoT brings a significant improvement compared to the base model when prompted with 3 noisy rationales + 1 clean rationale across various noise ratios of both irrelevant and inaccurate noise. The base accuracy of CD-CoT can be attributed to the improvement brought by the contrastive denoising with rationale rephrasing method. This accuracy is further enhanced in the voting method, which demonstrates the effectiveness of the self-consistency approach in improving reasoning performance.
\end{observation}

\begin{observation}[Noisy rationales with CD-CoT does not outperform full clean rationales]
The accuracy before voting and the accuracy after voting in CD-CoT do not exceed the accuracy of the base model prompted with 4 clean rationales. This indicates that while CD-CoT is effective at denoising rationales and enhancing reasoning performance in noisy situations, it does not inherently improve reasoning beyond the capabilities of the model when clean rationales are provided.
\end{observation}

\begin{observation}[Contrastive denoising and voting in CD-CoT are important]
The results highlight the importance and effectiveness of the two key aspects of CD-CoT: (1) contrastive denoising with rationale rephrasing, which helps to filter out the noise and align the format of the rationales, and (2) repeated reasoning with voting, which enhances the robustness and accuracy of the final predictions. Notably, CD-CoT proves to be a valuable approach for improving the reasoning performance in the presence of noisy rationales, though it does not surpass the performance with full clean rationales.
\end{observation}

\begin{observation}[CD-CoT is more effective at managing irrelevant noise compared to inaccurate noise]
As shown in Tab.~\ref{tab:std}, CD-CoT is more effective at handling irrelevant noise than it is at managing inaccurate noise at the same level. This is because irrelevant noise is easier to differentiate from the target information, making it simpler to remove during the rephrasing step of the contrasting denoising process. Additionally, Appendix.~\ref{app:exp-qualitative} presents the denoised results for high-noise scenarios.
\end{observation}

\textbf{Variants of prompt for CD-CoT.}
We also conduct additional experiments to test the sensitivity of CD-CoT to prompt designs. Specifically, we generate several variants of prompts for the proposed CD-CoT method, including a shorter prompt, the default prompt, and a longer one:

\begin{center}
\begin{tcolorbox}[
colback=white,
colframe=black,
width=1.0\linewidth,
arc=1mm, auto outer arc,
boxrule=0.75pt]
\label{box: shorter prompt variants}

\footnotesize 

\textbf{The Shorter Prompt} \\
Here are two examples: the first one has proper explanation and answer, while the second one has distracted explanation and correct answer. Please follow the first example's explanation and provide the correct explanation and answer for the second one.

\end{tcolorbox}
\end{center}

\begin{center}
\begin{tcolorbox}[
colback=white,
colframe=black,
width=1.0\linewidth,
arc=1mm, auto outer arc,
boxrule=0.75pt]
\label{box: default prompt variants}

\footnotesize 

\textbf{The Default Prompt} \\
Here are two examples for the same type of task: the first example has correct explanation and correct answer, and the second example has distracted explanation and correct answer. Please follow the first example and give me the correct explanation and answer for the second example, which should be logically consistent with the first one.

\end{tcolorbox}
\end{center}

\begin{center}
\begin{tcolorbox}[
colback=white,
colframe=black,
width=1.0\linewidth,
arc=1mm, auto outer arc,
boxrule=0.75pt]
\label{box: longer prompt variants}

\footnotesize 

\textbf{The Longer Prompt} \\
The following presents two examples of the same type of task. The first example contains both a correct explanation and a correct answer. The second example, however, includes a distracted explanation but still provides the correct answer. Your task is to analyze these examples and then provide a revised version explanation of the second example along with its answer. Ensure that your revised explanation is logically consistent with the first example.

\end{tcolorbox}
\end{center}
We conduct experiments on NoRa-Math Base-9 to compare these three prompts. The results in Tab.~\ref{Tab: ablation study of different prompts in CD-CoT} show that the performance of CD-CoT is only marginally influenced by these prompts. Note that the prompt of CD-CoT remains simple. Cooperating with advanced methods for iterating prompts and rationales, such as APE~\cite{zhou2022large} and Star~\cite{zelikman2022star}, can further improve this method. 

\begin{table*}[ht]
    \centering
    \fontsize{8}{8}\selectfont
    \begin{tabular}{c|c|c}
        \toprule
        method & Irrelevant-medium  & Inaccurate-medium \\
        \midrule
         Base Model & 0.284 & 0.082\\
        \midrule
        CD-CoT w/ original prompt & 0.49 & 0.48\\
        \midrule
        CD-CoT w/ short prompt & 0.46 & 0.46\\
        \midrule
        CD-CoT w/ long prompt & 0.47 & 0.48\\
        \bottomrule
    \end{tabular}
    \caption{Ablation study of different prompts in CD-CoT.}
    \label{Tab: ablation study of different prompts in CD-CoT}
\end{table*}

\textbf{Rephrasing performance of SD.} Besides, we observe that the SD method performs better on some tasks when using irrelevant and inaccurate rationales compared to clean ones. A thorough examination of the empirical results from the log files reveals some inconsistencies in the performance of the SD method, especially in the NoRa-Symbolic Equal and NoRa-Commonsense tasks. 

Specifically, for NoRa-Symbolic Equal task:
\begin{itemize}
    \item In this task, comparing SD and SC reveals that the domain knowledge of LLM is relatively weak. According to the experiment results, when SD rephrases symbolic tasks, it often directly returns the final answer of the original response in many cases. 
    \item The quality of the rephrased results basically does not change with the increase in noise.
    \item As a result, the accuracies of SD under clean and noisy settings are very close, with a maximum difference of less than 1\%, which we think is within the acceptable range of error.
\end{itemize}

For NoRa-Commonsense task:
\begin{itemize}
    \item We observe different behaviors of SD in clean and noisy settings. 
    \item In noisy settings, SD does not follow the prompt to make minor corrections. Instead, SD generates short rationales with only the answer, discarding the original rationale (as shown in Tables.~\ref{tab: rephrasing performance of SD exp1}-\ref{tab: rephrasing performance of SD exp2}). This simplification can benefit commonsense tasks where the model already has a strong inherent understanding.
    \item On the other hand, in the clean setting, SD unexpectedly generates entirely new rationales instead of making minor corrections. These new rationales may deviate from the original correct reasoning, potentially introducing errors or inconsistencies that weren't present in the initial, accurate solution.
\end{itemize}

In Tab.~\ref{tab: rephrasing performance of SD exp1} and Tab.~\ref{tab: rephrasing performance of SD exp2}, we present two representative examples of SD in the commonsense task. Our findings show that: (i) when provided with a clear rationale, SD generates entirely new rationales rather than simply making minor corrections; (ii) when given a noisy rationale, SD produces brief rationales that include only the answer, completely discarding the original rationale.

Overall, SD's denoising effect relies on LLMs' intrinsic capability. SD performs five maskings and reconstructions per noisy example, concatenating results into five prompts for LLM inference. In contrast, SC performs five direct inferences on noisy prompts without explicit denoising. By comparing SD and SC's performance, we can gain insights into the model’s denoising effects.

In Tab.~\ref{tab:vanilia_LLM}, SC outperforms SD in NoRa-Math and NoRa-Symbolic tasks, while SD only marginally excels in the NoRa-Commonsense task. This pattern persists in the clean settings. This suggests that LLMs struggle to reconstruct masked prompts in complex, domain-specific tasks but perform better in simpler common tasks, highlighting the varying levels of internal knowledge within LLMs across different domains.

Besides, the counterintuitive results on the NoRa-Commonsense task may be attributed to task-specific characteristics. Observation of reconstructed masked prompts in this task indicates that LLMs tend to bypass reconstruction instructions in noisy settings, directly providing final answers. This accidental removal of all rationales, including noisy ones, effectively acts as a noise filter. Combined with the LLM's natural strength in the commonsens task, this unplanned filtering likely explains the small improvement in accuracy under noisy conditions.

\begin{table*}[h!]
\centering
\fontsize{8}{8}\selectfont 
\begin{tabular}{p{\textwidth}} 
\toprule[1.5pt] 
\midrule 
\textbf{The original clean rationale}

Carrie is Hugh's sister's husband's daughter, so the relations path is sister, husband, daughter. For husband's daughter, we have husband's daughter is daughter. So the relations path are reduced to sister, daughter. For sister's daughter, we have sister's daughter is niece. So the relations path are reduced to niece. Therefore, Answer: niece.

\vspace{5pt}

\textbf{The rephrased rationale by SD} 

To determine Carrie's relationship to Hugh, we can break down the given information step by step: 
1. Carrie is Hugh's sister's husband's daughter. 
2. Hugh's sister is Carrie's sibling. 
3. Hugh's sister's husband is Carrie's brother-in-law. 
4. Carrie is the daughter of Hugh's sister's husband. Therefore, the answer is: Carrie is Hugh's niece.

\vspace{5pt}

\textbf{The original noisy rationale}

Carrie is Hugh's sister's husband's daughter, so the relations path is sister, husband, daughter. For husband's daughter, we have husband's daughter is daughter. We have daughter's grandmother is niece. So the relations path are reduced to sister, daughter. For sister's daughter, we have sister's daughter is niece. We have daughter's niece is niece. So the relations path are reduced to niece. Therefore, Answer: niece.
\vspace{5pt}

\textbf{The rephrased rationale by SD} 

niece.
\\
\midrule 
\bottomrule[1.5pt]
\caption{Rephrasing performance of SD under different noisy settings.}
\label{tab: rephrasing performance of SD exp1}
\end{tabular}
\end{table*}

\begin{table*}[h!]
\centering
\fontsize{8}{8}\selectfont 
\begin{tabular}{p{\textwidth}} 
\toprule[1.5pt] 
\midrule 
\textbf{The original clean rationale}

Francisco is Lynn's son's father's brother, so the relations path is son, father, brother. For father's brother, we have father's brother is uncle. So the relations path are reduced to son, uncle. For son's uncle, we have son's uncle is brother. So the relations path are reduced to brother. Therefore, Answer: brother. 

\vspace{5pt}

\textbf{The rephrased rationale by SD} 

To determine Francisco's relationship with Lynn, we start with the information provided: Francisco is Lynn's son's father's brother. 
Step 1: Lynn's son's father is Lynn's husband. 
Step 2: Francisco is Lynn's husband's brother. 
Based on this information, we can conclude that Francisco is Lynn's brother-in-law. Answer: brother-in-law

\vspace{5pt}

\textbf{The original noisy rationale}

Francisco is Lynn's son's father's brother, so the relations path is son, father, brother. For father's brother, we have father's brother is uncle. So the relations path are reduced to son, uncle. For son's uncle, we have son's uncle is brother. We have uncle's niece is nephew. So the relations path are reduced to brother. We have brother's nephew is mother. Therefore, Answer: brother.

\vspace{5pt}

\textbf{The rephrased rationale by SD} 

Answer: brother.
\\
\midrule 
\bottomrule[1.5pt]
\caption{Rephrasing performance of SD under different noisy settings.}
\label{tab: rephrasing performance of SD exp2}
\end{tabular}
\end{table*}

\subsection{The Difficulty of Irrelevant Noise Semantics}
\label{appendix: noise_semantic_difficulty}
We conduct extra experiments with different semantic difficulties. Specifically, we set up three levels of semantic difficulty:

\begin{itemize}
\vspace{-8pt}
\item \textbf{Level-1: Topic-irrelevant.} Noisy thought has little relevance to the task. Here, noisy thought involves concepts that are clearly irrelevant to the knowledge required for the task. 

For example, (1) in physics, the number one is the charge of a single unit of positive or negative electricity; (2) in chemistry, two is the atomic number of helium, the second most abundant element in the universe. These two are irrelevant thoughts.

\item \textbf{Level-2: Topic-relevant and task-irrelevant.} The noisy thought has some apparent relevance to the task but does not help solve the problem. Noisy thought may involve some terminology or concepts that are relevant to the task but do not provide useful information. 

For example, (1) the modulo operation (mod) finds the remainder after the division of one number by another. (2) In base-9, the leading digit 1 indicates that the number is greater than or equal to $1 * 9^2 = 81$ but less than $2 * 9^2 = 162$. (3) In base-9, the number 9 is represented as '10' since it is the base value itself. (4) The number 11 in base-9 is equivalent to the decimal value ninety-nine, which is one less than the square of the base value.

\item \textbf{Level-3: Topic-relevant, task-relevant, but not helpful.} Here, the noisy thought has some relevance to the task, but is not necessary information to solve the problem. The noisy thought may involve concepts, principles, or methods that are relevant to the task, but not directly applicable to the problem at hand. 

For example, (1) in a base-9 system, the place values are units, nines, eighty-ones, and so on, analogous to units, tens, hundreds, etc., in base-10; (2) the concept of carrying is essential when performing addition in positional numeral systems like base-9 and base-10. The modulo operation is a crucial step in the process of adding numbers in base-9, as it determines the digit and carry values; (3) the modulo operation is a crucial step in the process of adding numbers in base-9, as it determines the digit and carry values.

\end{itemize}

Based on the above definition and prompt, we generate the data with the following four steps.

\begin{itemize}
\vspace{-8pt}
\item \textbf{Step-1: Initial generation and human evaluation.} We used this prompt to generate an initial set of irrelevant statements for each level. These were then manually reviewed and filtered to ensure they accurately represented the intended level of irrelevance. We selected high-quality examples for each level.

\item \textbf{Step-2: Scaled generation.} Using these high-quality examples as ICL demonstrations, we prompted GPT-4 to generate a larger set of irrelevant statements for each level.

\item \textbf{Step-3: Validation through similarity analysis.} To confirm that our generated statements indeed represented different levels of irrelevance, we conducted a cosine similarity analysis. This analysis quantitatively demonstrated the semantic differences between levels, as shown in our previous response.

\item \textbf{Step-4: Dataset construction.} we integrated these generated irrelevant statements into our dataset. We inserted them into relevant demonstrations at appropriate positions, following a probability distribution that corresponds to the intended difficulty level of the task.
\end{itemize}

The experiment results with the above irrelevant thoughts are elaborated as follows: Here, the setup is kept the same as the main experiments to be consistent. Therein, we test the GPT-3.5 model with 3-shot irrelevant rationales where the noise ratio is 0.8.

For a reasonable quantification of relevance, we choose the metric of cosine similarity with the \href{https://platform.openai.com/docs/guides/embeddings}{embeddings API} provided by OpenAI with model: text-embedding-ada-002. We calculate the cosine similarity between all noise-inserted sample sentences and the original clean inference sentences and then take their average value.


\begin{table*}[ht]
    \centering
    \fontsize{8}{8}\selectfont
    \begin{tabular}{c|c|c|c}
        \toprule
        Dataset & Semantic level & Method & Acc(\%) \\
        \midrule
        \multirow{6}{*}{Math Base-9} & \multirow{2}{*}{Level-1 (Cosine Similarity = 0.75)}&base model & 26.6\\
        && CD-CoT & 58.3\\
        \cmidrule{2-4}
        & \multirow{2}{*}{Level-2 (Cosine Similarity = 0.87)}&base model & 30.6\\
        && CD-CoT & 60.0\\
        \cmidrule{2-4}
        & \multirow{2}{*}{Level-3 (Cosine Similarity = 0.88)}&base model & 33.0\\
        && CD-CoT & 59.7\\
        \midrule
        \multirow{6}{*}{Symbolic Equal} & \multirow{2}{*}{Level-1 (Cosine Similarity = 0.73)}&base model & 23.0\\
        && CD-CoT & 44.0\\
        \cmidrule{2-4}
        & \multirow{2}{*}{Level-2 (Cosine Similarity = 0.79)}&base model & 26.9\\
        && CD-CoT & 46.0\\
        \cmidrule{2-4}
        & \multirow{2}{*}{Level-3 (Cosine Similarity = 0.82)}&base model & 27.7\\
        && CD-CoT & 45.0\\
        \bottomrule
    \end{tabular}
    \caption{Performance of GPT-3.5-turbo under different noise semantics.}
\end{table*} 

For comparison, the accuracy when prompting with clean rationales is 46.4\% for NoRa-Math Base-9 and 32.7\% for NoRa-Symbolic Equal. Specifically, we have the following observations:

\begin{observation}[CD-CoT method is still effective under different semantic levels of noise]
Notably, we further verify that the proposed CD-CoT method is still effective under this imperfect scenario with different semantic levels of noise. Specifically, CD-CoT can achieve 58.3\%, 60.0\%, and 59.7\% accuracy on the three semantic levels of noise for the NoRa-Math Base-9 tasks and 44.0\%, 46.0\%, and 45.0\% accuracy on the NoRa-Symbolic Equal tasks. 
Compared to the base model, CD-CoT shows a significant improvement of 29.4\% and 19.1\% in accuracy on semantic level-2 reasoning for the NoRa-Math Base-9 and NoRa-Symbolic Equal tasks, respectively. Similarly, CD-CoT demonstrates a 26.7\% and 17.3\% increase in accuracy on semantic level-3 reasoning for the NoRa-Math Base-9 and NoRa-Symbolic Equal tasks.
\end{observation}

\begin{observation}[CD-CoT maintains a robust and stable performance across different semantic levels of noise.]
As the semantic level of noise increases, the base model exhibits a certain level of robustness to the irrelevant noise but still suffers a considerable performance drop (15.8\% and 5.8\% decrease for semantic level-2, and 13.4\% and 5.0\% decrease for semantic level-3 on the NoRa-Math Base-9 and NoRa-Symbolic Equal tasks, respectively). In contrast, CD-CoT maintains a robust and stable performance across different semantic levels of noise.
\end{observation}

\subsection{The Number of Noisy Thoughts}
\label{appendix: the Number of Noisy Thoughts}

As introduced in the main text, we use noise ratios of 0.3, 0.5, and 0.8 to control the number of noisy thoughts. The statistics of the NoRa dataset are shown in Tab.~\ref{tab:dataset_statis}, showing the exact number of clean/noisy thoughts in each task. NoRa-Math entails 8 thoughts, NoRa-Symbolic varies from 2 to 12, and NoRa-Commonsense requires 5 thoughts.

To figure out the exact number of noisy thoughts for the model to be robust, we conduct further experiments to cover a wider range of noisy thoughts. Specifically, 

\begin{itemize}

\item  We compare the model's performance on NoRa-Math Base-9 and NoRa-Symbolic Equal tasks using clean and noisy rationales containing varying numbers of irrelevant or inaccurate thoughts. 
\item We gradually increase the number of noisy thoughts from 1 to 6 to investigate the impact of the noise level on the model's accuracy. 
\item The total number of thoughts per shot, including both clean and noisy thoughts, is shown in parentheses for the NoRa-Math Base-9 task. 
\item For the NoRa-Symbolic Equal task, the number of thoughts in the clean demonstrations varies across different examples. 
\item To ensure a consistent comparison, we select a subset of examples with a similar number of thoughts (ranging from 8 to 12) for both clean and noisy rationale experiments. 
\item We report the average number of thoughts (10.2) for the clean rationales in this subset. 
\item Consequently, the total number of thoughts per shot for noisy rationales in the NoRa-Symbolic Equal task is also calculated based on this average.
\end{itemize}

\begin{table*}[ht]
    \centering
    \fontsize{8}{8}\selectfont
    \begin{tabular}{c|c|c|c}
        \toprule
        Dataset & Noise Type & Noisy Thought Num & Accuaracy \\
        \midrule
        \multirow{13}{*}{NoRa-Math Base-9}& clean & 0 (total thoughts per shot: 8) & 46.4 \\
        \cmidrule{2-4}
        & \multirow{6}{*}{irrelevant} & 1 (9) & 46.4 \\
        && 2 (10) & 40.0 \\
        && 3 (11) & 34.2 \\
        && 4 (12) & 30.5 \\
        && 5 (13) & 28.8 \\
        && 6 (14) & 25.3 \\
        \cmidrule{2-4}
        &\multirow{6}{*}{inaccurate} & 1 & 34.8 \\
        && 2 & 24.8 \\
        && 3 & 15.1 \\
        && 4 & 11.1 \\
        && 5 & 6.8 \\
        && 6 & 6.2 \\
        \midrule
        \multirow{13}{*}{NoRa-Symbolic Equal}& clean & 0 (total thoughts per shot: 10.2) & 39.8 \\
        \cmidrule{2-4}
        & \multirow{6}{*}{irrelevant} & 1 (11.2) & 32.8 \\
        && 2 (12.2)  & 34.2 \\
        && 3 (13.2) & 33.3 \\
        && 4 (14.2) & 30.6 \\
        && 5 (15.2) & 28.4 \\
        && 6 (16.2) & 29.2 \\
        \cmidrule{2-4}
        & \multirow{6}{*}{inaccurate} & 1  & 34.5 \\
        && 2 & 31.3 \\
        && 3 & 32.4 \\
        && 4 & 31.4 \\
        && 5 & 28.3 \\
        && 6 & 27.3 \\
        \bottomrule
    \end{tabular}
    \caption{Impact of the number of noisy thoughts on GPT-3.5 performance.}
    \label{tab:noisy-thought-num}
\end{table*}

According to the experimental results in the Tab.~\ref{tab:noisy-thought-num}, we can observe that:

\begin{observation}[The impact of noisy thoughts varies across different noise types]
For the NoRa-Math Base-9 task, the model's accuracy remains relatively stable when only one irrelevant noisy thought is introduced. However, the accuracy drops significantly as the number of irrelevant noisy thoughts increases. On the other hand, for the NoRa-Symbolic Equal task, even a single irrelevant or inaccurate noisy thought can lead to a noticeable decrease in accuracy.

\end{observation}

\begin{observation}[Inaccurate noisy thoughts have a more severe impact on the model's performance compared to irrelevant noisy thoughts]
In the NoRa-Math Base-9 task, the accuracy drops drastically from 46.4\% to 34.8\% with just one inaccurate, noisy thought, and it further deteriorates to 6.2\% when the number of inaccurate, noisy thoughts reaches 6. In contrast, with 6 irrelevant noisy thoughts, the accuracy decreases to 25.3\%, which is still higher than that of inaccurate thoughts.
\end{observation}

\begin{observation}[Model's robustness to noisy thoughts varies depending on the nature of the task] The NoRa-Symbolic Equal task exhibits a similar trend, where the accuracy decreases as the number of noisy thoughts increases, regardless of the noise type. However, the decline in accuracy is more gradual compared to the NoRa-Math Base-9 task, suggesting that the model's robustness to noisy thoughts varies depending on the nature of the task.
\end{observation}

The experimental results demonstrate that the presence of noisy thoughts, even in small numbers, can significantly impact the model's reasoning performance. The extent of the impact depends on factors such as the task type, noise type, and number of noisy thoughts.





\subsection{The Number of Noisy Examples}
\label{app: number of noisy examples}

In this part, we provide a deeper analysis and extract more insights into in-context learning with varying numbers of noisy examples. 

First, we obtain the accuracy of the base model with different numbers of noisy prompting examples, including 0-shot, 1-shot, 3-shot, and 5-shot. According to the results shown in Tab.~\ref{appendix: accuracies with varying number of prompts}, we propose the following hypothesis: \textit{The trend of model performance as the number of noisy examples increases depends on the change in performance between 1-shot and 0-shot learning.} 

Namely, if the performance in 1-shot learning is better than in 0-shot learning, this indicates that noisy examples still help the model in recognizing task patterns. Therefore, model performance can improve with an increasing number of noisy examples, albeit at a slower rate compared to clean examples. On the other hand, if the performance in one-shot learning is less than in zero-shot learning, it suggests that noisy examples interfere with the model's existing knowledge of the task, causing model performance to decline as the number of noisy examples increases.

\begin{table*}[h!]
    \centering
    \fontsize{8}{8}\selectfont
    \begin{tabular}{c|c|cccc}
            \toprule
            \multirow{2}{*}{Task} & \multirow{2}{*}{Setting} & \multicolumn{4}{c}{\#Prompting Examples} \\ 
            &            & 0-shot &	1-shot & 3-shot & 5-shot \\ 
            \midrule
            \multirow{2}{*}{Base-9} & ina.-easy & 7.2	& 17.5 & 23.2 & \textbf{25.6} \\ 
            & ina.-hard   & 7.2 & \textbf{11.3} & 6.0 & 5.7 \\ 
            \midrule
            \multirow{2}{*}{Base-11} & irr. easy  & 5.5	& 8.9 & 19.1 & \textbf{26.3} \\
            & irr. hard    & 5.5 & 7.7 & 10.7 & \textbf{16.1}\\
            \midrule
            \multirow{2}{*}{Sym.(E)} & ina.-easy & 8.8	& 17.3 & \textbf{29.1} & ---\\
            & ina.-hard   & 8.8 & 15.0 & \textbf{22.7} & ---\\
            \midrule
            \multirow{2}{*}{Sym.(L)} & irr. easy  & 0.0	& 2.3 & 7.0 & \textbf{8.9}\\
            & irr. hard  & 0.0 & 1.9 & \textbf{6.0} & --- \\
            \bottomrule
        \end{tabular}
    \caption{
    Performances of base model under varying number of examples
    ("---" denotes over token limit, \textbf{}{bolded number} denotes the highest accuracy).}
    \label{appendix: accuracies with varying number of prompts}
\end{table*}

Next, we conduct a further quantitative analysis of the impact of noisy shots on model performance. We evaluate (0-shot, 1-shot, 3-shot) results partition, there are 8 possible results in total: (CCC), (CCW), (CWC), (CWW), (WCC), (WCW), (WWC), (WWW). Specifically, 'C' stands for correct and 'W' stands for wrong. For example, CCC means the cases in which test questions are correctly answered in 0-shot (C), 1-shot (C), and 3-shot (C). WWC means the cases in which test questions are wrongly answered in 0-shot (W), 1-shot (W), and correctly answered in 3-shot (C).

To analyze the effects of noisy shots, we categorize these eight results into three types:

\begin{itemize}
\vspace{-8pt}
    \item The model answers correctly at zero-shot but starts to answer incorrectly as noisy shots increase. The corresponding results are (CCW), (CWC), and (CWW).
    \item The model answers incorrectly at zero-shot but starts to answer correctly as noisy shots increase. The corresponding results are (WCC), (WCW), and (WWC).
    \item The model's performance is consistent at zero-shot and with noisy shots. The corresponding results are (CCC) and (WWW).
    \vspace{-8pt}
\end{itemize}

\begin{table*}[h!]
    \centering
    \fontsize{8}{8}\selectfont
    \begin{tabular}{c|c|cccccccc}
            \toprule
            \multirow{2}{*}{Task} & \multirow{2}{*}{Setting} & \multicolumn{8}{c}{\#Result Partition} \\ 
            &  & CCC & CCW & CWC & CWW & WCC & WCW & WWC & WWW\\ 
            \midrule
            \multirow{2}{*}{Base-9} & ina.-easy & 0.006	& 0.01	& 0.019	& 0.037	& 0.049 & 0.11	& \underline{0.158} & \textbf{0.611} \\ 
            & ina.-hard   & 0.003 &	0.01 & 0.004	& 0.055& 0.012 & \underline{0.088} & 0.041 & \textbf{0.786}\\ 
            \midrule
            \multirow{2}{*}{Base-11} & irr. easy  & 0.004 & 0.003 & 0.013 & 0.036 & 0.029	& 0.054 & \underline{0.144} & \textbf{0.718}\\
            & irr. hard  &  0.001 & 0.001 & 0.008	& 0.043 & 0.005 & 0.069 & \underline{0.075} & \textbf{0.797}\\
            \midrule
            \multirow{2}{*}{Sym.(E)} & ina.-easy & 0.053	& 0.008 &	0.012 &	0.015 &	0.067	& 0.045	& \underline{0.16}	& \textbf{0.641}\\
            & ina.-hard   & 0.049	& 0.007	 & 0.015	& 0.016	& 0.042 &	0.052	& \underline{0.121}	& \textbf{0.698}\\
            \midrule
            \multirow{2}{*}{Sym.(L)} & irr. easy  &0.0	&0.0 & 0.0 & 0.0 &	0.005 & 0.018 & \underline{0.058}	& \textbf{0.919}\\
            & irr. hard  & 0.0 & 0.0 & 0.0 & 0.0	& 0.004	 & 0.015 & \underline{0.056} & \textbf{0.925}\\
            \bottomrule
        \end{tabular}
    \caption{
    Results partition of (0-shot, 1-shot, 3-shot).}
    \label{appendix: results partition}
\vspace{-5pt}
\end{table*}

By observing Tab.~\ref{appendix: results partition}, we arrive at the following analysis.

\begin{observation}
[Compared to zero-shot, models can benefit from noisy shots in many cases] Across all datasets, (WWW) is the most prevalent outcome, while the second most common results are almost always (WWC) or (WCW). This implies that there is a relatively large proportion of samples for which the model answers incorrectly at zero-shot, yet correctly as noisy shots increase.
\end{observation}

\begin{observation}
    [The instances where noisy shots have a negative effect are very rare.]
    For the first category of results, namely (CCW), (CWC), and (CWW), the average proportion across all datasets is 0.039. This indicates that the proportion of samples where the model answers correctly at zero-shot but incorrectly as noisy shots increase is very small. In other words, the instances where noisy shots have a negative effect are rare.
\end{observation}

\begin{observation}
    [There are significantly more instances where the model benefits from increasing noisy shots]
    For the second category of results, which includes (WCC), (WCW), and (WWC), the average proportion across all datasets is 0.186, which is greater than that of the first category. This means that there is a larger proportion of samples where the model answers incorrectly at zero-shot but correctly as noisy shots increase. Namely, there are more instances where noisy shots have a positive effect.
\end{observation}

\begin{observation}
    [The type of instances where the model produces consistent answers is dominant across all datasets]
    For the third category of results, namely (CCC) and (WWW), the average proportion across all datasets is the highest, at 0.776. This category of results does not allow for an assessment of the effect of noisy shots.
\end{observation}

We also provide a case study of in-context learning using a varying number of noisy examples. We provide cases corresponding to some possible outcomes (\textit{i.e.}, [(CCC), (CCW), (CWC)]) for (0-shot, 1-shot, 3-shot) on the NoRa-Math Base-9 dataset with inaccurate easy-level noise in Tables.~\ref{tab: Case for (CCC)}-\ref{tab: Case for (CWC)}.

\begin{table*}[h!]
\centering
\fontsize{8}{8}\selectfont 
\begin{tabular}{p{\textwidth}} 
\toprule[1.5pt] 
\midrule 
\textbf{Question}

In base-9, what is 25+18? Please reason it step by step. End the response with the result in "Answer:\boxed{result}".
\vspace{5pt}

\textbf{0-shot (Correct)} 

To perform addition in base-9, we need to first understand the place values. In base-9, the place values range from right to left as follows: ones, nines, eighty-ones, and so on. Let's start with the ones place. 5 plus 8 is equal to 13. However, in base-9, we can only have digits from 0 to 8. Therefore, we need to regroup. Since 13 is greater than 9, we can subtract 9 from it and carry over 1 to the next place value (nines place). The remainder is 4. Now let's move on to the nines place. We have 1 (carry over) plus 2 plus 1 (from the nines place) which is equal to 4. Therefore, when we add 25 and 18 in base-9, we get 44. Answer:$\boxed{44}$.
\vspace{5pt}

\textbf{1-shot (Correct)}

In base-9, the digits are "012345678". We have 5 + 8 = 13 in base-10. Since we're in base-9, that exceeds the maximum value of 8 for a single digit. 13 mod 9 = 4, so the digit is 4 and the carry is 1. Next, we have 2 + 1 + 1 (carry) = 4 in base-10. 4 in base-9 is written as 4. Since there's no carry this time, we have no additional digits. So, the sum is 44 in base-9. Answer: $\boxed{44}$.
\vspace{5pt}

\textbf{3-shot (Correct)} 

In base-9, the digits are "012345678". First, we add the digits in the rightmost place value, which is 5 + 8. We have 5 + 8 = 13 in base-10. Since we're in base-9, the maximum value for a single digit is 8. 13 mod 9 = 4, so the digit is 4 and the carry is 1. Next, we add the digits in the next place value, which is 1 (from the carry) + 2 + 1. We have 1 + 2 + 1 = 4 in base 10. Again, 4 mod 9 = 4, so the digit is 4 and there is no carry. Therefore, the final result is 44. Answer: $\boxed{44}$.
\\
\midrule 
\bottomrule[1.5pt]
\caption{Case for (CCC): 0-shot correct, 1-shot correct, 3-shot correct.}
\label{tab: Case for (CCC)}
\end{tabular}
\end{table*}

\begin{table*}[ht]
\centering
\fontsize{8}{8}\selectfont 
\begin{tabular}{p{0.95\textwidth}} 
\toprule[1.5pt] 
\midrule 
\textbf{Question}

In base-9, what is 57+26? Please reason it step by step. End the response with the result in "Answer: \boxed{result}".
\vspace{5pt}

\textbf{0-shot (Correct)} 

To add in base-9, we start by adding the digits in the rightmost column. In the units column, 7 + 6 = 13. Since we are in base-9, we need to regroup. We can rewrite 13 as 4 in the units column and carry-over 1 to the next column. In the next column, we add the carried-over 1 to the digits in that column. We have 1 + 5 + 2 = 8. So, the result in the next column is 8. Finally, we put the digits 8 and 4 together to get the final answer. Answer: $\boxed{84}$.
\vspace{5pt}

\textbf{1-shot (Correct)}

To add 57 and 26 in base-9, we'll start with the units digit. In base-9, the digits are "012345678". 7 + 6 = 13. Since 13 is greater than 8 (the maximum digit value in base-9), we need to regroup. We can write 13 as 13 = 1 * 9 + 4. Therefore, the units digit is 4 and there is a carry of 1 to the next column. Moving on to the next column, we have 5 + 2 + 1 (the carried digit) = 8. 8 is a valid digit in base-9, so we do not need to regroup. The tens digit is 8. Therefore, the sum of 57 and 26 in base-9 is 84. Answer: $\boxed{84}$.
\vspace{5pt}

\textbf{3-shot (Wrong)} 

In base-9, the digits are "012345678". We have 7 + 6 = 13. Since we're in base-9, we need to determine the digit and carry. 13 divided by 9 is 1 with a remainder of 4. So the digit is 4, and we carry 1. Next, we have 5 + 2 + 1 = 8. So the digit is 8 and there is no carry. A leading digit is not needed since there is no carry. Thus, the answer is 48. Answer: $\boxed{48}$.
\\
\midrule 
\bottomrule[1.5pt]
\caption{Case for (CCW): 0-shot correct, 1-shot correct, 3-shot wrong.}
\label{tab: Case for (CCW)}
\end{tabular}
\end{table*}

\begin{table*}[t!]
\centering
\fontsize{8}{8}\selectfont 
\begin{tabular}{p{0.95\textwidth}} 
\toprule[1.5pt] 
\midrule 
\textbf{Question}

In base-9, what is 67+25? Please reason it step by step. End the response with the result in "Answer: \boxed{result}".
\vspace{5pt}

\textbf{0-shot (Correct)} 

To add in base-9, we start by adding the digits in the rightmost place value and moving leftwards. In the units place, 7+5=12. However, in base-9, there are only digits from 0 to 8. Therefore, we need to do a "carry" to the next place value. The digit 12 is written as 13 in base-9, since the digit 12 in base-9 is represented as 10, followed by 11. Now, we move to the next place value, the 9s place. We have a carry of 1 from the units place, so the addition becomes 1+6+2 = 9. In base-9 notation, 9 is written as 10. Therefore, the sum of 67+25 in base-9 is 103. Answer: $\boxed{103}$.
\vspace{5pt}

\textbf{1-shot (Wrong)}

In base-9, the digits are "012345678". We have 7 + 5 = 12 in base-10. Since we're in base-9, that doesn't exceed the maximum value of 8 for a single digit. 12 mod 9 = 3, so the digit is 3 and the carry is 1. We have 1 + 6 + 2 = 9 in base-10. 9 mod 9 = 0, so the digit is 0 and the carry is 1. We have 1 + 1 + 0 = 2 in base-10. So the final digit is 2. Therefore, the answer is 203. Answer:\boxed{203}.
\vspace{5pt}

\textbf{3-shot (Correct)} 

In base-9, the digits are "012345678".  We have 7 + 5 = 12 in base-10. Since we're in base-9, that exceeds the maximum value of 8 for a single digit. 12 mod 9 = 3, so the digit is 3 and the carry is 1. We have 6 + 2 + 1 = 9 in base 10. 9 mod 9 = 0, so the digit is 0 and the carry is 1. A leading digit is 1. So the answer is 103. Answer:\boxed{103}.
\\
\midrule 
\bottomrule[1.5pt]
\caption{Case for (CWC): 0-shot correct, 1-shot wrong, 3-shot correct.}
\label{tab: Case for (CWC)}
\end{tabular}
\end{table*}

\subsection{Variants of CD-CoT}
\label{app:variant of CD-CoT}

Acquiring a human-annotated and clean rationale can be challenging in some practical scenarios. However, we would further clarify the relatively low cost and the necessity of this clean rationale to guarantee the effectiveness of robust reasoning.

Current LLMs cannot discern and remove noisy thoughts with their intrinsic denoising ability, even when enhanced with the self-denoising method. This is supported by the empirical observations in Sec.~\ref{sec: experiment}.
Hence, external supervision is necessary, which should be sufficient for denoising and accessible in practice. Existing methods with external supervision require (1) oracle feedback on the test question, (2) human feedback of errors on specific tokens or positions, or (3) expert knowledge to construct detailed descriptions of specific tasks.

By contrast, we believe that a clean CoT demonstration is more attainable and practical, which can be the minimal requirement for denoising-purpose prompting. With the proposed CD-CoT method, we show that LLMs can identify noisy thoughts by contrasting a noisy rationale with a clean one.

Besides, we build a self-supervised variant of CD-CoT that does not rely on manually crafted clean examples, \textit{i.e.}, this variant can get a clean rationale automatically. Specifically, we prompt LLM with a fixed question using the zero-shot CoT approach~\cite{kojima2022large}, allowing it to answer N times. We then select the responses with the correct answer based on the ground-truth label. 

Then, we conduct experiments on two datasets, NoRa-Math Base-9 and NoRa-Symbolic Equal, with the default model (GPT-3.5-0613). The hyper-parameters are temperature=0.1 and top-p=1. We set N=50 for NoRa-Math Base-9 and N=100 for NoRa-Symbolic Equal. Other settings are kept the same as default with the original CD-CoT.

\begin{table*}[h!]
    \centering
    \fontsize{8}{8}\selectfont
    \begin{tabular}{c|c|ccc}
            \toprule
            \multirow{2}{*}{Task} & \multirow{2}{*}{Setting} & \multicolumn{3}{c}{\#Method} \\ 
            &            & Base Model &	Variant of CD-CoT & CD-CoT \\ 
            \midrule
            \multirow{2}{*}{Base-9} & irr.-med & 30.3	& 47.0	& 60.7 \\ 
            & ina.-med   & 10.1	& 47.0	& 58.7\\ 
            \midrule
            \multirow{2}{*}{Sym.(E)} & irr.-easy & 28.1	& 33.0	& 44.7\\
            & ina.-easy   & 29.1	& 32.0	& 42.6\\
            \bottomrule
        \end{tabular}
    \caption{
    Performances of different variants of CD-CoT.}
    \label{appendix: variant CD-CoT}
\end{table*}

On the NoRa-Math Base-9 dataset, the LLM answers incorrectly in the initial 20 attempts.  When we increase N to 50, the LLM gets it right twice. We use one of the self-generated correct rationales to execute CD-CoT, testing the effect of LLM's self-generated rationale. 
On the NoRa-Symbolic Equal dataset, we need to increase N to 100 to get a correct self-generated response.

Regarding the empirical results in Tab.~\ref{appendix: variant CD-CoT}, we derive the following detailed analysis.

\begin{observation}
    [The self-supervised variant of CD-CoT achieves enhancements across various tasks]
    On NoRa-Math Base-9, the average accuracy is increased by 26.8 percentage points. 
    On NoRa-Symbolic Equal, it gains an average increase of 3.9 points in accuracy.
\end{observation}

\begin{observation}
    [LLM struggles to generate a clean and correct rationale on its own]
    Considering the fact that the model requires 50 attempts on the simple NoRa-Math Base-9 task to obtain the correct result, it indicates that the LLM struggles to generate a clean and correct rationale on its own.
Additionally, experiments on the NoRa-Symbolic Equal indicate that as the task's difficulty increases, it becomes more challenging for the model to resolve it by relying solely on its inherent capabilities.
\end{observation}

\begin{observation}
    [Self-generated clean examples are not as effective as manually designed clean examples]
    Although the self-supervised variant of CD-CoT acquires improvement across tasks, its gains on the difficult NoRa-Symbolic Equal tasks are quite modest, with an average increase of only 3.9 points.
    Conversely, the original CD-CoT registers an average accuracy enhancement of 39.5 points on NoRa-Math Base-9 tasks and 15.1 points on NoRa-Symbolic Equal tasks, respectively.
    We can conclude that self-generated clean examples can provide some guidance to the LLM on simpler tasks, but they are not as effective as manually designed clean examples. On more difficult tasks, the guidance provided by self-generated clean examples is very limited, and these tasks require additional supervision to steer the model effectively.
\end{observation}

In addition to the self-supervised CD-CoT, exploring strategies like using contrasting noisy examples or incorporating an external knowledge base, possibly through a retrieval-augmented denoising framework, may offer significant advances in automation and robustness of reasoning. We believe that these are promising directions that deserve further investigation.

\subsection{New Datasets and Large-scale Real Problems}
\label{appendix: Large-scale Real Problems}
Recall that the constructed NoRa benchmark covers five prevailing datasets from three different domains. The current LLMs present significant vulnerability to noisy rationales in all five datasets of NoRa, while the proposed method CD-CoT has shown advanced and consistent robustness against noisy rationales. Specifically,
\begin{itemize}
    \item Robustness with different datasets: CD-CoT consistently outperforms other methods in all five datasets in NoRa.
    \item Robustness with different noise levels: The results shown in Tab.~\ref{tab:proposed-method performance} demonstrate the remarkable robustness of CD-CoT to varying noise levels. Across the NoRa-Math, NoRa-Symbolic, and NoRa-Commonsense tasks, the performance decline of CD-CoT remains modest as the noise level increases.
    \item Robust with different LLMs: The results in Tab.~\ref{tab:methods_on_different_llms} further indicate that CD-CoT substantially improves over all three other baselines on the more powerful LLMs. Even on the relatively smaller Mistral-8x7B, CD-CoT significantly outperforms the other baselines on most tasks. 
\end{itemize}

Note that Sec.~\ref{sec: dataset} introduces a general framework for generating noisy rationales with existing datasets. This means, more datasets can be integrated into NoRa if necessary in future research. Here, we conduct additional experiments to generalize CD-CoT to three new datasets that are not covered in NoRa, including GSM-8K~\cite{cobbe2021training}, Blocksworld~\cite{valmeekam2023planning}, and BIG-Bench Hard Dyck Languages~\cite{srivastava2022beyond}. Specifically,
\begin{itemize}
    \item GSM-8K: A math dataset of linguistically diverse grade school math word problems.
    \item Blocksworld: A planning dataset simulating block stacking tasks.
    \item BIG-Bench Hard Dyck Languages: A symbolic dataset designed for predicting the sequence of closing parentheses in a Dyck-4 word.
\end{itemize}

Then, we generate noisy rationales and compare the following setups:
\begin{itemize}
    \item Zero-shot: Base model with no demonstration.
    \item CoT (clean rationales): Base model with 3 clean demonstrations.
    \item CoT (noisy rationales): Base model with 3 noisy demonstrations.
    \item CD-CoT (noisy rationales): Base model with 3 noisy demonstrations and our CD-CoT method.
\end{itemize}

The reasoning accuracy in Tab.~\ref{appendix: performances of different methods on new datasets} shows that CD-CoT consistently outperforms the zero-shot setting and standard CoT prompting under noisy rationales. This is consistent with the findings in Sec.~\ref{sec: Empirical Study}, showing CD-CoT’s strong capability of generalization to new datasets.

\begin{table*}[h!]
    \centering
    \fontsize{8}{8}\selectfont
    \begin{tabular}{c|c|cccc}
            \toprule
            \multirow{2}{*}{Dataset} & \multirow{2}{*}{Setup} & \multicolumn{4}{c}{\#Method} \\ 
            &            & Zero-shot &	CoT (clean) & CoT (noisy) & CD-CoT (noisy) \\ 
            \midrule
            \multirow{1}{*}{GSM-8K} & 300 samples & 84.3 & 87.7 & 84.3 & 86.0 \\ 
            \midrule
            \multirow{1}{*}{Blocksworld} & 200 samples & 2.0 & 25.0 & 13.0 & 25.5 \\ 
            \midrule
            \multirow{1}{*}{Hard Dyck Languages} & 250 samples & 12.4 & 40.8 & 29.2 & 35.2\\
            \bottomrule
        \end{tabular}
    \caption{Performances of different methods on new datasets.}
    \label{appendix: performances of different methods on new datasets}
\vspace{-5pt}
\end{table*}

Along with the introduction of new datasets, we also examine the issue of noisy rationale in large-scale real-world scenarios, particularly in conversational question-answering (QA) tasks. While our primary focus is on reasoning tasks, we discuss the differences and relationships between these areas.

Conversational QA tasks are usually conducted in a zero-shot setting, where users directly input the task instructions and hardly provide examples to the model. 
NoRa, on the other hand, considers the in-context learning scenario, where the model is given a few examples before answering questions. In-context learning can significantly enhance the model’s performance on complex or new tasks. The purpose of NoRa is to assess the model’s robustness to noisy examples in in-context learning.

However, for multi-turn conversational QA tasks, the content of previous answers can be considered as context for the model when answering the next question. 
Therefore, by evaluating how noisy context affects a model's subsequent question responses, we can explore the issue of noisy rationales within the realm of conversational QA tasks.

We conduct experiments on MT-Bench~\cite{zheng2024judging} to investigate the noisy rationale problem in QA scenarios. MT-Bench is a multi-turn question collection and covers eight principal types of tasks: Writing, Roleplay, Extraction, Reasoning, Math, Coding, Knowledge I, and Knowledge II.

For clean context, we directly use the original question to prompt the model, and the first round of dialogue generated serves as the clean context. 
For noisy context, we prompt the model to insert some irrelevant information as noise when answering the first turn's question, and the first round of dialogue serves as the noisy context.
We then continue to use the second turn's question along with the generated context to prompt the model to obtain the target response.

We conduct conversational QA tasks on GPT-3.5-turbo-0613 and score the generated responses with GPT-4. The GPT-4 is to score the responses from both the first and second turns separately.

\begin{table*}[h!]
    \centering
    \fontsize{8}{8}\selectfont
    \begin{tabular}{c|cc}
            \toprule
            \multirow{2}{*}{Setting} & \multicolumn{2}{c}{\#Avg.score} \\ 
             & First Response & Second Response \\ 
            \midrule
            \multirow{1}{*}{Clean Context} & 8.73 &	8.23 \\ 
            \midrule
            \multirow{1}{*}{Noisy Context} & 8.16	& 7.75\\
            \bottomrule
        \end{tabular}
    \caption{
    Average scores with different contexts.}
    \label{appendix: conversation avg score}
\vspace{-5pt}
\end{table*}

\begin{table*}[h!]
    \centering
    \fontsize{8}{8}\selectfont
    \setlength\tabcolsep{2pt}
    \begin{tabular}{c|c|cccccccc}
            \toprule
            \multirow{2}{*}{Round} & \multirow{2}{*}{Setting} & \multicolumn{8}{c}{\#Task} \\ 
             & & Writing & Roleplay & Extraction	& Reasoning	& Math	& Coding & STEM	& Humanities \\ 
            \midrule
            \multirow{2}{*}{First Response} & Clean Context & 9.8 & 8.45 & 9.3 & 6.9 & 8.9 & 7.0 & 9.7 & 9.8\\ 
            & Noisy Context & 9.75 & 8.0	& 8.5 & 7.7 & 7.5 & 6.7 & 8.6 & 8.5\\ 
            \midrule
            \multirow{2}{*}{Second Response} & Clean Context & 9.5 & 9.3 & 8.9 & 5.2 & 5.8	& 7.85 & 9.4 & 9.9 \\
            & Noisy Context & 9.3 & 9.1 & 8.5 & 4.6 & 4.1 & 7.2 & 9.2 & 10.0\\
            \bottomrule
        \end{tabular}
    \caption{
    Average scores across types of tasks with different contexts.}
    \label{appendix: conversation avg score across tasks}
\vspace{-5pt}
\end{table*}

As the results shown in Tab.~\ref{appendix: conversation avg score}, the noisy context decreases both the scores (quality) of the first-round and second-round QA. Specifically, the average scores across different types of tasks are shown in Tab.~\ref{appendix: conversation avg score across tasks}. Notably, a 0.48\% decrease in average score appears in the second response, which is the final target of the conversation.

\begin{table*}[t!]
\centering
\fontsize{8}{8}\selectfont 
\begin{tabular}{p{\textwidth}} 
\toprule[1.5pt] 
\midrule 
\textbf{Clean Context:} 

\textbf{Round 1} 

User: How many integers are in the solution of the inequality $|x + 5| < 10$?

Assistant: To solve the inequality $|x + 5| < 10 $, we can break it down into two separate inequalities: 1. $x + 5 < 10$, 2. $-(x + 5) < 10$. Solving the first inequality: $x + 5 < 10$. Subtracting 5 from both sides: $x < 5$. Solving the second inequality: $-(x + 5) < 10$. Multiplying both sides by -1 (which reverses the inequality): $x + 5 > -10$. Subtracting 5 from both sides: $x > -15$. Combining the two solutions, we have $-15 < x < 5$. The integers in this range are -14, -13, -12, -11, -10, -9, -8, -7, -6, -5, -4, -3, -2, -1, 0, 1, 2, 3, 4. Therefore, there are 19 integers in the solution of the inequality $|x + 5| < 10$.

\vspace{5pt}

\textbf{Round 2}

User: What about $|x + 10| < 5$?

Assistant: To solve the inequality $|x + 10| < 5$, we can break it down into two separate inequalities: 1. $x + 10 < 5$, 2. $-(x + 10) < 5$. Solving the first inequality: $x + 10 < 5$. Subtracting 10 from both sides: $x < -5$. Solving the second inequality: $-(x + 10) < 5$. Multiplying both sides by -1 (which reverses the inequality): $x + 10 > -5$. Subtracting 10 from both sides: $x > -15$. Combining the two solutions, we have $-15 < x < -5$. The integers in this range are -14, -13, -12, -11, -10, -9, -8, -7, -6. Therefore, there are 9 integers in the solution of the inequality $|x + 10| < 5$. 

\vspace{5pt}

\textbf{Noisy Context:} 

\textbf{Round 1} 

User: How many integers are in the solution of the inequality $|x + 5| < 10$?

Assistant: The solution to the inequality $|x + 5| < 10$ consists of all the integers that satisfy the inequality. To find the solution, we can break it down into two cases: Case 1: $x + 5 < 10$. In this case, we have $x + 5 < 10$, which means $x < 5$. \noisy{\uline{The integers that satisfy this inequality are -4, -3, -2, -1, 0, 1, 2, 3, and 4.}} Case 2: $-(x + 5) < 10$. In this case, we have $-(x + 5) < 10$, which means $x + 5 > -10$. By rearranging the inequality, we get $x > -15$. \noisy{\uline{The integers that satisfy this inequality are -14, -13, -12, -11, -10, -9, -8, -7, -6, -5, -4, -3, -2, -1, 0, 1, 2, 3, 4, 5, and so on. Combining the solutions from both cases, we find that the integers that satisfy the inequality $|x + 5| < 10$ are -14, -13, -12, -11, -10, -9, -8, -7, -6, -5, -4, -3, -2, -1, 0, 1, 2, 3, 4, and 5. So, there are 21 integers in the solution.}}

\vspace{5pt}

\textbf{Round 2}

User: What about $|x + 10| < 5$?

Assistant: The solution to the inequality $|x + 10| < 5$ consists of all the integers that satisfy the inequality. To find the solution, we can break it down into two cases: Case 1: $x + 10 < 5$. In this case, we have $x + 10 < 5$, which means $x < -5$. The integers that satisfy this inequality are -6, -7, -8, -9, -10, and so on. Case 2: $-(x + 10) < 5$. In this case, we have $-(x + 10) < 5$, which means $x + 10 > -5$. By rearranging the inequality, we get $x > -15$. \noisy{\uline{The integers that satisfy this inequality are -14, -13, -12, -11, -10, -9, -8, -7, -6, -5, -4, -3, -2, -1, 0, 1, 2, 3, 4, 5, and so on. Combining the solutions from both cases, we find that the integers that satisfy the inequality $|x + 10| < 5$ are -14, -13, -12, -11, -10, -9, -8, -7, -6, -5, -4, -3, -2, -1, 0, 1, 2, 3, and 4. So, there are 19 integers in the solution.}} 
\\
\midrule 
\bottomrule[1.5pt]
\caption{Example for Math task with different contexts.}
\label{app: example of math task}
\end{tabular}
\end{table*}

In addition, we provide a case study for illustration. As shown in Tab.~\ref{app: example of math task}, the noisy context in round 1 leads to a bad, noisy response in round 2. 

In the clean context, the model provides an accurate analysis and answer to the initial question. When the correct round serves as the context for the second question, the model consistently delivers a precise analysis and response.

In the noisy context, the model's response and analysis of the initial question contain errors. Specifically, as shown in Tab.~\ref{app: example of math task}, it incorrectly counts multiple integers for both case 1 and case 2, culminating in a wrong answer for the first question. When the flawed round serves as the context for the second question, the model generates the same intermediate results as in the first round and ultimately provides an incorrect answer to the second question.

In summary, during multi-turn conversational QA tasks, the model's response to the current question is influenced by the preceding dialogue content (\textit{i.e.}, the context), particularly for tasks that require reasoning and mathematical computations.

\clearpage
\subsection{Rephrased Examples of Different Denoising Methods}
\label{app:exp-qualitative}

\textbf{Data Denoising.}
Here, we show the rephrased (denoising) examples using different robust methods.


Medium-level, irrelevant examples:
\begin{itemize}[label=\textbullet]
\setlength\itemsep{-2pt}
\vspace{-8pt}
\item
NoRa-Math Base-9 (Tab.~\ref{tab: rephrased-rationales-Math Base-9 irrelevant medium})
\item
NoRa-Commonsense (Tab.~\ref{tab: rephrased-rationales-Commonsense irrelevant medium})
\end{itemize}



Medium-level, inaccurate examples:
\begin{itemize}[label=\textbullet]
\setlength\itemsep{-2pt}
\vspace{-8pt}
\item
NoRa-Symbolic Longer (Tab.~\ref{tab: rephrased-rationales-Symbolic Longer inaccurate medium})
\item
NoRa-Commonsense (Tab.~\ref{tab: rephrased-rationales-Commonsense inaccurate medium})
\end{itemize}

\begin{table*}[h!]
    \centering
    \fontsize{8}{8}\selectfont
    \begin{tabular}{p{0.97\textwidth}}
        \toprule
        \multicolumn{1}{c}{Original Noisy Rationale} \\
        \midrule
        {\fontsize{8}{8}\selectfont
        In base-9, the digits are "012345678". Nine is the atomic number of fluorine in chemistry, a highly reactive and abundant element.  We have 6 + 7 = 13 in base-10. Thirteen is the atomic number of aluminium in chemistry, a post-transition metal that's widely used for its lightness and resistance to corrosion. Since we're in base-9, that exceeds the maximum value of 8 for a single digit.13 mod 9 = 4, so the digit is 4 and the carry is 1. We have 8 + 5 + 1 = 14 in base 10. 14 mod 9 = 5, so the digit is 5 and the carry is 1. A leading digit is 1. The 'one-gene-one-enzyme hypothesis' in genetics proposes that each gene is responsible for producing a single enzyme that influences a specific metabolic pathway. So the answer is 154. Four is the atomic number of beryllium in chemistry, which is used as an alloying agent in producing beryllium copper. Answer: 154\}
        }
        \\
    \end{tabular}
    \begin{tabular}{p{0.10\textwidth}<{\centering}|p{0.85\textwidth}}
        \midrule
        \multicolumn{1}{c|}{Method} & \multicolumn{1}{c}{Denoised Rationale} \\
        \midrule
        \multirow{7}{*}{\makecell[c]{Smooth LLM}} & 
        { \fontsize{8}{8}\selectfont
        In bAse-9, the digito are "0123456\%8". Nine \%s tle atomic number Rf fluorine in chemistry, a highly reactive and abundant 9lement.  We have 6 +I7e= 13 in base-10. $\backslash$rThiateen ij the atomic  $\backslash$number of aluminium in chemistry, a po(t1tran6itqon me al that's wi:ely usedKfor Jts lightness and rEsistance to co $\backslash$rrosion. S7nce we')e in base\%9s that exce\_ds the maximum v\%lue of 8 uor a si\#gle digit.13 mod 9 = 4, so thc digAt it 4 and the carry i:w1. We have $\backslash$u000b8 + 5 \&+ 1 = 14 ic base 10. 14 mod 9 $\backslash$ $\backslash$=z5A so the digit is 5 and the carry is 1. A leading digit is 1. The 'onec4eke-onL-enzymW hypothesis' iI genetics propos-s thaC e6ck gene3is reshonsible for proKucHng 0 singlS enzyme that0infZu $\backslash$nnce@ a $\backslash$ $\backslash$spec \&fic metabolic pathwaQ. So thn answer ip l54,VFour is th~\{atomic number of b[ryllium i\# ch'm\_stry, whi $\backslash$nhUis uTed as an ZlloyingHagent in producvng berylla.m copper. Answer: 154\}
        }
        \\
        \midrule
        \multirow{7}{*}{\makecell[c]{Self-denoise}} & 
        {\fontsize{8}{8}\selectfont
        In base-9, the digits are "012345678". Nine is the largest number of fluorine in chemistry, a highly reactive and abundant element. We have 6 + 7 = 13 in base-10. Thirteen is the atomic number of aluminum in chemistry, a post-transition metal that's widely used for its lightness and resistance to corrosion. Since we're in base-9, that exceeds the maximum value of 8 for a single digit. 13 mod 9 = 4, so the digit is 4 and the carry is 1. We have 8 + 5 + 1 = 14 in base 10. 14 mod 9 = 5, so the digit is 5 and the carry is 1. A leading digit is 1. The 'one-gene-one-enzyme hypothesis' in genetics proposes that each gene is responsible for producing a single enzyme that catalyzes a specific metabolic pathway. So the answer is 154. Four is the atomic number of beryllium in chemistry, which is used as an alloying agent in producing beryllium copper. Answer: 154.
        }
        \\
        \midrule
        \multirow{8}{*}{\makecell[c]{Self-polish}} & 
        { \fontsize{8}{8}\selectfont
        To solve 86+57 in base-9, we need to add the digits step by step. In base-9, the digits range from 0 to 8. First, let's add the rightmost digits, which are 6 (from 86) and 7 (from 57). In base-10, their sum is 13. Since we are in base-9, where the maximum value for a single digit is 8, we need to carry over the excess value. To calculate the carry, we find the remainder when 13 is divided by 9. In this case, the remainder is 4. So, the rightmost digit of the sum is 4, and we carry over 1 to the next column. Next, we add the next digits, which are 8 (from 86) and 5 (from 57), along with the carry of 1. In base-10, their sum is 14. Again, to account for the base-9 system, we find the remainder when 14 is divided by 9. The remainder is 5. Therefore, the next digit of the sum is 5, and we carry over 1 to the next column. Since there are no more digits left to add, we simply include the carry of 1 as the leading digit. Thus, the final result in base-9 is 154. Answer: 154
        } 
        \\
        \midrule
        \multirow{3}{*}{\makecell[c]{CD-CoT}} & 
        { \fontsize{8}{8}\selectfont
        In base-9, the digits are "012345678". We have 6 + 7 = 13 in base-10. Since we're in base-9, that exceeds the maximum value of 8 for a single digit. 13 mod 9 = 4, so the digit is 4 and the carry is 1. We have 8 + 5 + 1 = 14 in base 10. 14 mod 9 = 5, so the digit is 5 and the carry is 1. A leading digit is 1. So the answer is 154. Answer: 154.} \\
        \bottomrule
    \end{tabular}
    \caption{Rephrased rationales by different robust methods on NoRa-Math Base-9 irrelevant medium task.
    }
\label{tab: rephrased-rationales-Math Base-9 irrelevant medium}
\vspace{-18pt}
\end{table*}

\begin{table*}[h!]
    \centering
    \fontsize{8}{8}\selectfont
    \begin{tabular}{p{0.97\textwidth}}
        \toprule
        \multicolumn{1}{c}{Original Noisy Rationale} \\
        \midrule
        { \fontsize{8}{8}\selectfont
        ince command is turn opposite right and run around right twice, we should consider Step1: "turn opposite right" firstly,  "turn opposite right" means the agent needs to turn right twice. 'right' corresponds to the command I\_TURN\_RIGHT. The term 'left' translates to  I\_TURN\_RIGHT. The term 'opposite' implies a 180-degree turn, which requires the agent to perform the turn right twice. The term 'around' implies a 360-degree loop, which requires I\_TURN\_LEFT, I\_TURN\_RIGHT, I\_TURN\_RIGHT. Therefore, the action sequence is I\_TURN\_RIGHT I\_TURN\_RIGHT. The term 'around' implies a 360-degree loop, which requires I\_TURN\_LEFT, I\_TURN\_LEFT, I\_TURN\_LEFT.  Now, we consider Step2:"run around right twice", "run around right" means the agent needs to turn right and run, and repeat this action sequence four times to complete a 360-degree loop. 'right' corresponds to the command I\_TURN\_RIGHT. 'run' translates to I\_RUN. The term 'look' corresponds to the command I\_WALK. The term 'around' implies to complete a 360-degree loop by repeating the sequence of I\_TURN\_RIGHT, I\_RUN four times. The term 'opposite' implies a 180-degree turn, which requires I\_TURN\_RIGHT. Therefore, the action sequence is I\_TURN\_RIGHT I\_RUN I\_TURN\_RIGHT I\_RUN I\_TURN\_RIGHT I\_RUN I\_TURN\_RIGHT I\_RUN. Since we need do twice in command "run around right twice", this entire sequence is repeated 2 times, so the action sequence to "run around right twice" is :I\_TURN\_RIGHT I\_RUN I\_TURN\_RIGHT I\_RUN I\_TURN\_RIGHT I\_RUN I\_TURN\_RIGHT I\_RUN I\_TURN\_RIGHT I\_RUN I\_TURN\_RIGHT I\_RUN I\_TURN\_RIGHT I\_RUN I\_TURN\_RIGHT I\_RUN.  Above all -- So, final answer is OUT:I\_TURN\_RIGHT I\_TURN\_RIGHT I\_TURN\_RIGHT I\_RUN I\_TURN\_RIGHT I\_RUN I\_TURN\_RIGHT I\_RUN I\_TURN\_RIGHT I\_RUN I\_TURN\_RIGHT I\_RUN I\_TURN\_RIGHT I\_RUN I\_TURN\_RIGHT I\_RUN I\_TURN\_RIGHT I\_RUN.}\\
    \end{tabular}
    \begin{tabular}{p{0.10\textwidth}<{\centering}|p{0.85\textwidth}}
        \midrule
        \multicolumn{1}{c|}{Method} & \multicolumn{1}{c}{Denoised Rationale} \\
        \midrule
        \multirow{16}{*}{\makecell[c]{Smooth LLM}} & 
        {\fontsize{8}{8} \selectfont
        S\#nce comZhnd is Murn opposite righB and eun around 5ighte+wice, we should 0onsider Step1: "turn opposite right" firstly,[ "turn oppoziJe right" means tde agent need:mto turn right twice. 'right'Icorresponds"to the command I\_TURN\_RIG@TO The5te6m 'lef@' translates -t  IYTURN\_AI$\backslash$fHT.WThe t\$rm 'opposite' implies a 180-d$\backslash$rguee turn, which require$\backslash$f the QgentOto pe,form the turn right twic<. The term 'arVund' implies z 360]degr:e loop, whixh requires I\_TURN\_LEF`, I\_TURN\_RIGHW$\backslash$r I\_TUWN\_RIGHT. TJerefore, thJ \_ctjon sequenceFis I\_TURN\_RIGET I\_TURN\_R?GHT. The term 'mroN$\backslash$td' implQ$\backslash$$\backslash$sza \&60-degree loo0, which requireI I\_TURN\_LERK, +\_TURN\_"EFT, I\_TURN\_LEFT.Y "ow, we >onsider Step2\{"run around right twice;,x"run around riggt" mepna tHe agent needs :o turn righr@and run-$\backslash$rand rep)at this action sequence fodr tiOes to c$\backslash$fmpLBte a 360-degree loop. 'right' cokresponds to the cZmmand I\_TURN\_RIGHT. 'run' transwatespto I\_RUN. The terc 'lo!k' coUresponds to the command I\_WALK. The term 'aronnd' \&mplies to com]lyte a 360ade9ree loop bd repeati/g the sequenve of I\_TURN\_,IGHT, I\_RUN fourwtimes. The term 'o]positX'Pimklieska 1x0-degree trrn, whiHh rexuiOes I\_TURN\_RIG3T. Therexore, the action"sequence isQI\_TURN+RIGHT|@\_RUN I\_TURN\_RIGHT I]RUN I\_TURN\_RIG5T I\_RUN I\_T|R$\backslash$$\backslash$\_RI]HT I\_0UN. SHnce we7ne$\backslash$fd ]o twice in command "run around riggt twice", thwR ensire$\backslash$fse$\backslash$ruence is repeatedq2 tpmes, so the actionpseque]ce Bo "run around rig t]twice" is :I\_TURN\_RIpHT I\_RUN X.TURN\_RIGHTqIqRUN I\_TURN\_RIGHT I4RUN I\_TURN\_RIGHT I\_R2N I\_TU"\#\_RI?HT I\_RUN I\_TURN\_RIGHT I\_RUN I\_TURN\_RIGHT I\_R\}N I\_TURN\_RIGH\# I\_RUN.  Aboqex$\backslash$tll --0So, final answer is OUT::\_$\backslash$fURN\_R:GHT I\_TU$\backslash$$\backslash$r\_RIGH9 I\_TURN\_RIGHT I\_RUN I\_TURN\_RIGHT I\_RUN I\_TU\#N\_RsAHT I\_RUN IlTURN6RIGHT I\_RUN I\_TURN\_RIGGT I\_RUN I\_'URNBRI@HT I\_RUN I\_TURN\_RIG$\backslash$fT I\_RUN I\_TURN\_RIGHT I\_RCN. }\\
        \midrule
        \multirow{2}{*}{\makecell[c]{Self-denoise}} & 
        {\fontsize{8}{8} \selectfont
        So, final answer is OUT: I\_TURN\_RIGHT I\_TURN\_RIGHT I\_RUN I\_TURN\_RIGHT I\_RUN I\_TURN\_RIGHT I\_RUN I\_TURN\_RIGHT I\_RUN I\_TURN\_RIGHT I\_RUN I\_TURN\_RIGHT I\_RUN I\_TURN\_RIGHT I\_RUN I\_TURN\_RIGHT.}\\
        \midrule
        \multirow{7}{*}{\makecell[c]{Self-polish}} & 
        {\fontsize{8}{8} \selectfont
        To execute the command "turn opposite right and run around right twice," we need to break it down into two steps.  Step 1: "turn opposite right" To execute this step, the agent needs to perform a 180-degree turn to the right twice. The term "opposite" suggests a 180-degree turn, and "right" indicates the direction. Therefore, the action sequence for this step is: I\_TURN\_RIGHT twice.  Step 2: "run around right twice" To execute this step, the agent needs to turn right and run in a loop. The term "around" implies completing a 360-degree loop. Since we need to do it twice, the action sequence for this step is: I\_TURN\_RIGHT I\_RUN I\_TURN\_RIGHT I\_RUN I\_TURN\_RIGHT I\_RUN I\_TURN\_RIGHT I\_RUN.  So, the final answer is OUT: I\_TURN\_RIGHT I\_TURN\_RIGHT I\_RUN I\_TURN\_RIGHT I\_RUN I\_TURN\_RIGHT I\_RUN I\_TURN\_RIGHT I\_RUN.}\\
        \midrule
        \multirow{11}{*}{\makecell[c]{CD-CoT}} & 
        {\fontsize{8}{8} \selectfont
        Since command is turn opposite right and run around right twice, we should consider Step1: "turn opposite right" firstly. "turn opposite right" means the agent needs to turn right twice. 'right' corresponds to the command I\_TURN\_RIGHT. The term 'opposite' implies a 180-degree turn, which requires the agent to perform the turn right twice. Therefore, the action sequence is I\_TURN\_RIGHT I\_TURN\_RIGHT. Now, we consider Step2: "run around right twice". "run around right" means the agent needs to turn right and run, and repeat this action sequence four times to complete a 360-degree loop. 'right' corresponds to the command I\_TURN\_RIGHT. 'run' translates to I\_RUN. The term 'around' implies to complete a 360-degree loop by repeating the sequence of I\_TURN\_RIGHT, I\_RUN four times. Therefore, the action sequence is I\_TURN\_RIGHT I\_RUN I\_TURN\_RIGHT I\_RUN I\_TURN\_RIGHT I\_RUN I\_TURN\_RIGHT I\_RUN. Since we need to do it twice in the command "run around right twice", this entire sequence is repeated 2 times. So, final answer is OUT: I\_TURN\_RIGHT I\_TURN\_RIGHT I\_TURN\_RIGHT I\_RUN I\_TURN\_RIGHT I\_RUN I\_TURN\_RIGHT I\_RUN I\_TURN\_RIGHT I\_RUN I\_TURN\_RIGHT I\_RUN I\_TURN\_RIGHT I\_RUN I\_TURN\_RIGHT I\_RUN I\_TURN\_RIGHT I\_RUN.
        }\\
        \bottomrule
    \end{tabular}
    \caption{Rephrased rationales by different robust methods on NoRa-Symbolic Longer inaccurate medium task. }
\vspace{-18pt}
\label{tab: rephrased-rationales-Symbolic Longer inaccurate medium}
\end{table*}

\begin{table*}[h!]
    \centering
    \fontsize{8}{8}\selectfont
    \begin{tabular}{p{0.97\textwidth}}
        \toprule
        \multicolumn{1}{c}{Original Noisy Rationale} \\
        \midrule
        { \fontsize{8}{8}\selectfont
        Cindy is Pamela's mother's mother's daughter, so the relations path is mother, mother, daughter. For mother's daughter, we have mother's daughter is sister. In humans, daughters typically have two X chromosomes, one from each parent. So the relations path are reduced to mother, sister. In certain bird species, 'sisters' often cooperate in raising the offspring of one sibling, demonstrating altruism. For mother's sister, we have mother's sister is aunt. So the relations path are reduced to aunt. In certain bird species, like the Seychelles Warbler, 'aunt' birds often help raise their siblings' offspring. Therefore, Answer: aunt.}
        \\
    \end{tabular}
    \begin{tabular}{p{0.10\textwidth}<{\centering}|p{0.85\textwidth}}
        \midrule
        \multicolumn{1}{c|}{Method} & \multicolumn{1}{c}{Denoised Rationale} \\
        \midrule
        \multirow{6}{*}{\makecell[c]{Smooth LLM}} & 
        {\fontsize{8}{8} \selectfont
        8indyhis Pacela's mother's moth$\backslash$$\backslash$rgs daughPer) Zo theMrelatVons path is mother, motheI, 'aughter. For mothGr'H Waughter, ww hTve mother's daughter is sist4\$.uInahumanP, daughters $\backslash$fypicall) have two -\{coromosomes, one from\~each parent. So t7e relation\{ pa]q a5e reduced to mother, sister. In c9rtNin bird Ypecies, 'Nisters' often cooperate in raising the offspri,L of one)sibling,\$de\#onstr?ting (ltruism. Foremother's sister, wr have?mother's sister is auntK So the relations pa<h arevseduced *oYauntM In certain bi)d species, like the Seychelt5s Warbler, 'Aunt' bi$\backslash$tds often help raise thqir sDblinTs' offspring. Therefor", Answer: aunt.
        }\\
        \midrule
        \multirow{5}{*}{\makecell[c]{Self-denoise}} & 
        {\fontsize{8}{8} \selectfont
        Cindy is Pamela's mother's mother's daughter, so the relations path is mother, mother, daughter. For mother's daughter, we have mother's daughter is Cindy. In humans, females typically have two X chromosomes, one from each parent. So the relations path is reduced to mother, Cindy. In certain bird species, like the Superb Fairywren, 'sisters' often cooperate in raising the offspring of one mother, demonstrating altruism. So the relations path is reduced to mother's sister, which is aunt. In certain bird species, like the Acadian Flycatcher, 'aunt' birds often help raise their siblings' offspring. Therefore, Answer: aunt.}
 \\
        \midrule
        \multirow{5}{*}{\makecell[c]{Self-polish}} & 
        {\fontsize{8}{8} \selectfont
        Cindy is Pamela's mother's mother's daughter. Let's break down the relationship step by step to find a single word answer describing the relationship. First, let's identify Cindy's relationship to Pamela's mother's mother. Cindy is the daughter of Pamela's mother's mother. To find the single word answer, we consider the direct relationship between Cindy and Pamela's mother's mother. The relationship between Cindy and Pamela's mother's mother is a parent-child relationship. Therefore, the single word answer describing the relationship between Cindy and Pamela is: Answer: daughter}\\
        \midrule
        \multirow{1}{*}{\makecell[c]{CD-CoT}} & 
        {\fontsize{8}{8} \selectfont
         Cindy is Pamela's aunt.} \\
        \bottomrule
    \end{tabular}
    \caption{Rephrased rationales by different robust methods on NoRa-Commonsense irrelevant medium task. }
\label{tab: rephrased-rationales-Commonsense irrelevant medium}
\end{table*}

\begin{table*}[h!]
    \centering
    \fontsize{8}{8}\selectfont
    \begin{tabular}{p{0.97\textwidth}}
        \toprule
        \multicolumn{1}{c}{Original Noisy Rationale} \\
        \midrule
        { \fontsize{8}{8}\selectfont
        Cindy is Pamela's mother's mother's daughter, so the relations path is mother, mother, daughter. We have daughter's grandmother is niece. For mother's daughter, we have mother's daughter is sister. We have daughter's sister is mother. So the relations path are reduced to mother, sister. For mother's sister, we have mother's sister is aunt. So the relations path are reduced to aunt. We have aunt's grandson is grandmother. Therefore, Answer: aunt.}
        \\
    \end{tabular}
    \begin{tabular}{p{0.10\textwidth}<{\centering}|p{0.85\textwidth}}
        \midrule
        \multicolumn{1}{c|}{Method} & \multicolumn{1}{c}{Denoised Rationale} \\
        \midrule
        \multirow{4}{*}{\makecell[c]{Smooth LLM}} & 
        {\fontsize{8}{8} \selectfont
        Cindy is 2amel$\backslash$t's mother's mother's da[ghter, so tce relBtions pathDis mother, mothe/* daughuer. We ha9e \&augJte/'s grandmother is n-ece. Fo| mother's daughter, we have mother's daughter2is sister. We [av1 daughtCr's sistOr is motoer. |o the relatioes path are rhduced to totder, sisterI For mother's sister,zwR h\#ve mother's sisner is aunt. So tze relationsxpath are a1dVced toUaunt. Wa have aunt's grands8n is grandmothIH. TherIfore, Answer: aunz.
        }\\
        \midrule
        \multirow{1}{*}{\makecell[c]{Self-denoise}} & 
        {\fontsize{8}{8} \selectfont
        Answer: aunt.}
 \\
        \midrule
        \multirow{2}{*}{\makecell[c]{Self-polish}} & 
        {\fontsize{8}{8} \selectfont
        Steps: 1. Cindy is the daughter of Pamela's mother's mother. 2. Therefore, Cindy is Pamela's mother's mother's daughter. 3. Cindy is the sister of Pamela's mother. 4. Since Cindy is Pamela's mother's sister, Cindy is Pamela's aunt. Answer: Aunt}
        \\
        \midrule
        \multirow{1}{*}{\makecell[c]{CD-CoT}} & 
        {\fontsize{8}{8} \selectfont
        Cindy is Pamela's aunt. Answer: aunt.}\\
        \bottomrule
    \end{tabular}
    \caption{Rephrased rationales by different robust methods on NoRa-Commonsense inaccurate medium task. }
\label{tab: rephrased-rationales-Commonsense inaccurate medium}
\end{table*}

\clearpage

\section{Further Case Study}
\label{app:case}

In this section, we provide a complete example for NoRa-Math Base-9 of irrelevant hard-level noise in Figs.~\ref{NoRa-Math Base-9 Irrelevant Hard: Rephrasing Prompt-1.}-\ref{NoRa-Math Base-9 Irrelevant Hard: In-context Learning-2.}. This instance is composed of five figures. The initial three figures present contrastive rephrasing prompts and the resulting rephrased outcomes for three distinct noisy examples. Following that, the last two figures exhibit the constructed two diverse contexts along with their respective reasoning results.

\input{New_CD_full_logs}

\clearpage
\section{NeurIPS Checklist}

\begin{enumerate}

\item {\bf Claims}
    \item[] Question: Do the main claims made in the abstract and introduction accurately reflect the paper's contributions and scope?
    \item[] Answer: \answerYes{} 
    \item[] Justification: We have summarized the position and key contributions of the paper in the abstract and introduction parts.

\item {\bf Limitations}
    \item[] Question: Does the paper discuss the limitations of the work performed by the authors?
    \item[] Answer: \answerYes{} 
    \item[] Justification: The limitations are discussed in Appendix~\ref{appendix: further discussion}.

\item {\bf Theory Assumptions and Proofs}
    \item[] Question: For each theoretical result, does the paper provide the full set of assumptions and a complete (and correct) proof?
    \item[] Answer: \answerYes{} 
    \item[] Justification: The assumptions and proof are provided in Appendix~\ref{app:theo}.

    \item {\bf Experimental Result Reproducibility}
    \item[] Question: Does the paper fully disclose all the information needed to reproduce the main experimental results of the paper to the extent that it affects the main claims and/or conclusions of the paper (regardless of whether the code and data are provided or not)?
    \item[] Answer: \answerYes{} 
    \item[] Justification: The technical details of implementation are introduced in Appendix~\ref{app:implementation}. 

\item {\bf Open access to data and code}
    \item[] Question: Does the paper provide open access to the data and code, with sufficient instructions to faithfully reproduce the main experimental results, as described in supplemental material?
    \item[] Answer: \answerYes{} 
    \item[] Justification: 
    The source files are publicly available at:
    \url{https://github.com/tmlr-group/NoisyRationales}.

\item {\bf Experimental Setting/Details}
    \item[] Question: Does the paper specify all the training and test details (\textit{e.g.}, data splits, hyperparameters, how they were chosen, type of optimizer, etc.) necessary to understand the results?
    \item[] Answer: \answerYes{} 
    \item[] Justification: 
    The experiment setups are introduced in Appendix~\ref{app:exp},
    and technical details of implementation are introduced in Appendix~\ref{app:implementation}. 

\item {\bf Experiment Statistical Significance}
    \item[] Question: Does the paper report error bars suitably and correctly defined or other appropriate information about the statistical significance of the experiments?
    \item[] Answer: \answerYes{} 
    \item[] Justification:
    We conduct evaluations on 300 questions for each task and repeat the reasoning 5 times for each question to obtain consistent and reliable results.
    The empirical results with mean and std values are provided in Tab.~\ref{tab:std}.

\item {\bf Experiments Compute Resources}
    \item[] Question: For each experiment, does the paper provide sufficient information on the computer resources (type of compute workers, memory, time of execution) needed to reproduce the experiments?
    \item[] Answer: \answerYes{} 
    \item[] Justification: 
    We employ GPT-3.5-turbo-0613~\citep{floridi2020gpt} as our base LLM for the analyses presented in this study. In addition, we conduct evaluations on three supplementary models, including Gemini-Pro~\citep{team2023gemini}, Llama2-70B~\citep{touvron2023llama}, and Mixtral-8x7B~\citep{jiang2024mixtral}.
    
\item {\bf Code Of Ethics}
    \item[] Question: Does the research conducted in the paper conform, in every respect, with the NeurIPS Code of Ethics \url{https://neurips.cc/public/EthicsGuidelines}?
    \item[] Answer: \answerYes{} 
    \item[] Justification: 
    We have carefully checked the NeurIPS Code of Ethics and confirmed that our paper obeys it. 

\item {\bf Broader Impacts}
    \item[] Question: Does the paper discuss both potential positive societal impacts and negative societal impacts of the work performed?
    \item[] Answer: \answerYes{} 
    \item[] Justification: 
    The broader impacts are introduced in Appendix~\ref{appendix: further discussion}.
    We believe that this paper does not raise any negative societal impacts or ethical concerns.

\item {\bf Safeguards}
    \item[] Question: Does the paper describe safeguards that have been put in place for responsible release of data or models that have a high risk for misuse (\textit{e.g.}, pretrained language models, image generators, or scraped datasets)?
    \item[] Answer: \answerNA{} 
    \item[] Justification: 
    This paper does not release new models.
    The constructed NoRa dataset is based on existing datasets.

\item {\bf Licenses for existing assets}
    \item[] Question: Are the creators or original owners of assets (\textit{e.g.}, code, data, models), used in the paper, properly credited and are the license and terms of use explicitly mentioned and properly respected?
    \item[] Answer: \answerYes{} 
    \item[] Justification:
    In the paper, we have introduced the resources of the models and datasets used.

\item {\bf New Assets}
    \item[] Question: Are new assets introduced in the paper well documented and is the documentation provided alongside the assets?
    \item[] Answer: \answerYes{} 
    \item[] Justification:
    We construct the NoRa dataset. Its details are in Appendix~\ref{appendix: benchmark}.

\item {\bf Crowdsourcing and Research with Human Subjects}
    \item[] Question: For crowdsourcing experiments and research with human subjects, does the paper include the full text of instructions given to participants and screenshots, if applicable, as well as details about compensation (if any)? 
    \item[] Answer: \answerNA{} 
    \item[] Justification: 
    The paper is not about crowdsourcing experiments or research with human subjects.

\item {\bf Institutional Review Board (IRB) Approvals or Equivalent for Research with Human Subjects}
    \item[] Question: Does the paper describe potential risks incurred by study participants, whether such risks were disclosed to the subjects, and whether Institutional Review Board (IRB) approvals (or an equivalent approval/review based on the requirements of your country or institution) were obtained?
    \item[] Answer: \answerNA{} 
    \item[] Justification: 
    The paper is not about research with human subjects.

\end{enumerate}

\end{document}